%% file: main.tex
\documentclass{article}

\usepackage{microtype}
\usepackage{graphicx}
\usepackage{booktabs} 

\usepackage{hyperref}

\usepackage[accepted]{icml2025}

\usepackage{amsmath}
\usepackage{amssymb}
\usepackage{mathtools}
\usepackage{amsthm}

\usepackage{algorithm}
\usepackage{algorithmic}
\usepackage{newfloat}
\usepackage{listings}

\usepackage{amsmath}
\usepackage{amsthm}

\usepackage{amssymb}
\usepackage{enumitem}   
\usepackage{bm}
\usepackage{tabularx}
\usepackage{subcaption}
\usepackage{array}
\usepackage{multirow}
\usepackage{siunitx}
\usepackage{booktabs}

\newcommand{\rf}{\textsl{Reinfier}}
\newcommand{\rt}{\textsl{Reintrainer}}
\newcommand{\drlp}{DRLP}

\DeclareMathOperator*{\argmax}{arg\,max}
\DeclareMathOperator*{\argmin}{arg\,min}

\usepackage[capitalize,noabbrev]{cleveref}

\theoremstyle{plain}
\newtheorem{theorem}{Theorem}[section]

\newtheorem{corollary}[theorem]{Corollary}
\theoremstyle{definition}
\newtheorem{definition}[theorem]{Definition}

\theoremstyle{remark}

\usepackage[textsize=tiny]{todonotes}

\icmltitlerunning{Verification and Interpretation-Driven Safe Deep Reinforcement Learning Framework}

\begin{document}

\twocolumn[
\icmltitle{Verification and Interpretation-Driven \\ Safe Deep Reinforcement Learning Framework}



\icmlsetsymbol{equal}{*}

\begin{icmlauthorlist}
\icmlauthor{Zixuan Yang}{nju}
\icmlauthor{Jiaqi Zheng}{nju}
\icmlauthor{Guihai Chen}{nju}
\end{icmlauthorlist}

\icmlaffiliation{nju}{School of Computer Science, University of Nanjing, Nanjing, China}


\icmlkeywords{Machine Learning, ICML}

\vskip 0.3in
]




\begin{abstract}
    Ensuring verifiable and interpretable safety of deep reinforcement learning (DRL) is crucial for its deployment in real-world applications. 
    Existing approaches like verification-in-the-loop training, however, face challenges such as difficulty in deployment, inefficient training, lack of interpretability, and suboptimal performance in property satisfaction and reward performance.
    In this work, we propose a novel verification-driven interpretation-in-the-loop framework \rt~to develop trustworthy DRL models, which are guaranteed to meet the expected constraint properties. Specifically, in each iteration, this framework measures the \textsl{gap} between the on-training model and predefined properties using formal verification, interprets the contribution of each input feature to the model's output, and then generates the training strategy derived from the on-the-fly measure results, until all predefined properties are proven. 
    Additionally, the low reusability of existing verifiers and interpreters motivates us to develop \rf, a general and fundamental tool within \rt~for DRL verification and interpretation. \rf~features \textsl{breakpoints} searching and verification-driven interpretation, associated with a concise constraint-encoding language \drlp.
    Evaluations demonstrate that \rt~outperforms the state-of-the-art on work six public benchmarks in both performance and property guarantees.
    Our framework can be accessed at https://github.com/Kurayuri/Reinfier. 
\end{abstract}

\setlength{\abovedisplayskip}{4pt}
\setlength{\belowdisplayskip}{5pt}
\input{sections/Introduction}

\input{sections/Background}

\input{sections/Overview}
\input{sections/Reinfier}

\input{sections/Reintrainer}
\input{sections/Implementation}
\input{sections/Evaluation}

\input{sections/RelatdWork}
\input{sections/Conclusion}

\section*{Impact Statement}
The potential positive societal impacts of this paper lie in providing verifiable and interpretable safety guarantees for DRL systems. This could lead to an increasing number of DRL designers developing and designing systems based on this research, resulting in safer and more efficient systems and advancing the field's applications.
However, the potential negative societal impacts of this paper might also result in over-reliance on and excessive trust in the verification and interpretation results. Unknown issues may exist that could lead to erroneous outcomes, necessitating the use of multiple methods and approaches for thorough verification.

\nocite{langley00}

\bibliography{bibliography}
\bibliographystyle{icml2025}

\newpage
\appendix
\onecolumn

\input{supsec/SupProof.tex}
\input{supsec/SupDRLP.tex}

\input{supsec/SupReinfier.tex}
\input{supsec/SupAlg.tex}
\input{supsec/SupEvaluation.tex}
\input{supsec/SupImplementation.tex}

\end{document}

%% file: sections/Introduction.tex
\section{Introduction}\label{Intro}

Recently, deep reinforcement learning (DRL) has demonstrated exceptional advancements in the development of intelligent systems across diverse domains, including game~\cite{AlphaGo}, networking~\cite{Network}, and autonomous driving~\cite{drive}. Nevertheless, the inherent opacity of the deep neural network (DNN) decision-making process necessitates verifiable and interpretable safety and robustness before the practical deployment~\cite{tnf23,tnf44}. Therefore, significant efforts have been invested in learning approaches integrated with formal verification to develop property-guaranteed DRL systems.

DRL properties are typically defined in terms of safety (informally, the agent never takes undesired actions), liveness (informally, the agent eventually takes desired actions), and robustness (informally, the agent takes similar actions with perturbation in the environment)~\cite{mc}. 
An important approach to achieving interpretability is by addressing interpretability problems, which can evaluate the behaviors of the trained models and provide useful insight into their underlying decision-making mechanisms.
Wildly concerned problems include decision boundary (identifying the boundary between different inputs that leads to different actions)~\cite{iw26,iw43}, counterfactual explanation (identifying the counterfactual inputs that make a property unsatisfied, which highlight the system’s weaknesses)~\cite{iw28,iw29}, sensitivity and importance analysis (identifying which input feature plays a central role in the decision-making process)~\cite{iw27,iw8}. 

However, developing verifiable and interpretable DRL systems remains a technical challenge.
Firstly, verification results often serve only as adversarial counterexamples~\cite{Attack,trainify,cegar,attack2} or just indicate whether the training process should be prolonged, leading to inefficient interaction and lengthy training time. 
For instance, the verification-in-the-loop approaches, where the model is alternately trained and verified against properties in each iteration until all predefined properties are proven~\cite{trainify,tnf_r1,tnf_r2,tnf_r3}, rely on a single counterexample derived from time-consuming verification for property learning, making the process inefficient.
Secondly, our evaluation results reveal that existing approaches fall short in terms of both property satisfaction and reward performance.
Lastly, interpretability, a widely-concerned factor in developing trustworthy DRL systems, is overlooked.

In this paper, we propose a novel verification-driven interpretation-in-the-loop framework \rt~ (Reinforcement Trainer) that can develop reliable DRL systems with verifiable and interpretable guarantees. 
Specifically, we design the \textsl{Magnitude and Gap of Property Metric} algorithm integrated with reward shaping~\cite{rs} strategies during training.
This algorithm leverages the \textsl{gap} (the difference between the current training model and the predefined properties), and the \textsl{density} (each input feature's contribution metric to the model's output based on interpretability), to measure the magnitudes of violations, i.e. the distances between the boundary of the property-constrained spaces and the states where violations occur. 
We measure the \textsl{gap} through introducing \textsl{breakpoint} searching, where \textsl{breakpoints} are the properties whose adjacent parameters' differences lead to the reverse verification results, i.e., from \texttt{Proven} to \texttt{Falsified} or vice versa. The identified \textsl{breakpoints} corresponds to the strictest states where the model can maintain predefined safety, liveness, or robustness, providing accurate and persuasive metrics in evaluating how safe and robust the on-training DRL model is.


Meanwhile, the existing verifiers and interpreters~\cite{verily,whiRL,whiRL_2.0,verisig2.0,uint} lack reusability for DRL system designers and are limited in the types of properties and interpretability questions they can address, which hinders their effectiveness and applicability. To address this, we design and implement a verifier and interpreter named \rf~(Reinforcement Verifier) that serves as the backend of \rt~or can be used as an independent tool. \rf~can answer various interpretability questions by searching for \textsl{breakpoints} effectively.
Although we can use programming languages like Python to encode DRL properties, interpretability questions, and training constraints separately, the additional learning cost and workload from hardcoding and repetitive coding are significant. Therefore, we present \drlp, a language that allows encoding them in a single, independent, configurable script with a unified format and clean, straightforward syntax.

We implement \rt~in Python, seamlessly integrable with popular vanilla DRL algorithms such as DDPG~\cite{ddpg}.
Different from the state-of-the-art (SOTA) verification-in-the-loop approach~\cite{trainify}, resizing the DNN with a specialized data structure to map the environment state to the DNN input, \rt~is completely decoupled from the DRL model structures and thus deployment-friendly. 
Furthermore, the training process and trained model of \rt~are inherently interpretable and are adapted to endorsed DRL training frameworks, such as Stable-Baselines~\cite{sb3}, making the transformation of vanilla training 
into property-constrained training a non-burdensome task. 

We evaluate \rt~on six classic control tasks in public benchmarks. For each task, we train DRL systems under identical settings in \rt, the SOTA approach~\cite{trainify}, and the corresponding vanilla DRL algorithm. 
Experimental results show that the systems trained in \rt~present superior reward performance and satisfaction guarantee of properties.
We also assess \rf~on a widely used case study, yielding valuable insights.

Collectively, our contributions are enumerated as follows:
\begin{itemize}[leftmargin=*]

\item  \rt, a novel verification-driven interpretation-in-the-loop framework with the \textsl{Magnitude and Gap of Property Metric} algorithm to develop reliable DRL systems.

\item  \rf, a verifier and interpreter featuring \textsl{breakpoints} searching and a cohesive methodology to address interpretability problems using identified \textsl{breakpoints}.


\item \drlp, a concise and unified language for DRL properties, interpretability questions, and training constraints, designed for simplicity and readability. 


\item Evaluation of our implementations, demonstrating that \rt~outperforms the SOTA on public benchmarks, while \rf~effectively verifies key properties and provides valuable interpretability insights.
    
\end{itemize}

%% file: sections/Background.tex
\section{Preliminaries}

\begin{figure*}[t]
\centering
\includegraphics[width=\textwidth]{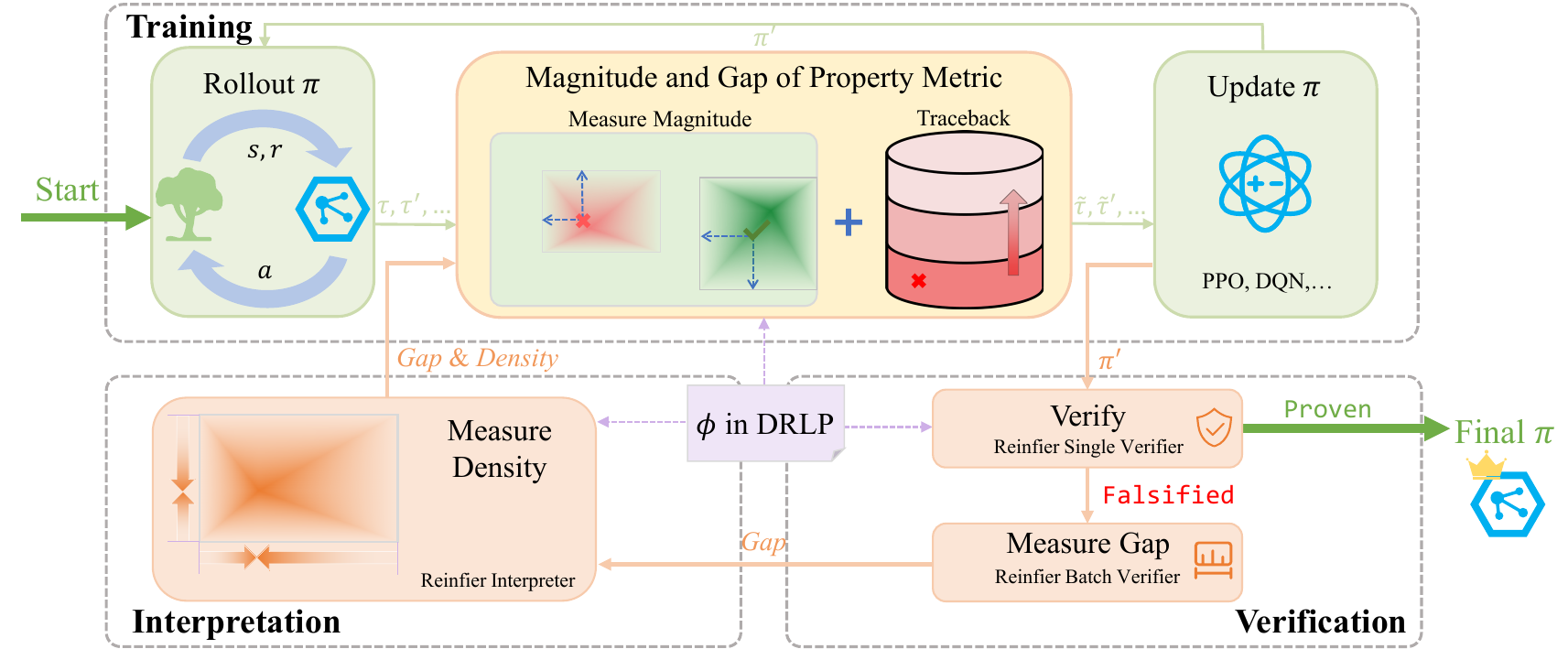}
\caption{Reintrainer architecture. }
\label{Reintrainer}
\end{figure*}


\paragraph{Markov Decision Process (MDP).}
An MDP is defined as $M=\langle S,A,T,\gamma,r \rangle$. $S$ is the state space; $A$ is the action space; 
$T=\{P(\cdot|s,a):s\in S,a \in A\}$ is the transition dynamics, and $P(s'|s,a)$ is the probability of transition from state $s$ to state $s'$ when the action $a$ is taken; $r:S\times A\to\mathbb{R}$ is the reward function; $\gamma\in[0,1)$ is the discount factor. $I\subseteq S$ is the initial state space. 
A policy $\pi:S\to A$ defines that
the decision rule the agent follows. The goal in DRL is to learn an optimal policy maximizing the expected discounted sum of rewards, i.e., $\argmax_\pi\mathbb{E}_{\tau\sim\pi}[\sum_{t=0}^\infty\gamma^t r(s_t,a_t)]$, where $t$ demotes the timestep, $\tau$ denotes a trajectory $(s_0,a_0,s_1,...)$, and $\tau\sim\pi$
 indicates the distribution over trajectories depends on the policy $\pi$.


\paragraph{Formal Verification of DRL Systems.}
Given a DRL system whose policy is modeled by a DNN $N$, a property $\phi$ defines a set of constraints over the input $x$ and the output $N(x)$ (or denoted by $y$).
Verifying such a system aims to prove 
or falsify
: $\forall x \in \mathbb{R}^n, P(x)\Rightarrow Q(N(x))$ (for all $x$, if pre-condition $P(x)$ holds, post-condition $Q(N(x))$ is also satisfied)~\cite{drlvs}. 
For simplicity, the agent's policy $\pi$ is consistent with its DNN $N$, and the  state $s$ input to the policy is consistent with the input $x$ of the DNN.
In this paper, $k$ denotes the verification depth (the maximum timestep limitation considered in verification), and $n$ denotes the number of the state features.

In the properties of formal verification (e.g., $\phi\equiv s\in[\langle 1,2\rangle,\langle 3,4\rangle] \Rightarrow a\in[0,+\infty)$), the property-constrained spaces consist of the constraint intervals (e.g., $\phi[s]=[\langle 1,2\rangle,\langle 3,4\rangle]$, where $\phi[s]$ denotes the constraint interval of the state $s$ in property $\phi$, indicating the 0th feature of the state $s^0\in[1,3]$ and the 1st feature $s^1\in[2,4]$), rather than concrete system states and actions (e.g., $s=\langle 1,2\rangle$). Further, $\underline{\phi}[s]$ donates the lower bounds of the constraint interval of the state, while $\overline{\phi}[s]$ donates the upper bounds.
The properties of verification can be formulated using satisfiability modulo theories (SMT) expressions, and the sets and functions defined in $M$ can be treated as predicates; for example, when $s\in S$, predicate $S(s)$ is \texttt{True}, or when $T(s'|s,a)>0$, predicate $T(s,s')$ is \texttt{True}.



%% file: sections/Overview.tex
\section{Framework Overview}\label{sec:Framework}


Fig.~\ref{Reintrainer} shows the overview of \rt~consisting of three sequential stages in every iteration: 

\textbf{(i) Training stage}: 
This stage consists of several training epochs. After several epochs, the model is updated and we move to the verification stage. Specifically, in every epoch, the model is rolled out in the environment, and we obtain set of trajectories $\{\tau_i\}$ stored in the buffer.
Subsequently, within the experiences of each trajectory, we measure the magnitude of violation or satisfaction of the predefined properties in DRLP and trace back the sequence leading to violation or satisfaction. Accordingly, we modify the corresponding reward of each experience following our generated reward shaping strategy, which is a self-adaptive algorithm to shape reward. Finally, we update the model based on DRL algorithms such as PPO~\cite{ppo}, and proceed to the next epoch. In addition to reward shaping, our metric result might also be considered as the cost in Constrained MDP to integrate with advanced optimization algorithms in the future.

\textbf{(ii) Verification stage}: 
We verify the predefined DRLP properties through \rf~on the model. If all properties are proven and the model is converged, we stop the iterations and return the model. Otherwise, we measure the \textsl{gap} between the model and predefined properties.

\textbf{(iii) Interpretation stage}: 
We interpret the model through \rf~to measure the \textsl{density} of state features within the constraint space defined by given properties in DRLP. Later, we generate a reward shaping strategy based on the \textsl{density} and the \textsl{gap} for the training, and then we return to the training stage to resume training the model.

Our measure and modification are entirely decoupled from the vanilla algorithms. The integration of training, verification, and interpretation seamlessly constitutes a verification-driven interpretation-in-the-loop DRL approach. After several iterations, \rt~can develop well-performing, property-compliant, verification-guaranteed, and interpretable models.



%% file: sections/Reinfier.tex
\section{Design of Verifier and Interpreter}\label{sec:Reinfier}

\begin{figure*}[t]
\centering
\includegraphics[width=0.9\textwidth]{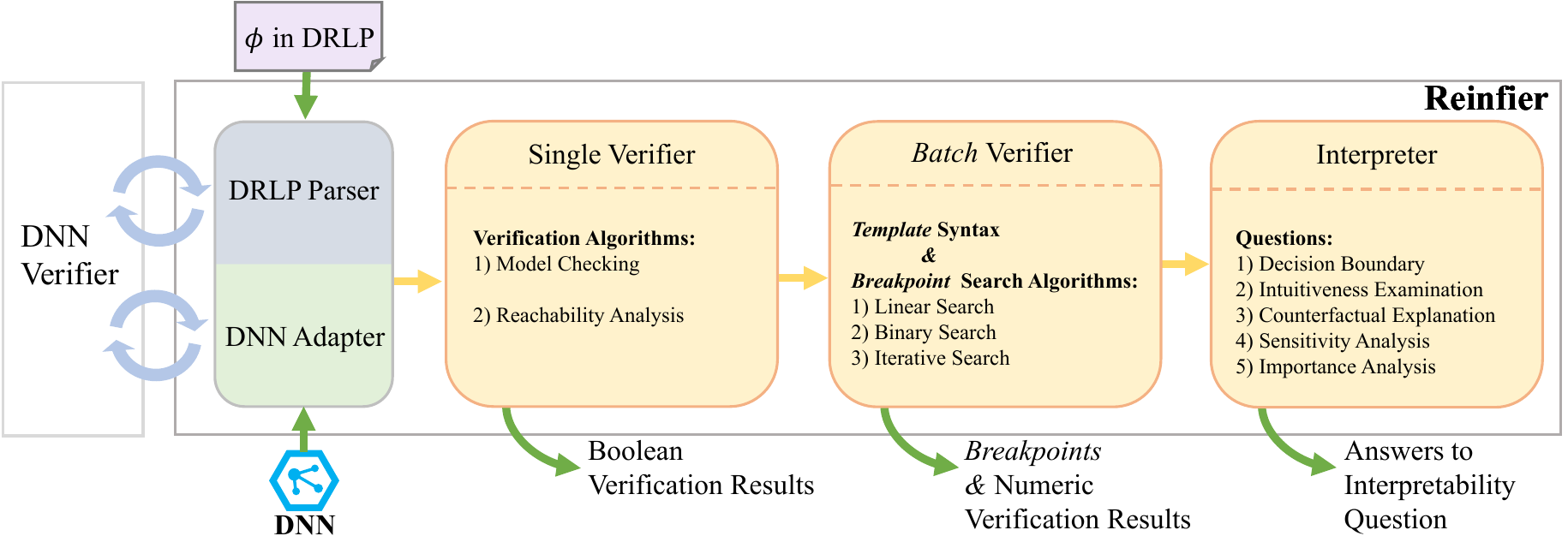}
\caption{Working procedure of Reinfier. }
\label{Reinfier}
\end{figure*}

We design and implement \rf, which not only achieves functions satisfying \rt's requirements but also addresses the challenges of reusability and completeness in terms of the type of properties as well as interpretability questions. Therefore, it can also be used as an independent tool for DRL verification and interpretation. We introduce novel concepts in the main text, while the detailed functional and technical design is provided in Appendix~\ref{Appx:Reinfier}.

Fig.~\ref{Reinfier} shows the working procedure of \rf. To avoid reinventing the wheel, 
we opt to use existing DNN verifiers DNNV~\cite{DNNV}, Marabou~\cite{Marabou} and Verisig~\cite{verisig2.0} as the DNN verification backend.
A parser for given properties in DRLP and an adapter for given DNN models designed and implemented by us function as intermediaries, bridging \textsl{Reinfier} with the DNN verifiers. We implement the single verifier by realizing model-checking algorithms or using reachability analysis, obtaining conventional Boolean DRL verification results (whether properties hold or not) derived from previous DRL verification works~\cite{whiRL,verisig2.0}.
Building upon this foundation, we introduce the novel \textsl{batch} verifier, featuring \textsl{breakpoints} searching on the DRLP templates. The numeric DRL verification results from our \textsl{batch} verifier provide a more accurate measure of the DRL model's safety and robustness. Additionally, our interpreter for answering interpretability questions features a cohesive methodology and corresponding solutions based on the identified \textsl{breakpoints}.

\subsection{Descriptive Language: DRLP}

\if false
\textbf{and Other condition $\bm{C}$.} In $\bm{C}$, 
extra constraints related to specified properties are encoded.
This part stands apart from the $\bm{S},\bm{T},\bm{I}$ parts usually describing the unchanging characteristics of a particular DRL system. For example, some liveness properties require no \texttt{bad} state cycle existing~\cite{mc,whiRL}; and this constraint should be considered as an extra constraint due to its exclusive association with this particular property, rather than a characteristic of the system.
We believe that distinct division in this way is beneficial for users to formulate and articulate properties.
\fi

The detailed syntax of DRLP is defined in Appendix~\ref{Appx:DRLP}.
A DRLP verification script can be structurally divided into three segments.
The first segment initializes a series of pre-assigned variables used for the verification.  
The second segment, starting with a delimiter \texttt{@Pre}, is a set of the pre-conditions $P$ (the prerequisite of properties).
The third segment, starting with a delimiter \texttt{@Exp}, is a set of the post-conditions $Q$ (the expected results). 
In the script, \texttt{x} denotes the input of DNN, \texttt{y} denotes the output of DNN, and \texttt{k} denotes the verification depth; variables prefixed with ``\_'' are categorized as \textsl{iterable variables}. 
Fig.~\ref{DRLPExample} shows four verification scripts encoded with the properties and a DRL system with two state features and one continuous action.

\begin{figure*}[t]
\centering
\includegraphics[width=\textwidth]{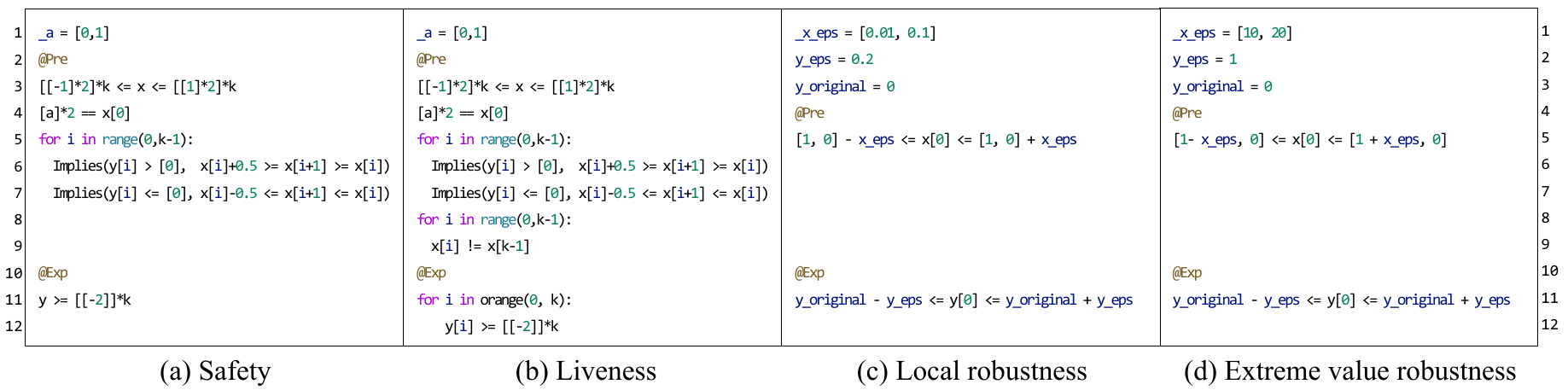}
\caption{Illustration of four verification scripts characterizing the properties of (a) safety, (b) liveness, (c) local robustness and (d) extreme value robustness.
} \label{DRLPExample}
\end{figure*}

\textbf{Safety and Liveness Property.} 
A verification script on the safety property is shown in Fig.~\ref{DRLPExample}(a).
We first detail pre-conditions (Line 3-7).
Each state feature $x$ belongs to the interval $[-1, 1]$ (Line 3).
Considering two cases that \texttt{a} equals $0$ or $1$ (Line 1), the value of each initial state feature is respectively $0$ or $1$ (Line 4).
Note that each value of the next state increases by at most $0.5$ if the action is greater than $0$ (Line 5,6), and decreases by at most $0.5$ otherwise. 
The post-condition is that the action should always be not less than $-2$ (Line 11). 
Different from the safety one, the liveness property (Fig.~\ref{DRLPExample}(b)) enforces the additional pre-conditions that no state cycle exists (Line 8,9) and the post-condition is that the action should be eventually not less than $-2$ (Line 11,12).

\textbf{Robustness Property.} 
A verification script on the local robustness property is
shown in Fig. 3(c). The pre-condition is that considering two degrees of perturbations \texttt{x\_eps} (equals $0.01$ or $0.1$) respectively (Line 1), each input feature is perturbed accordingly
(Line 5). The post-condition (Line 11) is that the output's deviation to its original output \texttt{y\_original} (equals 0) should not exceed \texttt{y\_eps} (equals $0.2$). 
Another script on extreme value robustness is shown in Fig. 3(d). 
The pre-condition is that considering two degrees of outliers \texttt{x\_eps} (equals $10$ or $20$) respectively (Line 1), only the first input feature is the outlier while the second is not (Line 5). The post-condition (Line 11) is that the output's deviation to its original output \texttt{y\_original} (equals 0) should not exceed \texttt{y\_eps} (equals $1$).

\subsection{Batch Verifier} 

\begin{algorithm}[htb]
    \small
    \caption{FindBreakpoints}
    \label{alg:find_bps}
    
    \textbf{Input}: DRLP template $p$, DNN $N$, variable list $\{var\}$\\
    \textbf{Output}: Breakpoint list $\{bp\}$ 
    \begin{algorithmic}[1]
        \STATE $var \gets $ pop the top variable from $\{var\}$
        \STATE $lb \gets $ the lower bound of $var$, $ub \gets $ the upper bound of $var$
        \STATE $prec \gets $ the search precision of $var$
        \STATE 
        \IF{$\{var\}$ is empty }
        \STATE $lb\_r,violation \gets$ verify(concretize$(p,var,lb),N$)
        \STATE $ub\_r,violation \gets$ verify(concretize$(p,var,ub),N$)
            \WHILE{$ub- lb \ge prec$ \AND $lb\_r \neq ub\_r$}   
                \STATE $curr \gets \frac{lb+ub}{2}$
                \STATE $script \gets$ concretize$(p,var,curr)$
                
                \STATE $curr\_r, violation \gets $verify$(script,N)$
                \IF{$curr\_r = \texttt{Falsified}$}
                    \STATE $curr \gets $ the value of $var$ in $violation$
                \ENDIF
                \STATE \textbf{if} $curr\_r = lb\_r$ \textbf{then}: $lb \gets curr$
                \STATE \textbf{else if} $curr\_r = ub\_r$ \textbf{then}: $ub \gets curr$
            \ENDWHILE
            \STATE Add $script$'s corresponding property to $\{bp\}$

        \ELSE
            \WHILE{$ub- lb \ge prec$}   
        
                \STATE $script \gets$ concretize$(p,var,lb)$
            
                \STATE FindBreakpoints$(script,N,\{var\})$
                \STATE $lb \gets lb+prec$
            \ENDWHILE
        \ENDIF
    \end{algorithmic}
\end{algorithm}

A DRLP template is a type of DRLP script characterized by the inclusion of inconcrete parameters (i.e., variables rather than specific values) and is typically derived from a qualitative property. 
The scripts depicted in Fig.~\ref{DRLPExample} excluded by the first segment can be viewed as templates.
By iterable variables or feeding a sequence of parameter values, each template can be concretized (replacing variables with specific values) into concrete scripts for multiple cases.

We introduce the \textsl{breakpoints} search algorithm, as outlined in Algorithm~\ref{alg:find_bps}, which leverages single verification along with a combination of binary and linear search to identify \textsl{breakpoints} on DRLP templates. The algorithm's general form is in Appendix \ref{appx:batch verfier}.
The identified \textsl{breakpoints} correspond to the strictest states where the model can maintain predefined properties.
For example, the target model should satisfy the safety property $\phi \equiv x > 0 \Rightarrow N(x) > 0$. However, identified \textsl{breakpoints} reveals that initially, the strictest property that the model satisfies is $\phi \equiv x > 0 \Rightarrow N(x) > -5$, and then advances to $\phi \equiv x > 0 \Rightarrow N(x) > -2$. These numeric verification results extracted from the parameters in the properties, specifically -5 and -2, demonstrate that the model is progressively becoming safer, indicating the efficacy of the preceding training strategy.

\subsection{Interpreter}\label{Sec:interpret}
Our proposition introduces a cohesive methodology to these problems, fundamentally reliant on the identified \textsl{breakpoints}. This methodology has the potential to address a broader spectrum of interpretability concerns. Specifically, we propose novel \textsl{breakpoints}-based solutions for five interpretability problems, including decision boundary, intuitiveness examination, counterfactual explanation, sensitivity analysis, and importance analysis. Except for the sensitivity analysis, which is related to the metric of \textsl{density}, the other four solutions are provided in Appendix \ref{appx:interpreter}.

\subsubsection{Sensitivity Analysis} 
Given the original input $\hat{x_0}$ and the perturbation $\varepsilon$, the sensitivity analysis question~\cite{iw27} aims to evaluate the distances between $N(x_0)$ and $N(\hat{x_0})$, where $x_0$ denotes the input whose features under discussion are subjected to the perturbation $\varepsilon$. Sensitivity analysis conducted on all input features reveals which specific feature, with a higher sensitivity result, has the potential to induce greater fluctuations in the output.

\textbf{Our solution: } 
Find \textsl{breakpoints} by stepping through possible $y_0$ on the property formulated as:
\begin{equation}
\begin{split}
             \forall x_0, &\bigwedge_{j \in D} (\hat{x_{0}^{j}} -\varepsilon \le x_0^j \le \hat{x_{0}^{j}} +\varepsilon) \wedge \bigwedge_{j \in R} (\hat{x_{0}^{j}}=x_0^j) \\
             &\Rightarrow N(x_0) \approx y_0
\end{split}
\end{equation}
where $D$ denotes the set of indices of the discussed features within the input, while $R$ denotes the set of indices of other features.
Finally, the answer is the minimum distance between $y_0$ and $N(\hat{x_0})$, i.e., sensitivity, formulated as: $\min_{bp\in \{bp\}} d(y_0,N(\hat{x_0}))$.




%% file: sections/Reintrainer.tex
\section{Magnitude and Gap of Property Metric Algorithm}\label{sec:Alg}



The hypothesis we advocate posits that penalties (reward reductions applied to discourage undesired actions) are directly proportional to the magnitude of violations.
Furthermore, this proportionality extends to the deviation from current states 
to those states where the property holds.
This deviation can be equivalently described as the distance between the boundary of the property-constrained spaces and the states where violations occur. 
The theoretical proof that our hypothesis and algorithm for property learning do not affect the optimality of the policy is provided in Appendix \ref{Proof:potentiol}. Additional details are provided in Appendix \ref{Appx:Reintrainer}.



\subsection{Distance of States}
\paragraph{Distance density.}


The concept of sensitivity analysis in Sensitivity Analysis 
inspires us that the fluctuations in the output could be different for the same input feature at different values, thus we introduce the distance 
\textsl{density} considering that perturbations only need evaluating in one direction at the boundary: given the DNN $N$ and the property $\phi$, when the input feature $s_i^j$ is subject to perturbation $\varepsilon$, the \textsl{density} of the lower bound $\underline{\rho}(\phi,s_i^j)$ can be formulated as:
\setlength{\abovedisplayskip}{4pt}
\setlength{\belowdisplayskip}{5pt}
\begin{equation}\label{eq:rho}
            \underline{\rho}(\phi,s_i^j) = 
            \max_{s} d(N(\underline{\phi}[s_i]),N(s))
\end{equation}
where $s\in[\langle\underline{\phi}[s_i^0],...,\underline{\phi}[s_i^j],...,\underline{\phi}[s_i^{n-1}]\rangle,\langle\underline{\phi}[s_i^0],...,\underline{\phi}[s_i^j]+\varepsilon,...,\underline{\phi}[s_i^{n-1}]\rangle]$.
The upper bound \textsl{density} function $\overline{\rho}$ can be defined similarly.

\paragraph{Exact middle point.}
The exact middle point, denoted as $\dot{\phi}[s_i^j]$, represents the position with the maximum distance to the boundary, and also the point where the \textsl{density} of the lower and upper boundaries switches. 
Consequently, the exact middle point $\dot{\phi}[s_i^j]$ is the midpoint of the \textsl{density}-weighted upper and lower bounds, formulated as:
\setlength{\abovedisplayskip}{4pt}
\setlength{\belowdisplayskip}{5pt}
\begin{equation}
    \dot{\phi}[s_i^j] = \frac{
    \underline{\rho}(\phi,s_i^j) \cdot \underline{\phi}[s_i^j] +
    \overline{\rho} (\phi,s_i^j) \cdot \overline{\phi}[s_i^j] }{
    \underline{\rho}(\phi,s_i^j) +
    \overline{\rho} (\phi,s_i^j)  }
\end{equation}


\paragraph{Distance of States in $1$-Dimension}
Given the state feature $s_i^j$ whose current observation is $\mathtt{v}$, its normalized distance $dist$ to its property-constrained space $\phi[s_i^j]$ in $1$-dimension is formulated as:
\begin{equation}\label{eq:mid}
    dist(\phi[s_i^j],\mathtt{v}) = \left\{
    \begin{array}{ll}
        (\frac{\mathtt{v}-\underline{\phi}[s_i^j]}{\dot{\phi}[s_i^j]-\underline{\phi}[s_i^j]})^{p_1} & \text{if } \mathtt{v} \in [\underline{\phi}[s_i^j],\dot{\phi}[s_i^j]]  \\
        (\frac{\overline{\phi}[s_i^j]-\mathtt{v}}{\overline{\phi}[s_i^j]-\dot{\phi}[s_i^j]})^{p_1} & \text{if } \mathtt{v} \in (\dot{\phi}[s_i^j],\overline{\phi}[s_i^j]]  \\
        0 & \text{else} 
    \end{array}
    \right.
\end{equation}
where $p_1$ represents the penalty intensity coefficient, and we typically set $p_1$ to $1$.



\paragraph{Distance of States in $n$-Dimension}

Considering the differences in input feature to affect the outputs, we employ a \textsl{density}-weighted $\ell^{p_2}$-norm. Given the state $s$ whose current observation is $\mathtt{s}$, its distance $Dist$ to the property-constrained space $\phi$ is formulated as:
\begin{equation}\label{eq:dist}
    Dist(\phi,\mathtt{s})= \frac{\sqrt[p_2]{\sum[{\rho(\phi,s_i^j)\cdot(dist(\phi[s_i^j],\mathtt{s}_i^j))^{p_2}}]}}{\sum{\rho(\phi,s_i^j)}}
\end{equation}
where $p_2$ is typically set to 2; $\rho(\phi,s_i^j)$ equals $\underline{\rho}(\phi,s_i^j)$ when $\mathtt{s}_i^j<\dot{\phi}[s_i^j]$, otherwise, it equals $\overline{\rho}(\phi,s_i^j)$.



\subsection{Metric of Violation's or Satisfaction's Magnitude.}
For the vast majority of cases, setting the distance of actions to a fixed value is sufficient. However, similarly, it can also be calculated as the distance of states.
Given the observed environmental state $\mathtt{s}$ and the action taken $\mathtt{a}$, we measure the magnitude of violation($-$) or satisfaction($+$) when the property $\phi$ is violated or satisfied, which is formulated as:
\begin{equation}
    Diff(\phi,\mathtt{s},\mathtt{a})=\pm Dist (\phi,\mathtt{s})\cdot Dist (\phi[a],\mathtt{a})
\end{equation}
\subsection{Metric of the Gap}
We introduce the concept of \textsl{gap} to quantitatively measure the divergence between the anticipated parameter specified in the predefined property and the identified \textsl{breakpoints} from the on-training model. 
The high computational complexity of finding \textsl{breakpoints} confines the \textsl{gap} measurement to just one parameter $z$ from the state features or the action.

\begin{definition}\label{def:re}
Property  $\Grave{\phi}$ is a relaxation of property $\phi$, if and only if $\phi \Rightarrow \Grave{\phi}$.
\end{definition}

\begin{theorem}\label{the:re}
Property  $\Grave{\phi}\equiv\Grave{P} \Rightarrow \Grave{Q}$ is a relaxation of property $\phi\equiv P\Rightarrow Q$, if $\Grave{P}\subseteq P, \Grave{Q}\supseteq Q$.

\end{theorem}

In this context, we find \textit{breakpoints} by stepping through $z$ within the $\phi$'s relaxation domain $\Grave{\Phi}_{\phi\equiv P\Rightarrow Q}=\{\Grave{P} \Rightarrow \Grave{Q}: \forall \Grave{P}, \Grave{Q},\Grave{P}\subseteq P, \Grave{Q}\supseteq Q \}$ on DNN $N$. Subsequently, $\{bp\}$ denotes the identified \textsl{breakpoints}, and the \textsl{gap} function is formulated as:
\begin{equation}
g(\phi,N)=\min_{\Grave{\phi}\in \{bp\}} |\phi[z]-\Grave{\phi}[z]|
\end{equation}
The \textit{gap} metric can serve as an indicator for property learning and can also be combined with a learning rate schedule function $Lr$ as: $ F_t = Lr(g(\phi,N))\times Diff(\phi,\mathtt{s},\mathtt{a})$.

\subsection{Traceback}
For properties that span multiple steps, including multi-step safety and liveness, 
violation occurring at a particular step implies increasing unsafety in the preceding trajectory. Therefore, it is imperative to trace back to adjust the reward trajectories in the buffer before policy optimization. For action avoidance property, we calculate the traced intermediate reward $\Tilde{F}_{t}$ backward from the last trajectory in the buffer 
with a discount factor $\mu\in[0,1)$ 
as: $\Tilde{F}_{t} = F_t - \lambda^{-1}F_{t-1}+ \mu \Tilde{F}_{t+1} $; for destination reach property, $\Tilde{F}_{t} = \lambda F_{t+1} - F_t  +\mu \Tilde{F}_{t+1}$.
The final shaped reward given to the agent $\Tilde{r}_t$ is formulated as: $\Tilde{r}_t = r_t + \Tilde{F}_{t}$. 

%% file: sections/Evaluation.tex
\section{Evaluation}\label{sec:Evo}

The evaluation of \rf, the supplementary and the ablation studies of \rt~are provided in Appendix~\ref{Appx:Eva}.



\input{sections/Evaluation_Reintrainer.tex}

%% file: sections/Evaluation_Reintrainer.tex
\paragraph{Benchmarks and Experimental Settings.}\label{sec:bench}
\input{tab/Reintrainer/Properties.tex}
We evaluate \rt~on the same six public benchmarks
used in Trainify~\cite{trainify}. 
Mountain Car (MC)~\cite{mountaincar} task consists of a car placed at the bottom of a valley, and its goal is to accelerate the car to reach the top of the right hill.  
Cartpole (CP)~\cite{cartpole} task consists of a pole attached by an unactuated joint to a cart, which moves along a frictionless track, and its goal is to balance the pole. 
Pendulum (PD)~\cite{gym} task consists of a pendulum attached at one end to a fixed point while the other end is free and starts in a random position, and its goal is to apply torque on the free end to swing it into an upright position. 
B1 and B2~\cite{verisig2.0,b1b2} are two classic nonlinear systems, where agents in both systems aim to arrive at the destination region from the preset initial state space.
Tora~\cite{tora} task consists of a cart attached to a wall with a spring,
and inside which there is an arm free to rotate about an axis; its goal is to stabilize the system at the stable state where all the system variables are equal to 0.

All experiments are conducted on a workstation with two 8-core CPUs at 2.80GHz and 256 GB RAM.
We adopt the same training configurations for each task respectively.




For each task, we predefine certain properties related to their safety or liveness. All these properties, along with their types and meanings, are presented in Table~\ref{tab:train_prop}, and their definitions are provided in Appendix~\ref{Appx:rf_prop}.
\rt~ can automatically translate DRLP scripts, encoded from predefined properties, into violation measure strategies at runtime. 

We employ the DQN algorithm~\cite{dqn} for tasks MC and CP with discrete actions, while we employ the DDPG algorithm~\cite{ddpg} for the other four tasks with continuous actions. For tasks MC, CP, and PD, all training hyperparameters for vanilla algorithms are finely tuned to achieve their optimal performance, and \rt~directly employs these same hyperparameters. We train two specific types of Trainify: Trainify\#1 employs the hyperparameters from its original work, and Trainify\#2 employs the same hyperparameters as the vanilla algorithms. 
For tasks B1, B2 and Tora, we employ the total same hyperparameters from the work of Trainify for Reintrainer, vanilla algorithms and Trainify\#1.

\paragraph{Final Model Comparison.}\label{sec:rt_final_model}

{

\input{tab/Reintrainer/AccReward.tex}

}
We
evaluate the final models in their environments for 100 episodes, and calculate the average and standard deviation of the cumulative rewards, as shown in Table~\ref{tab:train_ans}.
The results demonstrate that our approach significantly outperforms both types of Trainify and even fine-tuned vanilla algorithms in all cases.
We also formally verify the predefined properties on the final models. The results in Table~\ref{tab:train_ans} show that the vanilla algorithms do not guarantee these properties to be satisfied, while in some cases the property is naturally satisfied.
The final models trained by Trainify fail to satisfy certain properties, a shortfall we attribute to the inherent limitations of its method, which primarily focuses on subdividing the abstract state space based on counterexamples but lacks mechanisms to prevent the agent from taking unexpected actions, a critical aspect for ensuring safety. In contrast, Trainify performs well in terms of reachability properties, such as $\phi_4$ and $\phi_5$, enabling the agent to reach specific states. 
The evaluation results of the final models reveal that our approach guarantees property satisfaction without incurring performance penalties compared to the vanilla algorithms. 
Additionally, our approach outshines Trainify by demonstrating superior performance, guaranteeing satisfaction across a wider variety of property types, and high efficiency in terms of training timesteps.

%% file: tab/Reintrainer/Properties.tex
\setlength{\extrarowheight}{2pt}
\begin{table*}[]

\resizebox{\textwidth}{!}{%
\begin{tabular}{llll}
\hline
\textbf{Task} & \textbf{Type} & \textbf{ID} & \textbf{Meaning}                                                                                                                \\ \hline
MC            & Safety        & $\phi_{1}$  & When the car is moving right at high speed at the bottom of the valley, it should not accelerate to the left.                   \\ \hline
CP            & Safety        & $\phi_{2}$  & When the car is moving right in the left edge area and the pole tilts right at a large angle, it should be pushed to the right. \\ \hline
PD            & Safety        & $\phi_{3}$  & When the pendulum approaches the upright position of the target, large torque should not be applied.                            \\ \hline
B1            & Liveness      & $\phi_{4}$  & The agent always reaches the target eventually.                                                                                 \\ \hline
B2            & Liveness      & $\phi_{5}$  & The agent always reaches the target eventually.                                                                                 \\ \hline
Tora          & Safety        & $\phi_{6}$  & The agent always stays in the safe region.                                                                                      \\ \hline
\end{tabular}%
}
\caption{Predefined properties of benchmarks.}
\label{tab:train_prop}
\end{table*}

%% file: tab/Reintrainer/AccReward.tex
\setlength{\extrarowheight}{1pt} 
\begin{table*}[]

\resizebox{\textwidth}{!}{%
    \begin{tabular}{llclclclclclclc}
    \hline
    Task            & \multicolumn{2}{c}{MC}                          & \multicolumn{2}{c}{CP}                          & \multicolumn{2}{c}{PD}                          & \multicolumn{2}{c}{B1}                          & \multicolumn{2}{c}{B2}                          & \multicolumn{4}{c}{Tora}                                                                          \\ \hline
    Network         & \multicolumn{2}{c}{$2\times256$}                & \multicolumn{2}{c}{$2\times256$}                & \multicolumn{2}{c}{$[300,400]$}                 & \multicolumn{2}{c}{$2\times20$}                 & \multicolumn{2}{c}{$2\times20$}                 & \multicolumn{2}{c}{$3\times100$}                & \multicolumn{2}{c}{$3\times200$}                \\
                    & \multicolumn{1}{c}{$\Bar{R}\pm\sigma$} & V.R.         & \multicolumn{1}{c}{$\Bar{R}\pm\sigma$} & V.R.         & \multicolumn{1}{c}{$\Bar{R}\pm\sigma$} & V.R.         & \multicolumn{1}{c}{$\Bar{R}\pm\sigma$} & V.R.         & \multicolumn{1}{c}{$\Bar{R}\pm\sigma$} & V.R.         & \multicolumn{1}{c}{$\Bar{R}\pm\sigma$} & V.R.         & \multicolumn{1}{c}{$\Bar{R}\pm\sigma$} & V.R.         \\ \hline
    Reinfier (Ours) & -101.92$\pm$7.36                 & $\checkmark$ & 500.00$\pm$0.00                  & $\checkmark$ & -132.88$\pm$85.83                & $\checkmark$ & -119.14$\pm$0.99                 & $\checkmark$ & -20.77$\pm$0.95                  & $\checkmark$ & -50.00$\pm$0.00                  & $\checkmark$ & -50.00$\pm$0.00                  & $\checkmark$ \\
    Trainify\#1     & -103.81$\pm$13.21                & $\times$     & 500.00$\pm$0.00                  & $\times$     & -328.86$\pm$189.71                & $\checkmark$ & -122.00$\pm$0.00                 & $\checkmark$ & -23.72$\pm$0.69                  & $\checkmark$ & -50.00$\pm$0.00                  & $\times$     & -50.00$\pm$0.00                  & $\times$     \\
    Trainify\#2     & -116.21$\pm$4.00                 & $\times$     & 500.00$\pm$0.00                  & $\times$     & -354.52$\pm$189.94                & $\checkmark$ & -                                & -            & -                                & -            & -                                & -            & -                                & -            \\
    Vanilla         & -103.82$pm$9.37                  & $\times$     & 500.00$\pm$0.00                  & $\times$     & -142.90$\pm$86.21                & $\checkmark$ & -122.00$\pm$0.00                 & $\checkmark$ & -20.85$\pm$0.88                  & $\checkmark$ & -50.00$\pm$0.00                  & $\times$     & -50.00$\pm$0.00                  & $\times$     \\
\hline
    \end{tabular}%
    }
    \caption{Comparison of the final models. $\bm{\bar{R}\pm\sigma}$ stands for the average episode reward and its standard deviation of the final models' evaluation results. \textbf{V.R.} stands for verification result, where $\checkmark$ indicates that the predefined property is proven, while $\times$ indicates falsification.}
    \label{tab:train_ans}
\end{table*}

%% file: sections/RelatdWork.tex
\section{Related Work}

Numerous safe DRL methods have been proposed to ensure DRL safety. The mainstream safe DRL's methods can be categorized into three types: (1) modifying the reward function (“reward shaping”) to encourage safe actions, (2) incorporating an additional cost function (“constraining”) to constrain behaviors, and (3) preventing unsafe actions during running (“shielding”)~\cite{partialshielding}. 
Reward shaping~\cite{rs} has been used in RL for decades, and many works~\cite {cav1,cav2,cav3} contribute to training acceleration or property learning. The constrained optimization-based safe DRL methods focus on optimizing the agent’s behavior while adhering to safety constraints through Lagrangian relaxation~\cite{ppolag,omnisafe} and constrained policy optimization~\cite{cup,pcpo,cpo}. 
However, these two types of methods only aim to reduce unsafe actions during training and cannot ensure absolute safety during deployment.
Safe shields~\cite{trainify54,saferl} and barrier functions~\cite{trainify7,trainify53} are employed to prevent agents from adopting dangerous actions.
Although such type of methods can eliminate unsafe actions, the alternative actions lead to suboptimal performance.

Formal verification, unlike safe shielding methods, can assess the safety of trained models before deployment rather than during execution. It can integrate with, but is not limited to, the three mainstream safe DRL methods mentioned above.
Among these, verification-in-the-loop training, rather than train-then-verify, can more efficiently train models. The pioneering work~\cite{trainify36} proposes a correct-by-construction approach for developing Adaptive Cruise Control by defining safety properties and computing the safe domain. Trainify~\cite{trainify} develops DRL models driven by counterexamples, refining abstract state space, which limits model structure and is highly inefficient. COOL-MC~\cite{tnf_r1} combines a verifier and training environments without any safe learning method. The work~\cite{tnf_r2} leverages verifiers to select the best available policy and continues training on the chosen policy.

Moreover, interpretability is another method to indirectly enhance model safety. DRL interpretability focuses on explaining the decision-making rules of models~\cite{iw13,iw26,iw29,iw42} and analyzing model characteristics~\cite{iw27,iw28,iw8}. UINT~\cite{uint} first proposes that several interpretability questions can be formulated in SMT problem format, similar to verification. Our interpreter expands the types of questions and features an integration of more types of properties and problems.

The methods mentioned above do not explicitly address how to fully utilize verification and interpretation results, which should greatly aid DRL safety. Therefore, our main contribution is an efficient method to integrate verification interpretation and learning, leveraging insights from verification and interpretation results. Through open access implementation, we aim to encourage more DRL system designers to train verifiable, interpretable safe models.

%% file: sections/Conclusion.tex
\section{Conclusion}\label{sec:Conclusion and Limitations}
For multi-step properties, \rf~itself cannot address stochasticity or unknown environments as related works~\cite{whiRL,whiRL_2.0,verisig2.0,trainify}, which might limit its application scenarios; the former can be alleviated through over-approximation\cite{overapproximation}, while the latter can be addressed by environment modeling; and the lengthy verification time necessitates the improvement of verification techniques, like incremental verification. Besides, the current integration with reward shaping proves effective for simpler properties but might fall short for more intricate ones, thus integrating our framework with constrained optimization-based algorithms could enhance performance. 

We present a verification-driven interpretation-in-the-loop framework \rt, which can develop trustworthy DRL systems through property learning with verification guarantees and interpretability, without performance penalties.
Additionally, we propose a general DRL verifier and interpreter \rf~featuring \textsl{breakpoints} searching and \textsl{breakpoints}-based interpretation, and the language DRLP.
Evaluations indicate that \rt~surpasses the SOTA across six public benchmarks, achieving superior performance and property guarantees, while \rf~can verify properties and provide insightful interpretability answers.




%% file: supsec/SupProof.tex
\section{Theoretical Proofs}\label{Appx:proofs}
\subsection{Theorem 1 with Proof}

\begin{definition}
Property  $\Grave{\phi}$ is a relaxation of property $\phi$, if and only if $\phi \Rightarrow \Grave{\phi}$.
\end{definition}

\begin{theorem}
Property  $\Grave{\phi}\equiv\Grave{P} \Rightarrow \Grave{Q}$ is a relaxation of property $\phi\equiv P\Rightarrow Q$, if $\Grave{P}\subseteq P, \Grave{Q}\supseteq Q$.
\end{theorem}

\begin{proof}
\(\phi \Rightarrow \Grave{\phi}\) means that if \(P \Rightarrow Q\), then \(\Grave{P} \Rightarrow \Grave{Q}\) must also hold.



    
    

\begin{enumerate}
    \item \textbf{Assume \(P \Rightarrow Q\):}
    
    By the definition of \(\phi\), this means that for all \(x\), if \(x \in P\), then \(x \in Q\).
    
    \item \textbf{Consider \(\Grave{P} \subseteq P\):}
    
    By this inclusion, if \(x \in \Grave{P}\), then \(x \in P\).
    
    \item \textbf{Using the assumption \(P \Rightarrow Q\):}
    
    Given that \(\Grave{P} \subseteq P\), if \(x \in \Grave{P}\), then \(x \in P\). By the implication \(P \Rightarrow Q\), if \(x \in P\), then \(x \in Q\). Hence, if \(x \in \Grave{P}\), then \(x \in Q\).
    
    \item \textbf{Consider \(\Grave{Q} \supseteq Q\):}
    
    By this inclusion, if \(x \in Q\), then \(x \in \Grave{Q}\).
    
    \item \textbf{Combining the above steps:}
    
    Since if \(x \in \Grave{P}\), then \(x \in P\), and if \(x \in P\), then \(x \in Q\), and finally, if \(x \in Q\), then \(x \in \Grave{Q}\). Therefore, if \(x \in \Grave{P}\), then \(x \in \Grave{Q}\).
\end{enumerate}

Thus, we have shown that if \(\phi \equiv P \Rightarrow Q\) holds, then \(\Grave{\phi} \equiv \Grave{P} \Rightarrow \Grave{Q}\) also holds. Therefore, \(\phi \Rightarrow \Grave{\phi}\), proving that \(\Grave{\phi}\) is a relaxation of \(\phi\).

\end{proof}

\subsection{Theorem 1 with Proof}\label{Proof:potentiol}
\begin{definition}
A shaping reward function $F: S \times A \times S \rightarrow \mathbb{R}$ is \textbf{potential-based} if there exists potential function $\psi : S \rightarrow \mathbb{R}$ such that \[ F(s, a, s') = \gamma \psi(s') - \psi(s) \] for all $s \neq s_0, a, s'$.\cite{rs}
\end{definition}
\begin{theorem}\label{the:pro}
If $F$ is a potential-based shaping function, then every optimal policy in $M' = \langle S, A, T, \gamma, r + F \rangle$ will also be an optimal policy in $M = \langle S, A, T, \gamma, r \rangle $ and vice versa.\cite{rs}
\end{theorem}

\begin{proof}
The optimal $Q$-function \( Q^*_M(s, a) = \sup_\pi Q^\pi_M(\) \(s, a) \).
\( Q^*_M \) satisfies the Bellman equation:
\begin{equation*}
    Q^*_M(s, a) = \mathbb{E}_{s' \sim P(\cdot|s,a)} \left[ R(s, a, s') + \gamma \max_{a' \in A} Q^*_M(s', a') \right]
\end{equation*}
Let's subtract \( \psi(s) \) from both sides:
\begin{align*}
    &Q_M(s, a) - \psi(s) \\
    =&\mathbb{E}_{s' \sim P(\cdot|s,a)} \left[ R(s, a, s') + \gamma \max_{a' \in A} Q^*_M(s', a') \right] - \psi(s) \\
    =& \mathbb{E}_{s' \sim P(\cdot|s,a)} \biggl[  R(s, a, s') + \gamma \psi(s') + 
          \left. \gamma \max_{a' \in A} \left( Q^*_M(s', a') - \psi(s') \right) \right] - \psi(s) \\
    =& \mathbb{E}_{s' \sim P(\cdot|s,a)} \biggl[ R(s, a, s') + \gamma \psi(s')-  \left.\psi(s) + \gamma \max_{a' \in A} \left( Q^*_M(s', a') - \psi(s') \right) \right]
\end{align*}

Let
\begin{equation*}
    \hat{Q}_{M'}(s, a) := Q^*_M(s, a) - \psi(s).
\end{equation*}
and recall that
\begin{equation*}
    F(s, a, s') = \gamma \psi(s') - \psi(s).
\end{equation*}
Therefore,
\begin{align*}
    &\hat{Q}_{M'}(s, a) \\
    =& \mathbb{E}_{s' \sim P(\cdot|s,a)} \biggl[ r(s, a, s') + F(s, a, s') + \gamma \left. \max_{a' \in A} \left( \hat{Q}_{M'}(s', a') \right) \right] \\
    =& \mathbb{E}_{s' \sim P(\cdot|s,a)} \left[ r'(s, a, s') + \gamma \max_{a' \in A} \left( \hat{Q}_{M'}(s', a') \right) \right]
\end{align*}

This is the Bellman equation for \( M' \), so
\begin{equation*}
    \hat{Q}_{M'} = Q^*_{M'}.
\end{equation*}

\end{proof}

\begin{corollary}
    $\Tilde{F}_t$ does not affect the optimality of the policy.
\end{corollary}
\begin{proof} 
Recall that
\begin{equation*}
     Diff(\phi,\mathtt{s},\mathtt{a})=\pm Dist (\phi,\mathtt{s})\times Dist (\phi[a],\mathtt{a})
\end{equation*}
\begin{equation*}
    F_t=Lr(g(\phi,N))\times Diff(\phi,\mathtt{s},\mathtt{a})
\end{equation*}
Considering the violation condition, therefore,
\begin{equation*}
    F_t= - Lr(g(\phi,N))\times  Dist (\phi,\mathtt{s})\times Dist (\phi[a],\mathtt{a})
\end{equation*}
For the vast majority of cases, $Dist (\phi[a],\mathtt{a})$ is set to a fixed value, and within one \rt~iteration, $ Lr(g(\phi,N))$ keeps unchanged. 
Therefore, 
\begin{equation*}
    F_t = \kappa Dist_\phi(s)
\end{equation*}

So, $F_t$ can be considered as a potential value $\psi(s)$. 

Recall that, for destination reach property,
\begin{align*}
    \Tilde{F}_{t} &= \lambda F_{t+1} - F_t +\mu \Tilde{F}_{t+1} \\ 
    &= \left[\lambda\psi(s_{t+1}) - \psi(s_t)\right] + \mu \Tilde{F}_{t+1} 
\end{align*}
Generally, $\mu$ enhances property learning  based on the discount factor  $\gamma$  and typically takes on a smaller value. Therefore,
\begin{equation*}
    \Tilde{F}_{t} \approx \lambda\psi(s_{t+1}) - \psi(s_t)
\end{equation*}

This is equivalent to  the potential-based shaping reward function.

Recall that, for action avoiding property,
\begin{align*}
    \Tilde{F}_{t} &= F_{t+1} -\lambda^{-1} F_{t-1} +\mu \Tilde{F}_{t+1} 
\end{align*}

The significant difference between this and the destination reach property lies in the fact that \(F_t\) is shifted forward by one timestep. This is because, in the destination reach property, if an undesirable state region is entered at \(t+1\), it indicates that \(a_{t1}\) is a bad action, and correspondingly, \(\psi(s_{t+1})\) will be smaller. However, in the action avoiding property, we consider the current state at \(t\) itself to be undesirable, indicating that \(a_t\) is a bad action. So, we shift forward by one timestep on the basis of \(\lambda F_{t+1} - F_t\) to \(F_{t} - \lambda^{-1} F_t\), and they are equivalent.

Therefore, according to \textbf{Theorem \ref{the:pro}}, $\Tilde{F}_t$ does not affect the optimality of the policy.
\end{proof}

%% file: supsec/SupDRLP.tex
\section{Supplementary Details of DRLP }\label{Appx:DRLP}
\subsection{Motivation}
The novel DRLP appears because existing methods to encode properties, interpretability questions, and constraints  aren't user-friendly. While SMT and ACTL formulas are common, they resemble pseudocodes more than programming languages. Inspired by Python's popularity in DRL, we use its syntax to encode SMT formulas more intuitively with syntactic sugars, enhancing ease of writing. In previous works, Trainify\cite{trainify} claims ACTL use, but uses hardcoded conditional statements without parsing ACTL actually; model-checking approaches like whiRL\cite{whiRL,whiRL_2.0} use DNN verifiers' API, unsuitable for DRL verification. Verisig\cite{verily,verisig2.0} uses hybrid system models which lack readability.
\subsection{DRLP Syntax}
We design a Python-embedded Domain-Specific Language named DRLP (Deep Reinforcement Learning Property) and implement its parser. 
Its currently supported syntax is defined in Backus-Naur Form as:

\vspace{10pt}

\begin{lstlisting}[frame=single,basicstyle=\scriptsize]
<drlp> ::= 
    (<statements> NEWLINE `@Pre' NEWLINE)
    <io_size_assign> NEWLINE <statements> NEWLINE 
    `@Exp' NEWLINE <statements>

<io_size_assign> ::= `'
    | <io_size_assign> NEWLINE <io_size_id> `=' <int>
   
<io_size_id> ::= `x_size' | `y_size'

<statements> ::= `'
    | <statements> NEWLINE <statement>

<statement> ::= <compound_stmt> | <simple_stmts>

<compound_stmt> ::= <for_stmt> | <with_stmt>

<for_stmt> ::=
    `for' <id> `in' <range_type> <for_range> `:' 
    <block>

<with_stmt> ::= `with'  <range_type> `:' <block>

<block> ::= NEWLINE INDENT <statements> DEDENT
    | <simple_stmts>

<range_type> ::= `range' | `orange'

<for_range> ::= `('<int>`)'
    | `('<int> `,' <int> `)'
    | `('<int> `,' <int> `,' <int>`)'

<simple_stmts> ::= `'
    | <simple_stmts> NEWLINE <simple_stmt>

<simple_stmt> ::= <call_stmt> | <expr>

<call_stmt> ::= `Impiles' `(' <expr> `,' <expr> `)'
    | `And' `(' <exprs> `)'
    | `Or' `(' <exprs> `)'

<exprs> ::= <expr> 
    | <exprs> `,' <expr>

<expr> ::= <obj> <comparation>

<comparation> ::= `' 
    | <comparator> <obj> <comparation>

<obj> ::= <constraint> | <io_obj>

<io_obj> ::= <io_id> 
    | <io_id> <subscript>
    | <io_id> <subscript> <subscript>
    
<io_id> ::= `x' | `y'

<subscript> ::= `[' <int> `]'
     | `[' <int>`:'<int> `]'
     | `[' <int>`:'<int> `:'<int>`]'

<int> ::= <int_number> 
    | <id> 
    | <int> <operator> <int>

<constraint> ::= <int> 
    | <list>
    | <constraint> <operator> <constraint>
\end{lstlisting}

\subsection{DRLP Design}

\paragraph{DRLP Structure.}
A DRLP script is structured into three segments: \textsl{drlp\_v} encompassing pre-assigned variables and Python codes to execute, \textsl{drlp\_p} comprising the pre-condition $P$ initiated by the delimiter \texttt{@Pre}, and \textsl{drlp\_q} comprising the post-condition $Q$ initiated by the delimiter \texttt{@Exp}.

In \textsl{drlp\_v}, the statements are not directly associated with a specific property, but they facilitate to \textsl{batch verification}. Their execution outcomes, encompassing all variable values within the local scope, are transferred to the subsequent two segments for the assignment of variables without concrete values. Variables prefixed with ``\_'' are \textsl{iterable variables}; Cartesian products of all textsl{iterable variables} are generated by removing ``\_'' from their names. To illustrate, if \texttt{a=1,\_b=[2,3],\_c=[4,5]} is in \textsl{drlp\_v}, four variable values are passed to \textsl{drlp\_p} and \textsl{drlp\_q}: 
$a=1,b=2,c=4$; $a=1,b=2,c=5$; $a=1,b=3,c=4$; $a=1,b=3,c=5$.
Consequently, a single property in DRLP can be translated into four distinct properties, each with specific concrete values.

In \textsl{drlp\_p}, it is imperative to state all pre-condition constraints. In situations where the size of the DNN input or output cannot be inferred from other statements, an explicit declaration is mandated. For instance, \texttt{x\_size=3} or \texttt{y\_size=1} must be clearly stated.

In \textsl{drlp\_q}, all post-condition constraints are stated. 

\paragraph{DRLP Variables Access.}
In DRLP, the input $x$ is treated as a matrix of dimensions $k \times n$, where each row represents a set of inputs for the original DNN with a size of $1 \times n$. Similarly, the output $y$ follows a similar structure. Slicing can be applied to access specific portions of $x$ and $y$. For instance, \texttt{x[0][0:2]} refers to the 0-th and 1-st input variables of the 0-th original DNN. Also, \texttt{y[1]} denotes the output of the 1-st original DNN.

\paragraph{DRLP Statement.}
The statements in \textsl{drlp\_p} and \textsl{drlp\_q} primarily consist of constraints formulated as comparison expressions. For example, given an original DNN with an input size $1\times 2$, where the initial state condition $I$ dictates that both input variables are set to $0$, this can be stated as \texttt{[0]*2$\le$ x[0]$\le$ [0]*2} or \texttt{[0]*2 == x[0]}. Here, \texttt{[0]*2} is a simplified form of \texttt{[0,0]}.

\paragraph{DRLP Relation Domain.}
The relation \textsl{And} is the predominant Boolean operator employed in DRL properties. As a result, the constraint statements in each line of DRLP scripts are inherently connected using the \textsl{And} relation by default. This default connection is referred to as the \textsl{And} relation domain ($range$). Consequently, the constraint statements present within this range must be simultaneously satisfied.

When certain constraints are connected by the relation \textsl{Or}, it is necessary to explicitly define the \textsl{Or} relation domain ($orange$) using the statement with \texttt{with orange:}. This signifies that within the $orange$ segment, it is sufficient for just one constraint to be satisfied.

\paragraph{DRLP Loop.}
In DRL, it is common for certain constraints to take on similar forms due to multiple interactions. To accommodate this, DRLP introduces the \texttt{for} keyword, allowing the declaration of loop code segments in a Python-style manner. A looping range can be stated as \texttt{for i in range(0,3):}, similar to Python loops. Within each unrolled loop, statement segments are connected using the \textsl{And} relations.

Looping $orange$ can be stated like \texttt{for i in orange(0,3):}, and statement segments of each unrolled loop are connected with the \textsl{Or} relation. Note that only these segments are in $orange$ between them, but the statements inside a segment are still in $range$ if without explicit declaration.

%% file: supsec/SupReinfier.tex
\section{Supplementary Design of Reinfier}\label{Appx:Reinfier}
\subsection{DNN Adapter}
Due to the multiple interactions between the agent and the environment, all constraints from the property, the DRL system, and the DNN need encoding to be consistent with verification depth $k$ for model checking algorithms. Therefore, the agent DNN should be unrolled and encoded $k$ times larger than the original when the verification depth is $k$. This unrolled DNN is used as the input DNN for DNN verification. 
$x_i^j$ denotes the $j$-th value of the input of the $i$-th original DNN; the input of the $i$-th original DNN $x_i$ is $\langle x_i^0,...,x_i^{n-1} \rangle$; the input $x$ $\langle x_0,...,x_{k-1} \rangle$.
$y_i^j$ denotes the $j$-th value of the output of the $i$-th original DNN; the output of the $i$-th original DNN $y_i$ is $\langle y_i^0,...,y_i^{m-1} \rangle$; the output $y$ is $\langle y_0,...,y_{k-1} \rangle$.


\subsection{Type of Properties}
All constraints from the property and the DRL system need encoding. Thus, pre-condition $P$ can be divided into the following four parts:

\textbf{(i) State boundary condition $\bm{S}$:}
describes the boundary of the environment in DRL. 
and all input features should be bounded.
\textbf{(ii) Initial state condition $\bm{I}$:}
describes the initial state of the system, because a DRL system could start from a defined state or within a particular state boundary.
\textbf{(iii) State transition condition $\bm{T}$:}
describes how the environment transits from the current state $x_i$ to the new state $x_{i+1}$. Generally, the action $N(x_i)$ taken by the agent plays a pivotal role in the state transition, but $x_{i+1}$ is not only determined by $N(x_i)$ but also affected by random factors.
\textbf{(iv) Other condition $\bm{C}$:}
describes extra constraints related to specified properties. It stands apart from the $S$, $I$, and $T$ parts, which usually describe the unchanging characteristics of a particular DRL system. For example, liveness properties~\cite{mc} require no state cycle existing, such as $x_0 != x_2$; and this constraint should be considered as an extra constraint due to its exclusive association with this particular property, rather than being a characteristic of the entire system. We believe that distinct division in this way is beneficial for users to formulate and articulate properties.

\textbf{Post-condition $\bm{S}$:} describes the expected result of a given DRL system under previous constraints.

Safety properties~\cite{mc} represent one of the paramount property types, where the agent is expected to persistently avoid undesirable actions under given pre-conditions. Safety properties guarantee that the DRL system will never transition into \textit{bad} states. The predicate $Good(N(x))$ evaluates as \texttt{True} when in \textit{good} states and vice versa. Such a property is formulated as:
\begin{equation}
\begin{split}
        \forall x_0,x_1..., &\bigwedge_{i=0}(S(x_i)\wedge I(x_i)\wedge T(x_i,x_{i+1})\wedge C(x_i)) \\
        &\Rightarrow \bigwedge_{i=0}(Good(N(x_i)))
\end{split}
\end{equation}

Another type of property is liveness property~\cite{mc}, where the agent might take undesirable actions but is expected to eventually take desirable ones. Liveness properties guarantee that the DRL system will not remain persistently trapped in \textit{bad} states. Thus, verifying liveness properties involves the optional detection of cycles of states consistently in \textit{bad} states. $C'(x)$ denotes $C(x)$ without cycle detection, i.e., $x_i != x_\gamma$. Such a property is formulated as:

\begin{equation}
\begin{split}
        \forall x_0,x_1..., &\bigwedge_{i=0}(S(x_i)\wedge I(x_i)\wedge T(x_i)\wedge C'(x_i) \\ & \wedge(\bigwedge_{\gamma=0}^{i-1} (x_i != x_\gamma))) 
        \Rightarrow \bigvee_{i=0}(Good(N(x_i)))
\end{split}
\end{equation}

Robustness properties assess the stability of the agent's performance within adversarial environments, which constrains the agent to take consistent or similar actions in the presence of perturbed input values (local robustness~\cite{robness}) or input outliers (extreme value robustness~\cite{Reluplex}. $\hat{x_{0}}$ denotes the original input, and $\varepsilon$ denotes perturbation. Such a property is formulated as: 
\begin{equation}
        \forall x_0, I(x_0)\wedge C(x_0)\Rightarrow N(x_0) \approx N(\hat{x_{0}})
\end{equation}
For local robustness, $I(x_0)\equiv \bigwedge_{i=0} (\hat{x_{0}^i} -\varepsilon^i \le x_0^i \le \hat{x_{0}^i} +\varepsilon^i)$; for extreme value robustness, $I(x_0)\equiv \bigwedge_{i \in D} (\hat{x_{0}^{i}} -\epsilon^i \le x_0^i \le \hat{x_{0}^{i}} +\epsilon^i) \wedge \bigwedge_{i \in R} (\hat{x_{0}^{i}} -\varepsilon^i \le x_0^i \le \hat{x_{0}^{i}} +\varepsilon^i)$, in which $D$ denotes the set of indices of the input features whose values are outliers, while $R$ denotes the set of indices of other input features.

\subsection{Single Verifier}
\subsubsection{Model Checking Algorithm.}
Inspired by related works~\cite{whiRL,whiRL_2.0}, we implement the following algorithms in Single Verifier.
\paragraph{Bounded model checking (BMC).}
As the verification depth $k$ increases from $1$, both the original DNN and the DRLP script are unrolled or translated with depth $k$. Then, DNN verifier is queried. The query is formulated as:


\begin{equation}
\begin{split}
    \forall x_0,...,x_{k-1}, &(\bigwedge_{i=0}^{k-1}{(S(x_i)\wedge C(x_i))}) \wedge I(x_0) \\
    & \wedge(\bigwedge_{i=0}^{k-2}{T(x_i, x_{i+1})}) \rightarrow  \bigwedge_{i=0}^{k-1}{Q((N(x_i))}
\end{split}
\end{equation}

\paragraph{$\bm{k}$-induction.}
The first phase of $k$-induction, the basic situation verification, is identical to the BMC process; while the second phase diverges. In the inductive hypothesis verification, the original DNN is unrolled with depth $k+1$, and the DRLP script is translated for indication verification with depth $k+1$ and an assumption that the property holds with depth $k$. Then, the DNN verifier is queried. The query is formulated as:
\begin{equation}
\begin{split}
    \forall x_\alpha,...,x_{\alpha+k},  &(\bigwedge_{i=\alpha}^{\alpha+k}{(S(x_i)\wedge C(x_i))}) \wedge(\bigwedge_{i=\alpha}^{\alpha+k-1} (T(x_i \\
    & ,x_{i+1}) \wedge Q(N(x_i)))\rightarrow Q(N(x_{\alpha+k}))
\end{split}
\end{equation}

\subsubsection{Reachability Analysis}
Through our DRLP parser and DNN adapter, we use Verisig~\cite{verisig2.0} for reachability analysis.

\subsection{Batch Verifier} \label{appx:batch verfier}
\begin{algorithm}[htb]
    \small
    \caption{FindBreakpoints (General Form)}
    \label{alg:find_bpsapp}
    
    \textbf{Input}: DRLP template $p$, DNN $N$, variable list $\{vars\}$\\
    \textbf{Output}: Breakpoint list $\{bp\}$ 
    \begin{algorithmic}[1]
        \STATE $var \gets $ pop the top variable from $\{vars\}$
        \STATE $lb \gets $ the lower bound of $var$, $ub \gets $ the upper bound of $var$
        \STATE $prec \gets $ the search precision of $var$
        \STATE 
        \WHILE{\texttt{True}}
            \STATE $lb,ub,curr \gets$ step$(lb,ub,curr,prec,p,N,var,\{vars\})$
                        \IF{$curr = \texttt{None}$  }
            \STATE \textbf{break}
            \ENDIF

            \STATE $script \gets $concretize$(p,var,curr)$
            \IF{$\{vars\}$ is empty }
                \STATE $curr\_r, violation \gets $verify$(script,N)$
                \IF{curr\_r $\neq prev\_r$}
                    \STATE Add $script$'s corresponding property to $\{bp\}$
                \ENDIF
                \STATE $prev\_r \gets curr\_r$
            \ELSE
                \STATE FindBreakpoints$(script,N,\{vars\})$
            \ENDIF
            
        \ENDWHILE
       
    \end{algorithmic}
\end{algorithm}

\begin{algorithm}[htb]
    \small
    \caption*{\textbf{Function} step}
    \label{alg:step}
    \textbf{Input}: lower bound $lb$, upper bound $ub$, current value $curr$, search precise $prec$, DRLP template $p$, DNN $N$, variable $var$, variable list $\{vars\}$\\
    \textbf{Output}: lower bound $lb$, upper bound $ub$, current value $curr$ 
    \begin{algorithmic}[1]

        \STATE \{Linear Search\}
        \IF{$curr \leq ub$}  
            \STATE $curr \gets curr + prec$
        \ELSE
            \STATE $curr \gets \texttt{None}$
        \ENDIF
        \STATE
        \STATE \{Binary or Iterative Search\}
        
        \STATE $lb\_r, violation \gets$ verify(concretize$(p,var,lb,\{vars\}) , N)$
        \STATE $ub\_r, violation \gets$ verify(concretize$(p,var,ub,\{vars\}) , N)$
        \STATE $curr\_r, violation \gets$ verify(concretize$(p,var,curr,\{vars\}) , N)$
                    \IF{$curr\_r = \texttt{Falsified}$}
                    \STATE $curr \gets $ the value of $var$ in $violation$
            \ENDIF
        \STATE
        \STATE \{Binary Search\}
        \IF{$ub-lb > prec$}

            \IF{$lb\_r = curr\_r$}
                \STATE $lb\gets curr$
            \ELSIF{$ub\_r \neq curr\_r$}
                \STATE $ub\gets curr$
            \ENDIF
            \STATE $curr \gets \frac{lb+ub}{2}$
        \ELSE
            \STATE $curr \gets \texttt{None}$
        \ENDIF
        
        \STATE
        \STATE \{Iterative Search\}
        
        \IF{$lb\_r = curr\_r$ }
            \STATE $curr \gets curr \times iterative\_step$
        \ELSE
            \STATE $ub \gets curr$
            \STATE Switch to Linear or Binary Search
        \ENDIF
    \end{algorithmic}
\end{algorithm}

We design Algorithm~\ref{alg:find_bpsapp} to identify \textsl{breakpoints}, using three representative methods in the function $step$: linear, binary, and iterative search. The Algorithm~\ref{alg:find_bps} in the main text is a special case of Algorithm~\ref{alg:find_bpsapp}, and also the most common case, where all variables in the variable list $\{vars\}$ are searched using linear search except for the last variable, which uses binary search.

\subsection{Interpreter}\label{appx:interpreter}
In addition to the problems in the main text, we also provide solutions to the following problems, which are similarly based on breakpoints.
\subsubsection{Decision Boundary}
For a robustness property ~\cite{iw42,iw43}, the data points, which are formed by the parameters extracted from the \textsl{breakpoints}, are part of the decision boundary. This perspective inspires us to discern the specific conditions under which the DRL system satisfies given safety, liveness, and robustness properties. Additionally, in the context of safety and liveness properties, if the verification depth is confined to $k=1$, i.e., degenerating into one-shot DNN verification, the parameters extracted from the identified \textsl{breakpoints} are also recognized as invariants~\cite{whiRL_2.0}, which indicates the boundary between conditions wherein the property always holds or never holds. 

\textbf{Our solution:} Find \textsl{breakpoints} by stepping through interested parameters; and fit the corresponding decision boundary with the data points 
consisting of parameters extracted from these \textsl{breakpoints}. 

\subsubsection{Intuitiveness Examination}
An intuitive insight of the example of sending data using the computer: the worse the network condition is, the more reasonable for the computer to decrease its sending rate. When focusing on safety and liveness properties, as the values of certain parameters in the property increase or decrease, the associated constraints will consequently become more tightened or relaxed, rendering the property either easier or harder to hold; and this idea leads us to propose the intuitiveness examination. 

\textbf{Our solution: }
Find \textsl{breakpoints} by linear stepping through interested parameters; and if the count of the identified \textsl{breakpoints} is either $0$ or $1$, the intuitiveness examination passes. $\{bp\}$ denotes the list of corresponding properties associated with identified \textsl{breakpoints}, and the answer to this question is formulated as: $|\{bp\}|\in{0,1}$.

\subsubsection{Counterfactual Explanation} 
Given the original input $\hat{x_0}$ and the target counterfactual output $\hat{y_0}$ such that $\hat{y_0}\not\approx N(\hat{x_0})$, the counterfactual explanation question~\cite{iw28,iw29} aims to identify the counterfactual inputs $\{x_0\}$ such that $N(x_0) \approx \hat{y_0}$. The most valuable counterfactual input $x_0$ is the one that is closest (according to vector distance) to the original input $\hat{x_0}$. Because the minimum vector distance often implies the minimum attack cost~\cite{Attack}, this counterexample can effectively highlight the system's weakness.

\textbf{Our solution: } 
Find \textsl{breakpoints} by stepping through possible perturbation value $\varepsilon_0^j$ of each input $x_0^j$ on the property formulated as:
\begin{equation}
         \forall x_0, \bigwedge_{j=0}^{n-1}(\hat{x_{0}^{j}} -\varepsilon_0^j \le x_0^j \le \hat{x_{0}^{j}} +\varepsilon_0^j) \Rightarrow N(x_0) \approx \hat{y_0}
\end{equation}
Here, we expand the question to encompass safety and liveness properties as well: safety, $    P
\equiv \bigwedge_{i=0}((\bigwedge_{j=0}^{n-1}(\hat{x_{i}^{j}} -\varepsilon_i^j \le x_i^j \le \hat{x_{i}^{j}} +\varepsilon_i^j)) \wedge  T(x_i,x_{i+1})) \wedge I(x_0) $; liveness, $P(x) \equiv \bigwedge_{i=0}((\bigwedge_{j=0}^{n-1}(\hat{x_{i}^{j}} -\varepsilon_i^j \le x_i^j \le \hat{x_{i}^{j}} +\varepsilon_i^j) \wedge T(x_i,x_{i+1})
\wedge(\bigwedge_{\gamma=0}^{i-1} (x_i != x_\gamma)))\wedge I(x_0)$. Finally, the answer is the 
input closest to the original input, formulated as: $ \argmin_{bp\in \{bp\}} d(\hat{x}+\varepsilon,\hat{x})$, where $d$ denoted a specific distance function.

\subsubsection{Importance Analysis}
Given the original input $\hat{x_0}$, the importance analysis question~\cite{iw8} aims to evaluate the distances between $x_0$ and $\hat{x_0}$,
where $x_0$ denotes the input whose features under discussion can be perturbed such that $N(x_0) \not\approx N(\hat{x_0})$.
The lower the distance, the more important the feature, as even a slight perturbation in the input of such a feature can induce greater fluctuations in the output. Importance analysis on all input features con identifies which holds a central role in the decision-making process.

\textbf{Our solution: } 
Find \textsl{breakpoints} by stepping through possible perturbation values $\{\varepsilon^j_0\}$ applied to the input features $\{x_0^j\}$ on the property formulated as:
\begin{equation}
\small
\begin{split}
     \forall x_0, &\bigwedge_{j \in D} (\hat{x_{0}^{j}} -\varepsilon^j_0 \le x_0^j \le \hat{x_{0}^{j}} +\varepsilon^j_0) \wedge \bigwedge_{j \in R} (\hat{x_{0}^{j}}=x_0^j)  \\
     & \Rightarrow N(x_0) \not \approx N(\hat{x_0})
\end{split}
\end{equation}
where $D$ denotes the set of indices of the discussed features within the input, while $R$ denotes the set of indices of other features.
Finally, the answer is the minimum distance between $\hat{x_0}+\varepsilon_0$ and $\hat{x_0}$, i.e., importance, formulated as:$\min_{bp\in \{bp\}} d(\hat{x_0}+\varepsilon_0,\hat{x_0})$.

%% file: supsec/SupAlg.tex
\section{Supplementary Details of Reintrainer}\label{Appx:Reintrainer}


\subsection{Distance of States in $1$-Dimension.}
To elucidate our algorithm in~\rt and our motivation, we employ the Mountain Car~\cite{mountaincar} task as a case study. The task consists of a car placed stochastically at a valley's bottom, and its goal is to accelerate the car to reach the top of the right hill. The agent can take one of three discrete actions: accelerate to the left, to the right, or not accelerate. The environment states consist of two features: the car's position along the x-axis $p$, and the car's velocity $v$. 

Considering the safety property $\phi_1$ of Mountain Car (then denoted by $\phi$, $\phi \equiv \langle p,v \rangle \in [\langle -0.60,0.03\rangle,\langle -0.40, 0.07\rangle]
\Rightarrow Action \neq 0$) defined in Evaluation,
if this property is violated under the conditions such as when $\langle p,v \rangle$ is
\textbf{(i)} $\langle -0.50,0.05 \rangle$, \textbf{(ii)} $\langle -0.41,0.05 \rangle$, \textbf{(iii)} $\langle -0.59,0.05 \rangle$, \textbf{(iv)} $\langle -0.50,0.07 \rangle$, the agent's action is 0.

\paragraph{Raw distance.}
The agent in \textbf{(i)} should be penalized more severely than in \textbf{(ii)}, because the latter only needs to make a minor adjustment, i.e., reduce $p$ to $0.4$ to satisfy the property. Therefore, the distance from \textbf{(i)} to the boundary should be greater than the distance from \textbf{(ii)}.
Therefore, we define a raw distance function: given the state feature $s_i^j$ whose current observation is $\mathtt{v}$, its raw distance $dist'$ to its property-constrained space $\phi[s_i^j]$ in $1$-dimension is formulated as:

\begin{equation}\label{eq:raw}
    dist'(\phi[s_i^j],\mathtt{v}) = min(|\underline{\phi}[x_i^j]-\mathtt{v}|,|\overline{\phi}[x_i^j]-\mathtt{v}|)
\end{equation}

\paragraph{Distance density.}

However, the raw distance measure $dist'$ is not sufficient for certain scenarios, such as when comparing \textbf{(ii)} with \textbf{(iii)}. In both cases, the value of 
$p$ is equally close to the boundary, yet the fluctuations in the output could be different for the same input feature at different values.
This observation leads us to the concept of sensitivity analysis mentioned in interpretability questions, which is used to measure the potential of a feature to induce fluctuations in the output.

Given the DNN $N$ and the original input $\hat{x}_0$, when the input feature $x_0^j$ is subject to perturbation $\varepsilon$, the sensitivity of $x_0^j$ can be formulated as:

\begin{equation}
            Sensitivity(N,\hat{x}_0, x_0^j) = 
            \max_{x_0} d(N(\hat{x}_0),N(x_0))
\end{equation}
where $x_0\in[\langle\hat{x}_0^0,\hat{x}_0^1,...,\hat{x}_0^j-\varepsilon,...,\hat{x}_0^{n-1}\rangle,\langle\hat{x}_0^0,\hat{x}_0^1,...,\hat{x}_0^j+\varepsilon,...,\hat{x}_0^{n-1}\rangle]$.



Considering that perturbations need only be evaluated in one direction at the boundary, we define two types of \textsl{density} measures:

The first is the \textsl{density} of the lower bound $\underline{\rho}(N,\hat{x}_0, x_0^j)$, where $x_0\in[\langle\hat{x}_0^0\hat{x}_0^1,...,\hat{x}_0^j$ $,...,\hat{x}_0^{n-1}\rangle,\langle\hat{x}_0^0,\hat{x}_0^1,$ $...,\hat{x}_0^j+ \varepsilon , ... , \hat{x}_0^{n-1}\rangle]$. The second is the \textsl{density} of the upper bound $\overline{\rho}\allowbreak(N,\hat{x}_0, x_0^j)$, where $x_0\in[\langle\hat{x}_0^0,\hat{x}_0^1,...,\hat{x}_0^j-\varepsilon,...,\hat{x}_0^{n-1}\rangle,\allowbreak\langle\hat{x}_0^0,\hat{x}_0^1,...,\hat{x}_0^j,...,\hat{x}_0^{n-1}\rangle]$.
The function is also denoted as $\underline{\rho}(\phi,s_i^j)$ for simplicity, as Eq.~\ref{eq:rho} in tha main text.

\paragraph{Exact middle point.}
The exact middle point, denoted as $\dot{\phi}[s_i^j]$, represents the position with the maximum distance to the boundary, and also the point where the \textsl{density} of the lower and upper boundaries switch. Accordingly, this point should be closer to the boundary that exhibits greater \textsl{density}. Consequently, we determine the exact middle point by calculating the midpoint of the \textsl{density}-weighted upper and lower bounds. 
Besides, in the comparison between scenarios \textbf{(i)} with \textbf{(iv)}, although $v$ in \textbf{(iv)} is closer to $\phi[v]$, the value of $v$ equaling $0.07$ is actually the upper bound of velocity in the Mountain Car environment (denoted by $\overline{S}[v]$). Therefore, a significant adjustment is needed in \textbf{(iv)}, i.e., reduce $v$ to $0.3$ to satisfy the property. Consequently, the distance from \textbf{(iv)} to the boundary should be greater than the distance from \textbf{(i)}. Therefore, the exact middle point considering the environmental boundaries is formulated as: 



\begin{equation}
    \dot{\phi}[s_i^j] = \left\{
    \begin{array}{ll}
        \underline{S}[s_i^j] & \text{if }
            \begin{array}{l}
                  \underline{\phi}[s_i^j] = \underline{S}[s_i^j] \\
                \wedge \overline{\phi}[s_i^j] \neq \overline{S}[s_i^j]
            \end{array}
        \\
        \frac{
            \underline{\rho}(\phi,s_i^j) \cdot \underline{\phi}[s_i^j] +
            \overline{\rho} (\phi,s_i^j) \cdot \overline{\phi}[s_i^j] }{
            \underline{\rho}(\phi,s_i^j) + \overline{\rho} (\phi,s_i^j)  } & 
        \text{else}
        \\
        \overline{S}[s_i^j] & \text{if }
            \begin{array}{l}
                  \underline{\phi}[s_i^j] \neq \underline{S}[s_i^j] \\
                \wedge \overline{\phi}[s_i^j] = \overline{S}[s_i^j]
            \end{array}
    \end{array}
    \right.
\end{equation}

\subsection{Practical Method to Metric of Gap}
 One parameter $z$ from the environment state features or the action should be chosen for the gap measure. In our implementation, \rt could automatically choose the lower bound of action for continuous action space or the lower bound of the first environment state for discrete action space by default as the metric parameter; and it can also be chosen manually.


For example, the property $\phi \equiv x > 0 \Rightarrow N(x) > 0$ is predefined. We choose the lower bound of action as the metric parameter. If the verifier falsifies $\phi$, a relaxation domain of $\phi$ can be obtained by relaxing the action's constraint interval according to Theorem \ref{the:re}. For instance, changing $N(x) > 0$ to $N(x) > -10$ forms $\Grave{\Phi}_{x > 0 \Rightarrow N(x) > 0}=\{ x > 0 \Rightarrow N(x) > a: \forall a\in[-10,0]\}$. Using the \textsl{breakpoint} search algorithm, setting the search bounds of variable $a$ to $[-10, 0]$ allows us to identify the set of \textsl{breakpoint} $\{bp\}$, such as $\{\equiv x > 0 \Rightarrow N(x) > -5\}$. Then, we calculate $g(\phi, N) = \min_{\Grave{\phi} \in bps} |\phi[z] - \Grave{\phi}[z]| = |0 - (-5)| = 5$. If no \textsl{breakpoint} are found, the search bounds of variable $a$ are expanded to $[-40, 0]$, and the process is repeated until any \textsl{breakpoint} is found.

The method for expanding the search bounds is quite subjective. Typically, the \textsl{breakpoint} search adopts a binary stepping approach, so one possible way is to quadruple the size of the search interval each time.

\input{tab/Aurora/Safety_and_Liveness.tex}
\input{tab/Aurora/Robustness.tex}

\subsection{Weights for Multiple Properties}
If multiple properties are simultaneously required to be satisfied, a weight $w$ can be assigned to each property. The weights can be assigned manually based on the importance or based on the preference of each property, with higher weights given to properties more likely to be satisfied. 
Another approach is to build a decision tree based on predefined properties, with constraints of a property starting from the root node and ending at a leaf node. We believe that the deeper the leaf node, the more relaxed the property. By arranging all properties according to their depth, a sequence of properties ranging from relaxed to stringent can be established, and their weights can be chosen in ascending or descending order.
We use factor $\beta$ to scale the effect of $\Tilde{F}^i_t$ as a whole, and we find during the evaluation that when intermediate reward sum $\hat{F}_t$ and $r_t$ are at the same order of magnitude, our approach works well. 
The intermediate reward sum function is formulated as:
\begin{equation}
    \hat{F}_t = \beta\sum{w^i \Tilde{F}^i_t}
\end{equation}
In addition to reward shaping, each $\Tilde{F}^i_t$ can also be the cost respectively in Constrained MDP. 

\subsection{Counterexamples Generation}
\rt~can actively generate a small number of counterexamples in the buffer for the agent's property learning. Specifically, \rt~initially relies on reward shaping for property learning during training, and if few or no violations occur but the predefined properties are falsified, it later combines reward shaping with data augmentation using generated counterexamples.

%% file: tab/Aurora/Safety_and_Liveness.tex
\begin{table*}[htb]

    \centering
    \begin{tabularx}{\textwidth}{p{2.8cm}XXXX}
    \hline
     & \textbf{$\phi_{7}$} & \textbf{$\phi_{8}$} & \textbf{$\phi_{9}$} & \textbf{$\phi_{10}$} 
     \\ \hline
     Type & Liveness & Liveness & Safety & Liveness
     \\ \hline
    \normalsize State Boundary $S$ 
      & 
          {
          \setlength{\abovedisplayskip}{-8pt}
          \setlength{\belowdisplayskip}{-8pt}
          \begin{alignat*}{8}
          &i &&\in [&-0&.01&&,&+0&.01&&] \\
          &l &&\in [& 1&.00&&,& 1&.01&&] \\
          &r &&\in [& 1&.00&&,& 1&.00&&] 
          \end{alignat*}}
      &   
          {
          \setlength{\abovedisplayskip}{-8pt}
          \setlength{\belowdisplayskip}{-8pt}
          \begin{alignat*}{8}
          &i &&\in [&-0&.01&&,&+0&.01&&] \\
          &l &&\in [& 1&.00&&,& 1&.01&&] \\
          &r &&\in [& 1&.00&&,& 1&.00&&] 
          \end{alignat*}}
      
      &            
          {
          \setlength{\abovedisplayskip}{-8pt}
          \setlength{\belowdisplayskip}{-8pt}
          \begin{alignat*}{8}
          &i &&\in [&-0&.01&&,&+0&.01&&] \\
          &l &&\in [& 1&.00&&,& 1&.01&&] \\
          &r &&\in [& 2&.00&&,& 100&.00&&] 
          \end{alignat*}}
      &  
          {
          \setlength{\abovedisplayskip}{-8pt}
          \setlength{\belowdisplayskip}{-8pt}
          \begin{alignat*}{8}
          &i &&\in [&-0&.01&&,&+0&.01&&] \\
          &l &&\in [& 1&.00&&,& 1&.01&&] \\
          &r &&\in [& 2&.00&&,& 100&.00&&] 
          \end{alignat*}}
    \\ \hline
    Other Condition $C$ & {No state cycle} &{No state cycle} & None & {No state cycle} \\
    \hline
    Initial State $I$ & \multicolumn{4}{c}{None} \\
    \hline
    State Transition $T$ & \multicolumn{4}{c}{History features shift by one value}  
    \\ \hline
    
    Post-condition $Q$ &  Exist  $N(x_i) \not \approx 0$ 
    & Exist  $N(x_i) > 0$ 
    & Forall  $N(x_i) < 0$ &
    Exist  $N(x_i) < 0$
    \\    \hline
    Reinfier Result   & \texttt{Proven} & \texttt{Falsified}   & \texttt{Falsified}      & \texttt{Falsified}  \\
    whiRL Result      & \texttt{Proven} & \texttt{Falsified}   & \texttt{Falsified}      & \texttt{Proven}    
    \\ \hline
    
    \end{tabularx}
        \caption{Aurora safety and liveness properties and comparison of their results by Reinfier and whiRL.}
    \label{tab:ausl}
\end{table*}

%% file: tab/Aurora/Robustness.tex
\begin{table*}[htb]
    \centering
    
    \begin{tabularx}{\textwidth}{p{2.8cm}XXXX}
    \hline
     & \textbf{$\phi_{11}$} & \textbf{$\phi_{12}$} & \textbf{$\phi_{13}$} & \textbf{$\phi_{14}$} 
     \\ \hline
     Type & {Local Robustness} & {Local Robustness} & {Extreme Value} & {Extreme Value}
     \\ \hline
    $\langle i,l,r\rangle$\newline
    \normalsize{Original Input $\hat{x}_0$} 
      &   {$\langle 2.00,2.00,10.00\rangle$}  
      &   {$\langle -0.70,0.50,1.00\rangle$}  
      &   {$\langle 0.00,1.05,1.00\rangle$}  
      &   {$\langle 0.00,1.05,1.00\rangle$} 
    \\ \hline
    Perturbation $\varepsilon$
    & {$\langle0.00,0.01,0.01\rangle$}
    & {$\langle0.01,0.01,0.00\rangle$}
    & {$\langle0.005,0.005,20\rangle$}
    & {$\langle0.005,0.005,-0.05\rangle$}
    
    \\ \hline
    Post-condition $Q$ 
    & {$N(x_0)\approx N(\hat{x}_0)$}  \newline{$\varepsilon=0.05$}  
    & {$N(x_0)\approx N(\hat{x}_0)$}  \newline{$\varepsilon=0.05$}  
    & {$N(x_0)\approx N(\hat{x}_0)$}  \newline{$\varepsilon=1.0$}  
    & {$N(x_0)\approx N(\hat{x}_0)$}  \newline{$\varepsilon=0.3$}  
    
    \\ \hline
    Reinfier Result & \texttt{Proven}     & \texttt{Falsified}    & \texttt{Falsified}    & \texttt{Falsified}  \\
    whiRL Result   & {Not applicable} & {Not applicable} & {Not applicable} & {Not applicable}
    \\ \hline
    
    \end{tabularx}
    \caption{Aurora robustness properties and comparison of their results by Reinfier and whiRL.}
    
    \label{tab:aurub}
\end{table*}

%% file: supsec/SupEvaluation.tex
\section{Supplementary Evaluation}\label{Appx:Eva}

\subsection{Reinfier}\label{Appx:Eva_rf}
\subsubsection{Study Case}
To demonstrate the usability and efficacy of \rf as well as provide some insights into the model, we opt for the study case Aurora~\cite{Aurora} utilized in related works whiRL~\cite{whiRL,whiRL_2.0} and UINT~\cite{uint} as the experimental environment. 

Aurora is a DRL system for congestion control of Transmission Control Protocol (TCP). It takes the latest 10-step history of these features as input: latency inflation $i$, latency ratio $l$, send ratio $r$; and it outputs the changed ratio of sending rate over the sending rate of the last action. 



\subsubsection{Verification}
\input{tab/Advantage_whiRL.tex}
\input{tab/Aurora/Decision_Boundary.tex}

\begin{figure*}[htb]
\centering
\begin{subfigure}[t]{0.24\textwidth}
    \includegraphics[width=\textwidth]{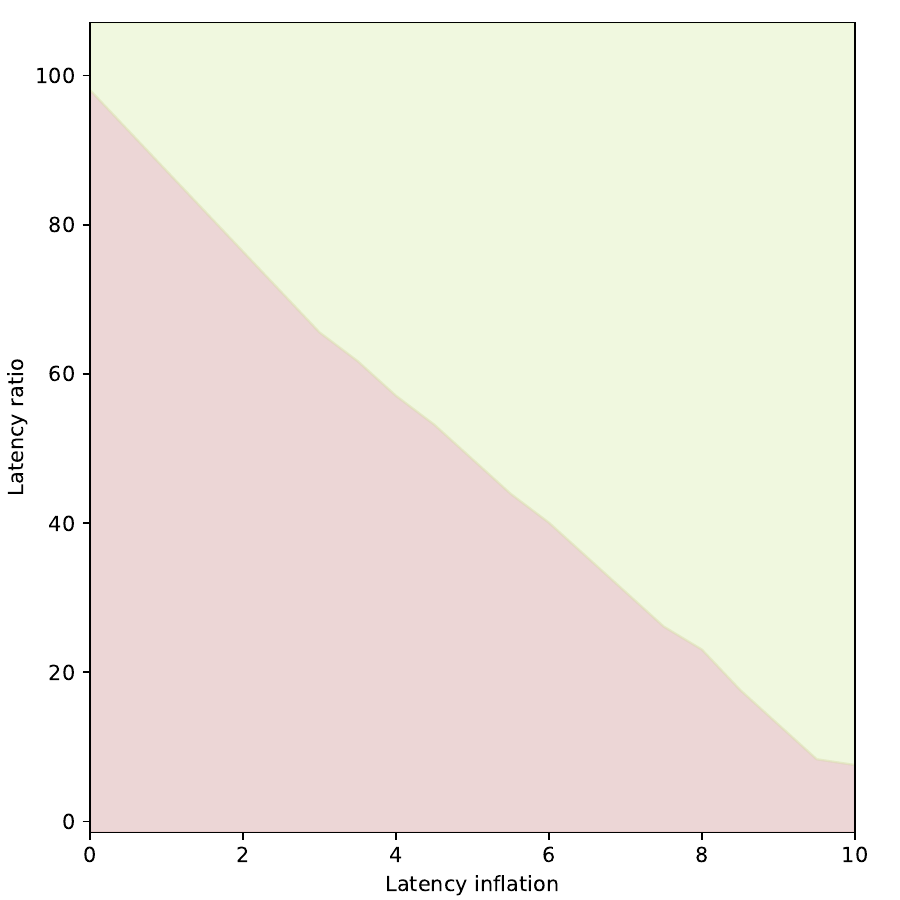}
    \caption{$\phi_{9.d}$}
\end{subfigure}
\begin{subfigure}[t]{0.24\textwidth}
    \includegraphics[width=\textwidth]
    {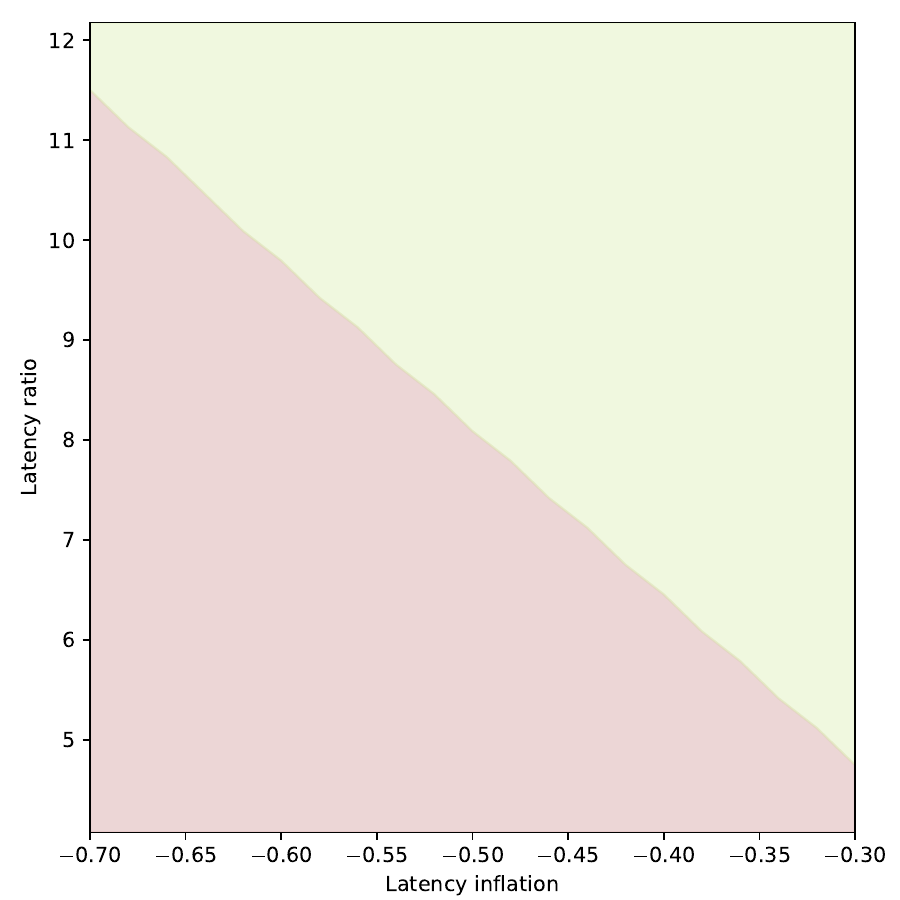}
    \caption{$\phi_{7.a}$}
\end{subfigure}
\begin{subfigure}[t]{0.24\textwidth}
    \includegraphics[width=\textwidth]
    {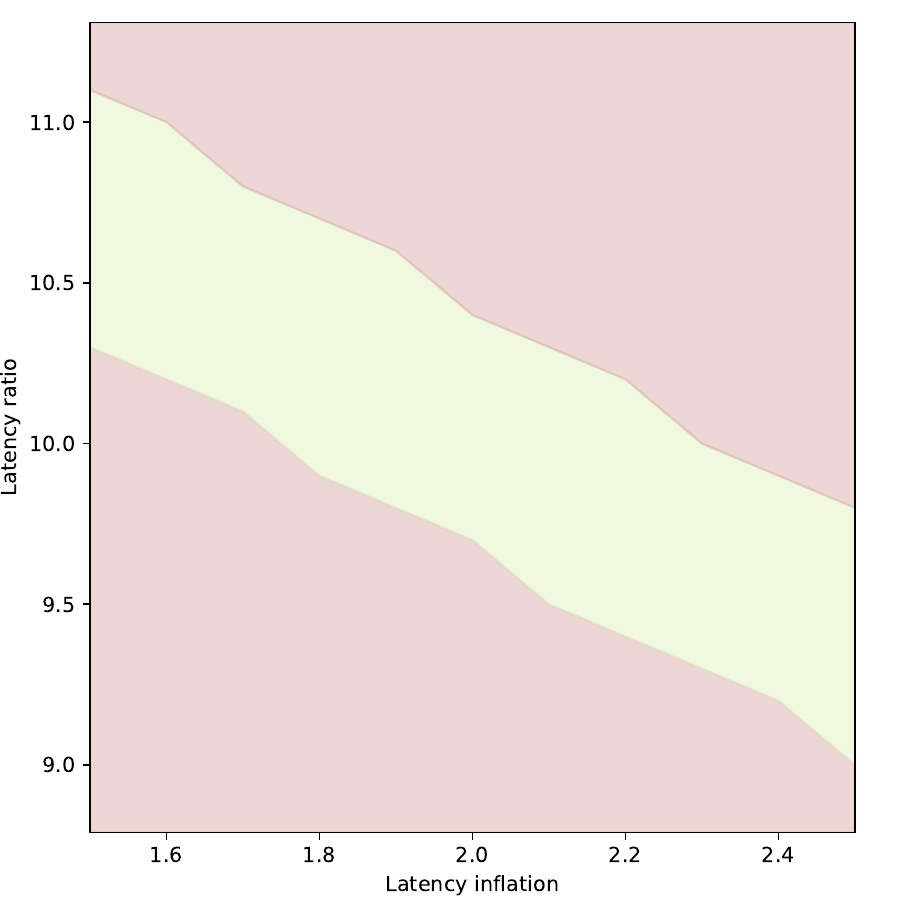}
    \caption{$\phi_{11.a}$}
\end{subfigure}
\begin{subfigure}[t]{0.24\textwidth}
    \includegraphics[width=\textwidth]  
    {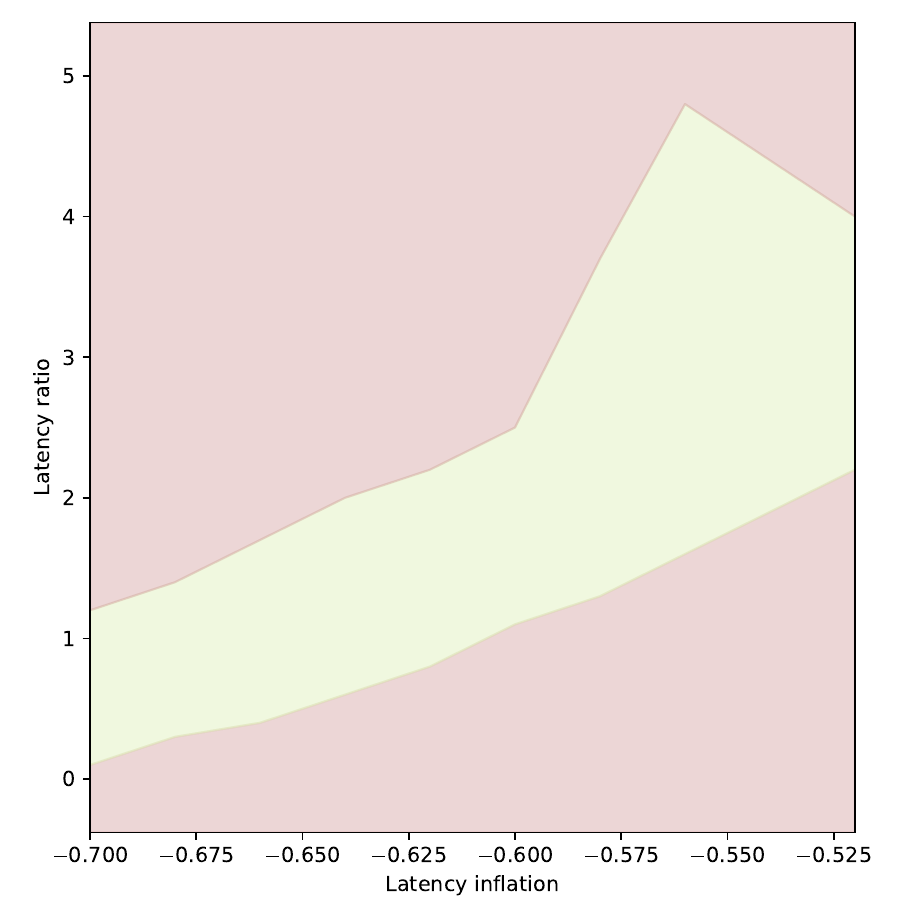}
    \caption{$\phi_{12.a}$}
\end{subfigure}

\caption{Aurora decision boundary. The red area represents the condition where properties do not hold, while the green area represents the condition where properties hold.}
\label{fg:audb}
\end{figure*}

\input{tab/Aurora/Intuitiveness_Examination.tex}

Four properties are proposed in whiRL~\cite{whiRL,whiRL_2.0} as follows. 


\textit{$\phi_{7}$:} When the history of local observations reflects excellent network conditions (close-to-minimum latency, no packet loss), the DNN should not get stuck at its current rate.

\textit{$\phi_{8}$:} When the history of local observations reflects excellent network conditions (close-to-minimum latency, no packet loss), the DNN should not constantly maintain or decrease its sending rate.

\textit{$\phi_{9}$:} When the congestion controller is sending on a link with a shallow buffer (and so experienced latency is always close to the minimum latency) and experiences high packet loss, it should decrease its sending rate.

\textit{$\phi_{10}$:} When the congestion controller is sending on a link with a shallow buffer (and so experienced latency is always close to the minimum latency) and consistently experiences high packet loss, it should not indefinitely maintain or increase its sending rate.

We verify these properties by \rf~but get a different result shown in Table~\ref{tab:ausl}. It has been confirmed by the authors of whiRL~\cite{whiRL} that it is a mistake in their papers after our feedback.

Here, we propose four robustness properties for evaluation. These properties are subsequently verified by \rf, and the outcomes are presented in Table~\ref{tab:aurub}. 
$\phi_{11}$ and $\phi_{12}$ are local robustness. $\phi_{13}$ describes a scenario where the input values of a single feature, i.e., send ratio, are extremely large; while $\phi_{14}$ describes a scenario where the input values of a single feature, i.e., send ratio, are abnormal behavior. This abnormal behavior is unrealistic in real-world contexts because the send ratio cannot be less than 1, which implies receiving more packages than sent ones.





\paragraph{Comparison.}

When compared to the verification tool whiRL, our \rf~possesses  several advantages shown in Table~\ref{tab:Advantage whiRL}. As demonstrated by the results for Aurora in Table~\ref{tab:ausl}, the error in whiRL can be attributed to the incorrect encoding of the property. This emphasizes the importance of using concise and intuitive syntax for encoding properties, which can greatly enhance the widespread and accurate application of DRL verification.

\subsubsection{Interpretation}
Additionally, we apply \rf~for the interpretation of Aurora in the context of the five interpretability questions outlined in the main text.

\paragraph{Decision boundary.}

\input{tab/Aurora/Counterfactual.tex}

Based on the previous properties of Aurora, we opt for one safety, one liveness, and two robustness properties to analyze their decision boundaries, detailed in Table~\ref{tab:aubd}. In the case of safety and liveness, the decision boundary delineates the condition where the agent takes safe or lively actions. As long as the environmental state falls within the specified space, the agent's actions align with the desired properties, averting unforeseen \texttt{bad} actions. 

The results shown in Fig. \ref{fg:audb} reveal that, within the designated range, the decision boundaries for safety and liveness properties are \textit{linear}, effectively partitioning the input space into two regions: one where the property holds and another where not. Conversely, the decision boundary for the robustness property is \textit{coarse}, segmenting the input space into three distinct zones: a central area where the property holds, flanked by two regions where not.

\paragraph{Intuitiveness examination.}
\input{tab/Aurora/Importance.tex}

\input{tab/Aurora/Sensitivity.tex}
\input{tab/Advantage_UINT.tex}

Based on the previous properties of Aurora, we evaluate \rf~to answer several intuitiveness examination questions, as shown in Table~\ref{tab:auint}. 

\textit{$\phi_{10.a}$}: This property is derived from $\phi_{10}$. \textbf{Qualitative property:} A higher send ratio, which implies more severe packet loss, is associated with a greater likelihood of the sending rate eventually decreasing. The result shows that this property only has one \text{breakpoint} for $var\in[1,100]$ following linear stepping. Within this range, the properties transition from not holding to holding. Consequently, the model passes the intuitiveness examination on this qualitative property.

\textit{$\phi_{9.a}$}: This property is derived from $\phi_{9}$. \textbf{Qualitative property:} A higher send ratio, which implies more severe packet loss, is associated with a greater likelihood of the sending rate immediately decreasing. The result shows that this property only has one \text{breakpoint} for $var\in[1,100]$ following linear stepping. Consequently, the model also passes the intuitiveness examination on this qualitative property. Interestingly, in comparison to $\phi_{4.a}$, this result reveals that safety properties are more tightened than liveness ones for the same property structure.

\textit{$\phi_{9.b}$}: This property is derived from $\phi_{9}$. \textbf{Qualitative property:} The border the constrained range of output is, which implies the more relaxed post-condition, the higher the likelihood of the property being satisfied. The result shows that this property indeed only has one \text{breakpoint} for $var\in[1,100]$ after linear search. Consequently, the model also passes the intuitiveness examination on this qualitative property. Notably, this result indicates the upper bound of the output constrained by the given input.

\textit{$\phi_{9.c}$}: This property is derived from $\phi_{9}$. \textbf{Qualitative property:} A higher latency ratio, which implies higher latency, is associated with a greater likelihood of the sending rate immediately decreasing. The result shows that this property only has one \text{breakpoint} for $var\in[1,100]$ following linear stepping. Consequently, the model also passes the intuitiveness examination on this qualitative property. Besides, the result indicates that the send rate decreases only when the latency ratio is significantly high.

\paragraph{Counterfactual explanation.}
Based on the previous properties of Aurora, we conduct the searches for the closest counterexamples that lead to altered verification results. The questions posed and their results are illustrated in Table~\ref{tab:auce}. Taking $\phi_{5.b}$ derived from $\phi_{5}$ as an example, the original output is $-1.119$; thus, it proves valuable to seek a counterexample input resulting in an output greater than 0. Remarkably, the result demonstrates that the Euclidean distance to the closest counterexample amounts to $2.511$.

Indeed, as outlined in the counterfactual explanation section of the main text, this approach can also be applied to safety and liveness properties. For instance, consider $\phi_{7.a}$ derived from $\phi_{7}$, and since $\phi_{7}$ is verified to be satisfied, we can proceed to search for its counterexamples. In this case, we are interested in finding inputs where the output is approximately 0, and we can gain insights into the conditions under which the sending rate might remain unchanged.

\paragraph{Importance analysis.}

Indeed, the assessment of feature importance can be influenced by how \texttt{output changes} are defined. In cases where the output is a floating-point number, like in Aurora, defining the extent of change in $N(x_0)$ compared to $N(\hat{x}_0)$ becomes crucial. To address this, four different levels of defining \texttt{change}, i.e., $N (x_0) \not \approx N(\hat{x} _0) \equiv (N(x_0)\ge N(\hat{x} _0)) \vee (N(x_0)\le N(\hat{x} _0)) $, have been adopted. This approach provides a range of criteria for determining change and has been applied to four original inputs $\hat{x}_0$, as detailed in Table~\ref{tab:imau}.


The evaluation results are shown in Fig.~\ref{fg:auim}. They illustrate that under the four different original inputs discussed, small perturbations of latency inflation lead to output change, and the required perturbation values are much less than the other two features.
The relationship between the importance of the latency ratio and the send ratio varies depending on the original input context. However, in the chosen experimental settings, the diverse definitions of $N(x_0) \not \approx N(\hat{x}_0)$ do not substantially affect the relative importance relationship between the required perturbation values for these three features.

To quantify feature importance, the reciprocal of the perturbation value is employed as the measure of importance. The results illustrate that latency inflation stands out as the most crucial feature in decision-making under these circumstances. This is likely attributed to the fact that the selected original inputs have latency inflation values close to 0, indicating network latency stability. As a result, even a slight perturbation to latency inflation might signify a shift to an unstable network latency condition, prompting the agent to significantly adjust its sending rate.

\begin{figure*}
\centering
\begin{subfigure}[t]{0.16\textwidth}
    \makebox[0pt][r]{\makebox[30pt]{\raisebox{40pt}{\rotatebox[origin=c]{90}{Original Input 1}}}}%
    \includegraphics[width=\textwidth]
    {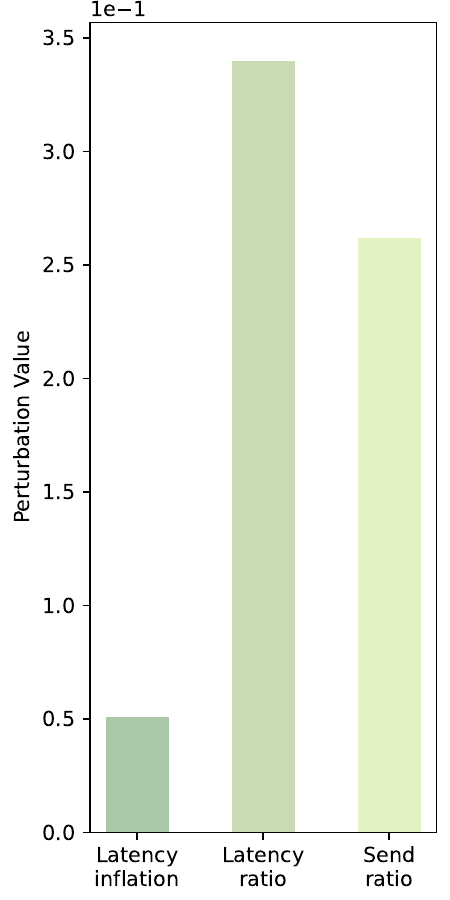}
    \makebox[0pt][r]{\makebox[30pt]{\raisebox{40pt}{\rotatebox[origin=c]{90}{Original Input 2}}}}%
    \includegraphics[width=\textwidth]
    {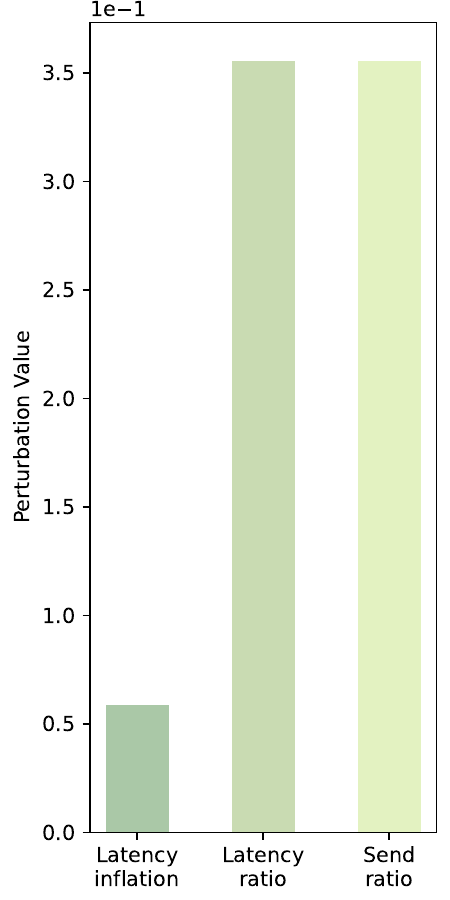}
    \makebox[0pt][r]{\makebox[30pt]{\raisebox{40pt}{\rotatebox[origin=c]{90}{Original Input 3}}}}%
    \includegraphics[width=\textwidth]
    {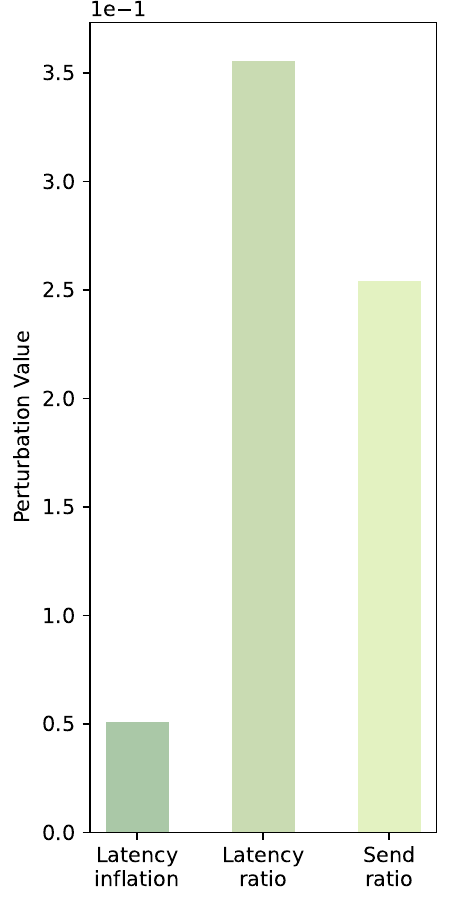}
    \makebox[0pt][r]{\makebox[30pt]{\raisebox{40pt}{\rotatebox[origin=c]{90}{Original Input 4}}}}%
    \includegraphics[width=\textwidth]
    {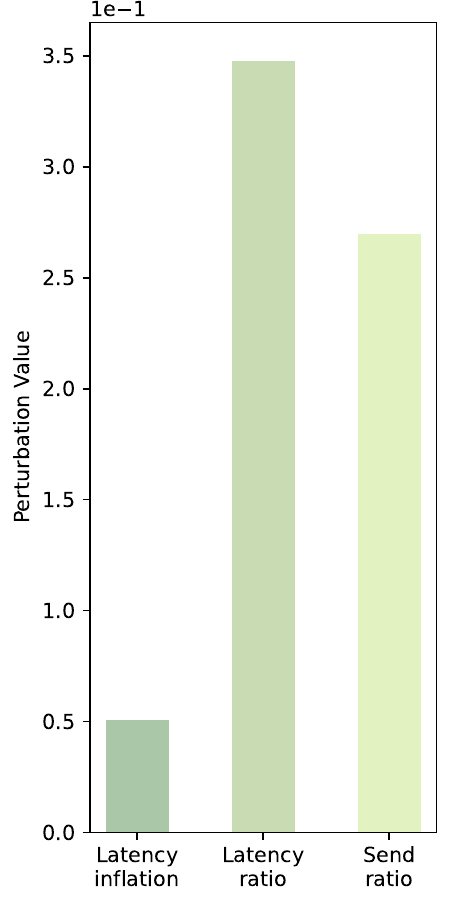}
    \caption{Level 1}
\end{subfigure}
\begin{subfigure}[t]{0.16\textwidth}
    \includegraphics[width=\textwidth]  
    {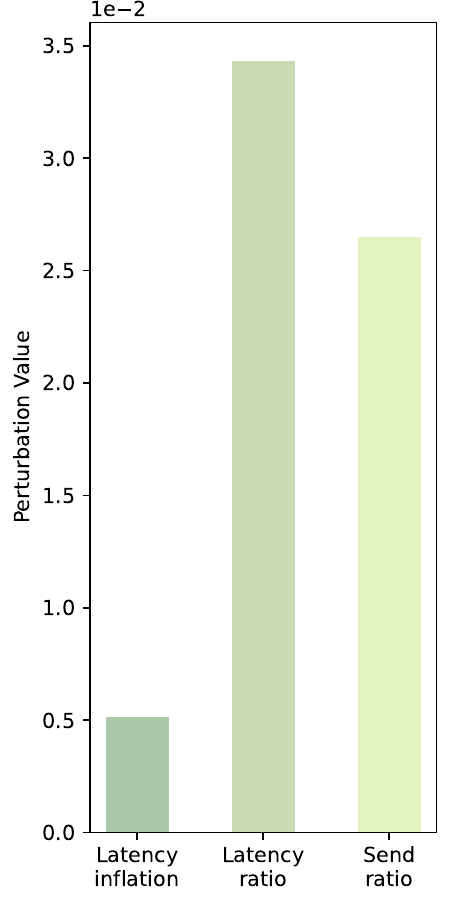}
    \includegraphics[width=\textwidth]
    {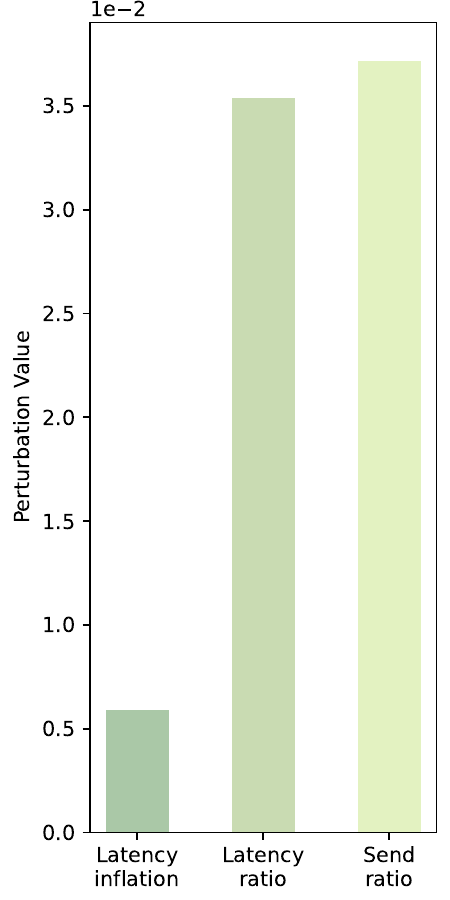}
    \includegraphics[width=\textwidth]
    {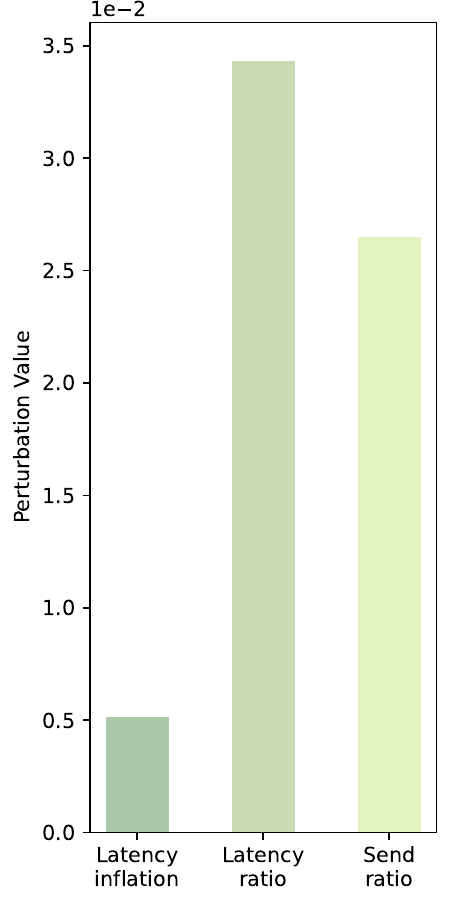}
    \includegraphics[width=\textwidth]
    {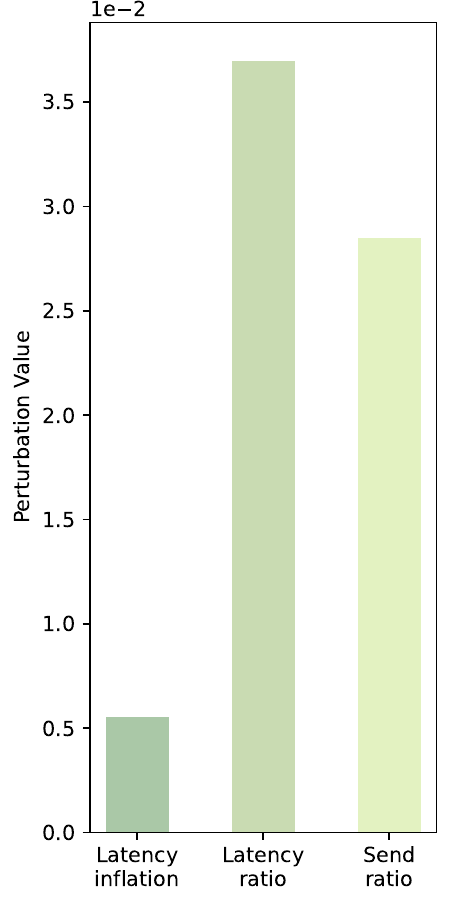}
    \caption{Level 2}
\end{subfigure}
\begin{subfigure}[t]{0.16\textwidth}
    \includegraphics[width=\textwidth]  
    {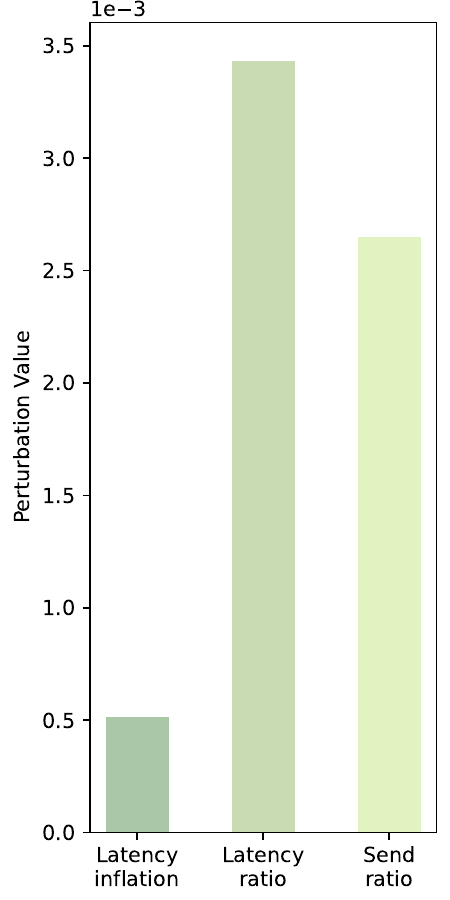}
    \includegraphics[width=\textwidth]
    {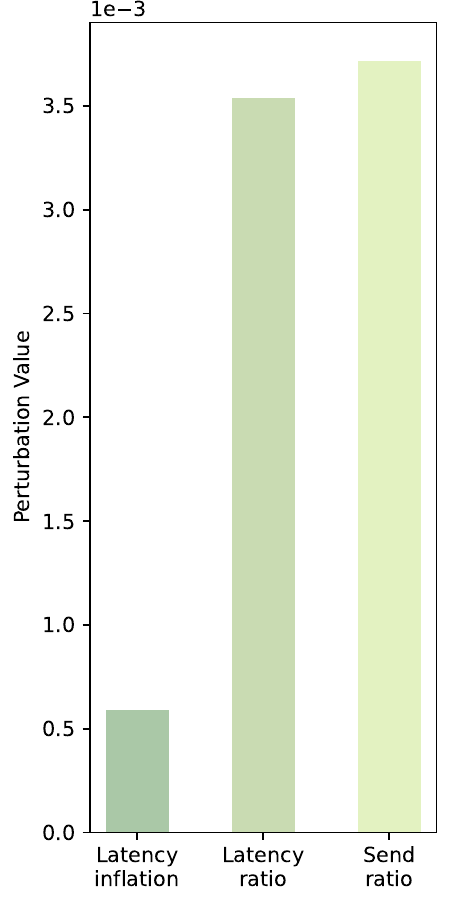}
    \includegraphics[width=\textwidth]
    {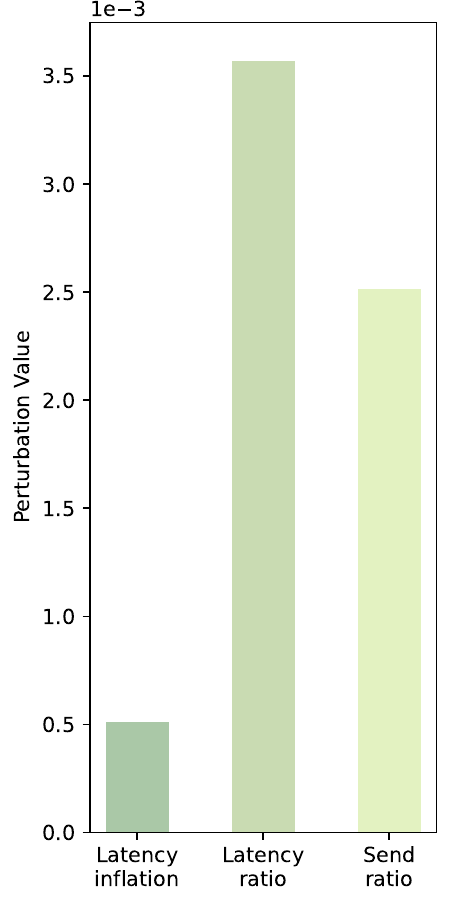}
    \includegraphics[width=\textwidth]
    {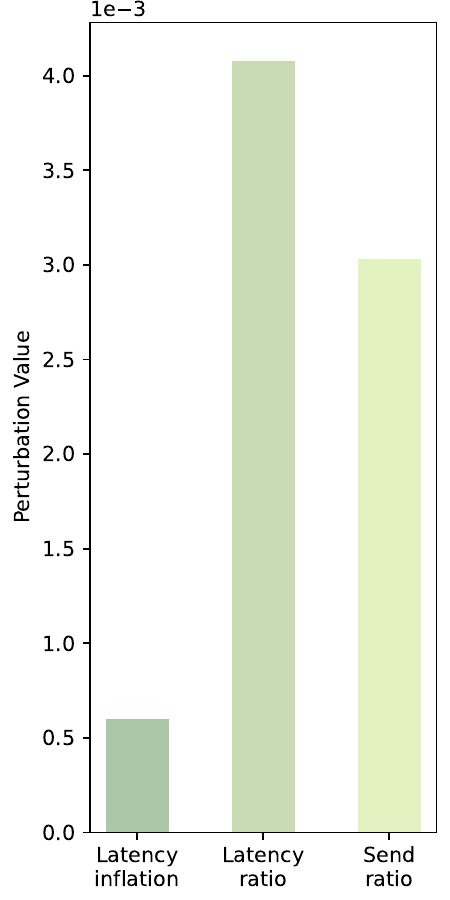}
    \caption{Level 3}
\end{subfigure}
\begin{subfigure}[t]{0.16\textwidth}
    \includegraphics[width=\textwidth]  
    {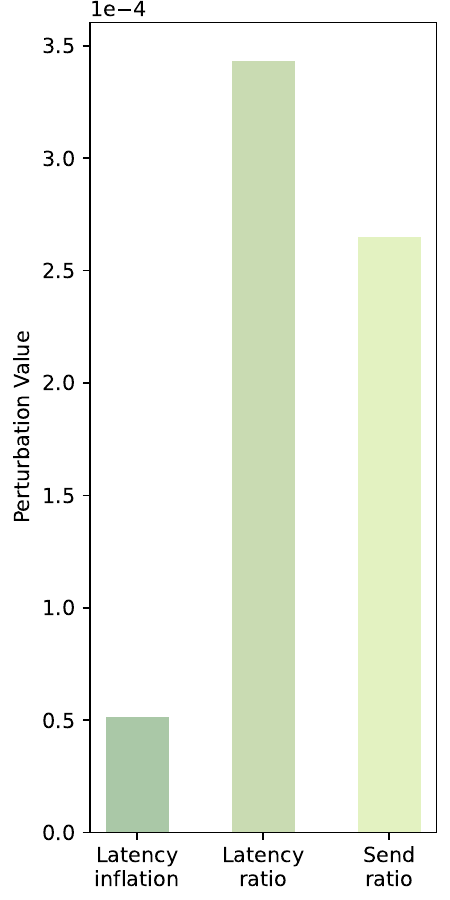}
    \includegraphics[width=\textwidth]
    {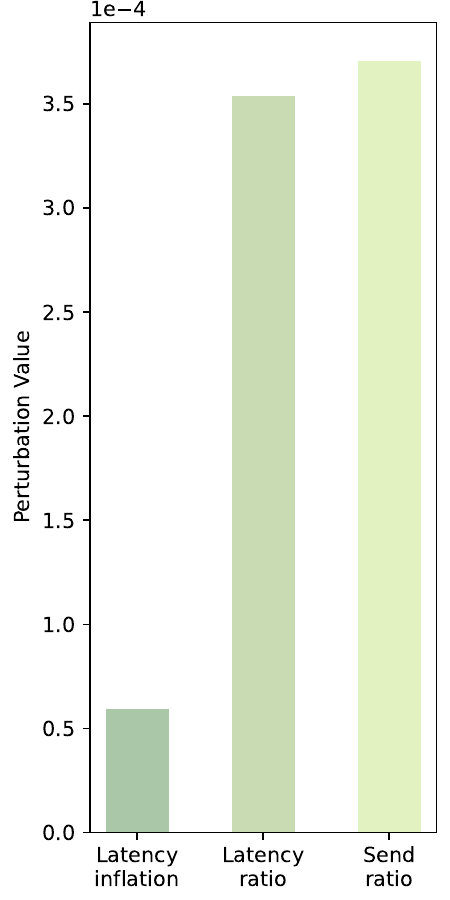}
    \includegraphics[width=\textwidth]
    {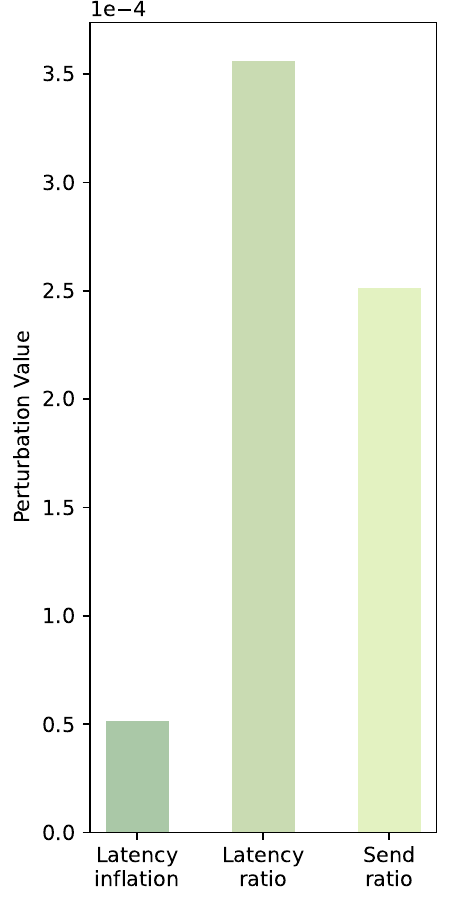}
    \includegraphics[width=\textwidth]
    {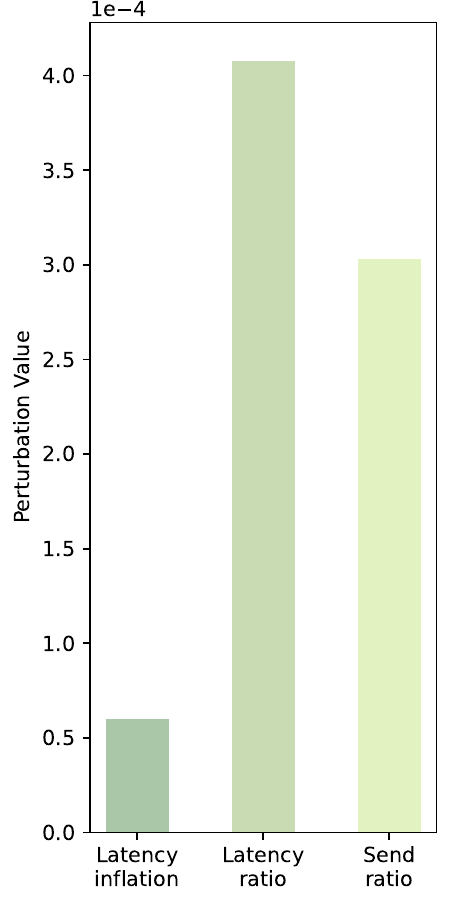}
    \caption{Level 4}
\end{subfigure}
\begin{subfigure}[t]{0.32\textwidth}
    \includegraphics[width=\textwidth]  
    {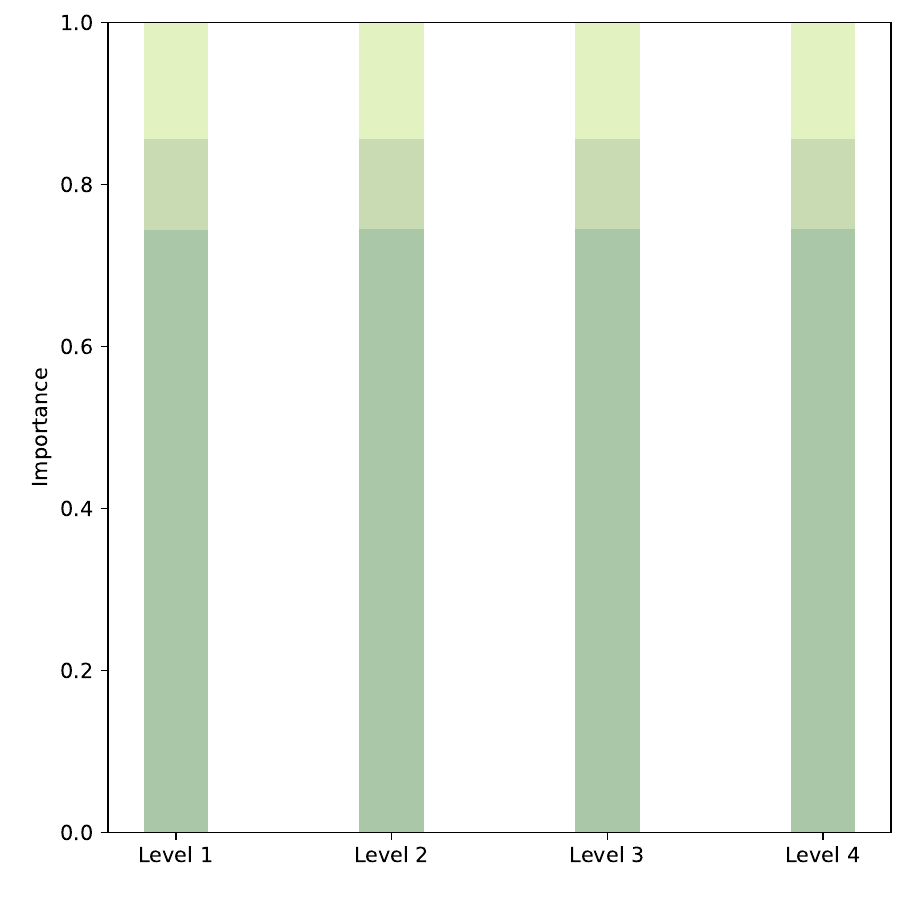}
    \includegraphics[width=\textwidth]
    {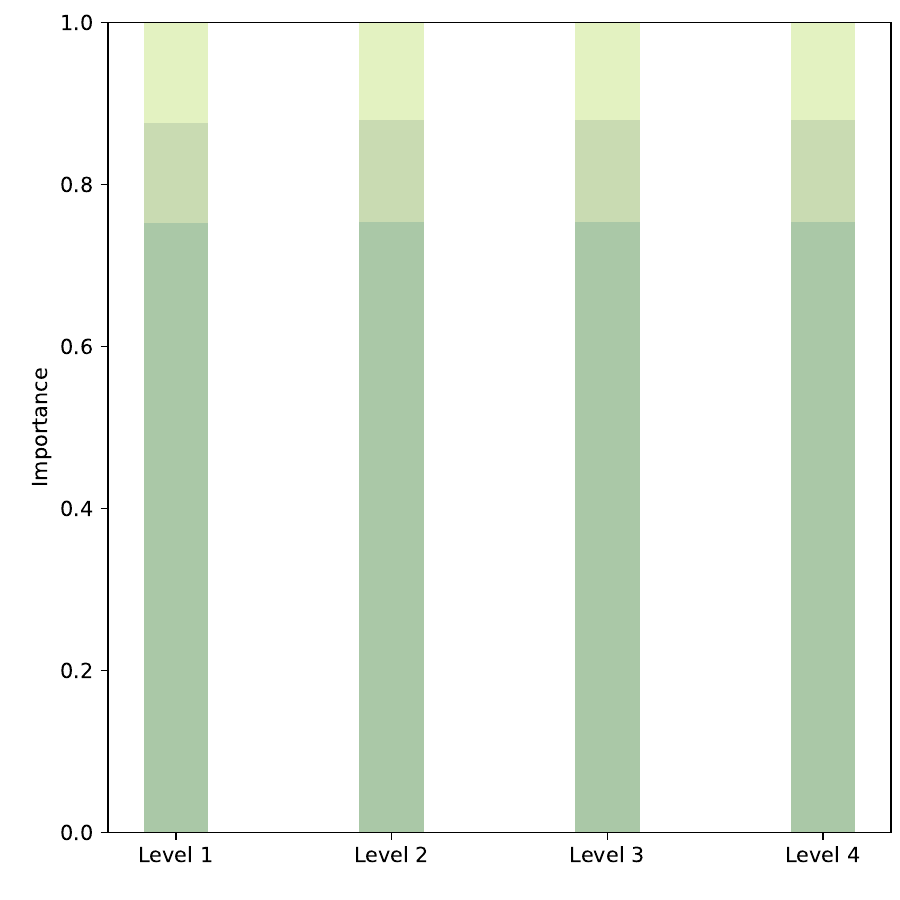}
    \includegraphics[width=\textwidth]
    {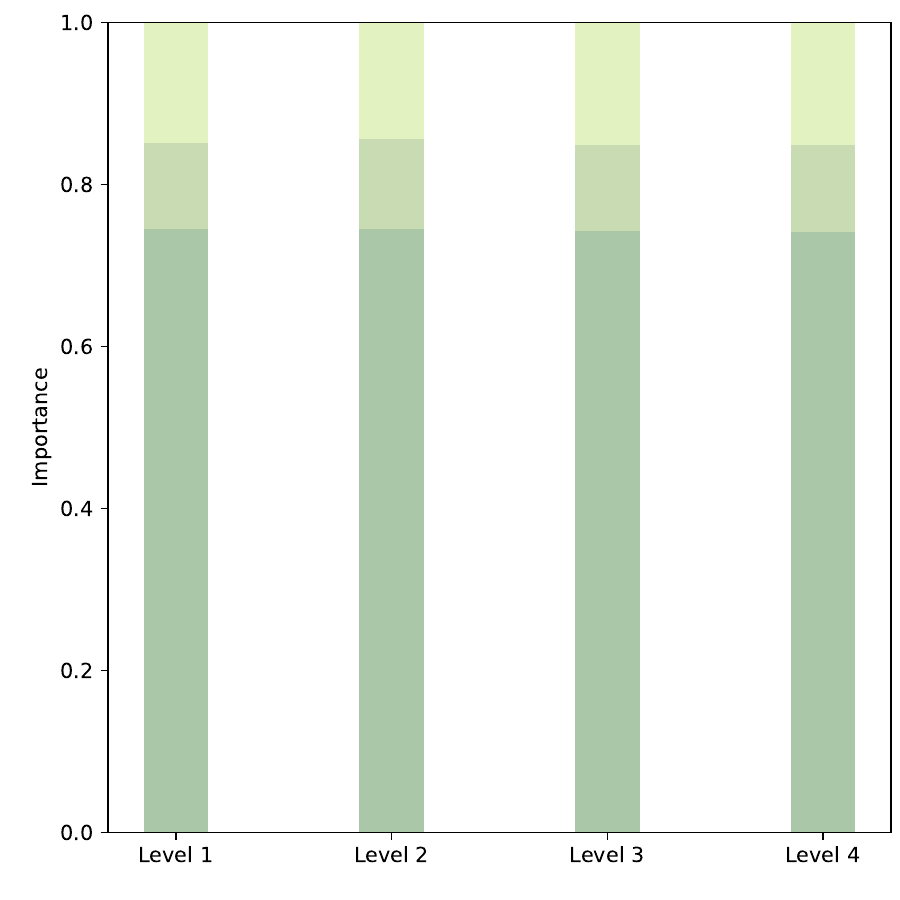}
    \includegraphics[width=\textwidth]
    {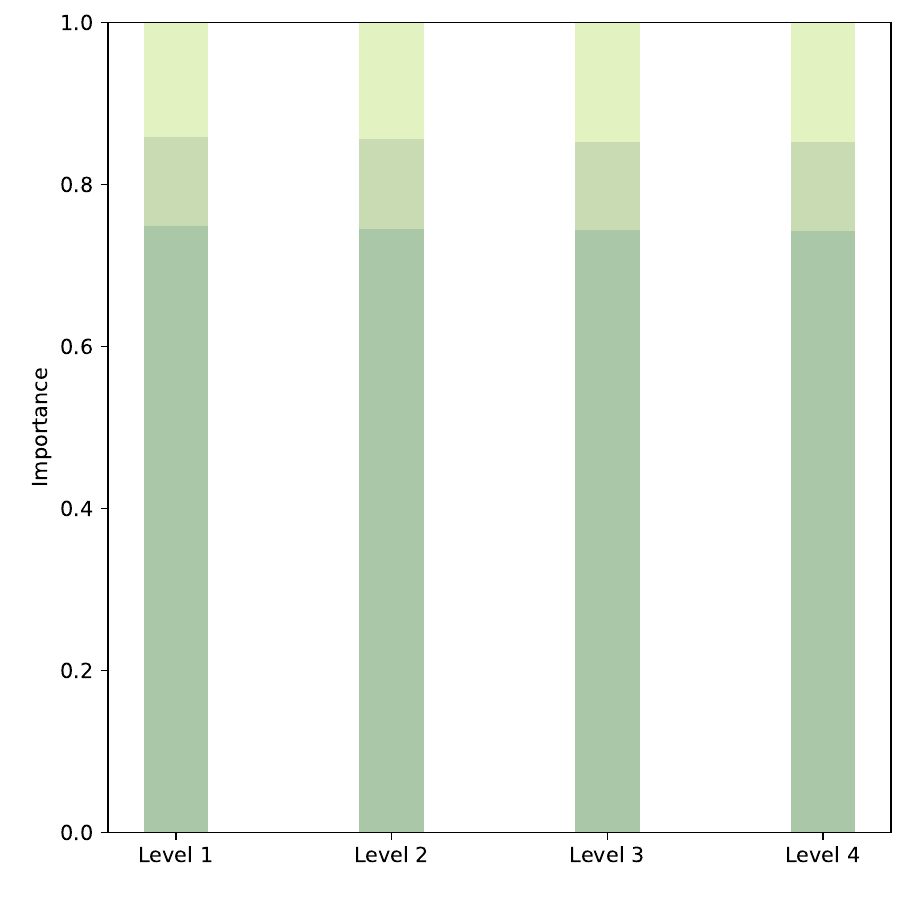}
    \caption{Importance}
\end{subfigure}
\caption{Aurora importance analysis results. Each row indicates different original input $\hat{x}_0$ as shown in Table~\ref{tab:imau}. The first to fourth columns indicate different $N(x_0) \not \approx N(\hat{x_0})$ Levels as shown in Table~\ref{tab:imau}, and the fifth column indicates the proportion of Importance ($1/Perturbation\ value$) of each feature under the three features of the input.}
\label{fg:auim}
\end{figure*}


\begin{figure*}
\centering
\begin{subfigure}[t]{0.16\textwidth}
    \makebox[0pt][r]{\makebox[30pt]{\raisebox{40pt}{\rotatebox[origin=c]{90}{Original Input 1}}}}%
    \includegraphics[width=\textwidth]
    {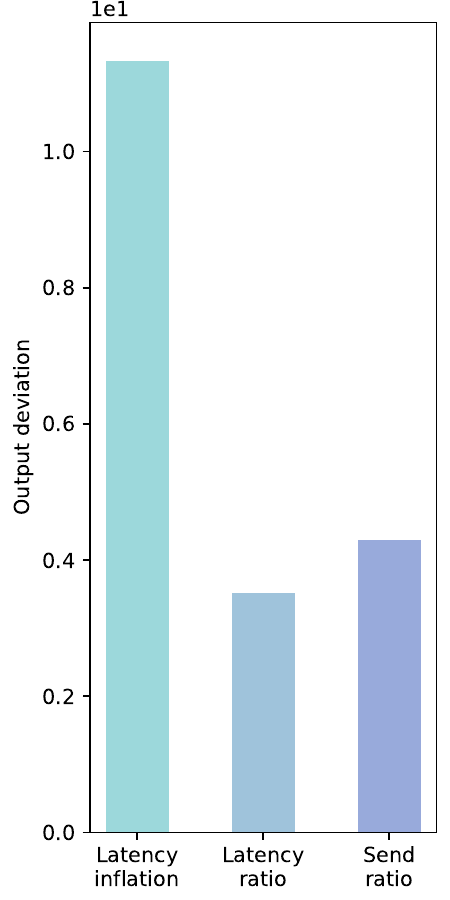}
    \makebox[0pt][r]{\makebox[30pt]{\raisebox{40pt}{\rotatebox[origin=c]{90}{Original Input 2}}}}%
    \includegraphics[width=\textwidth]
    {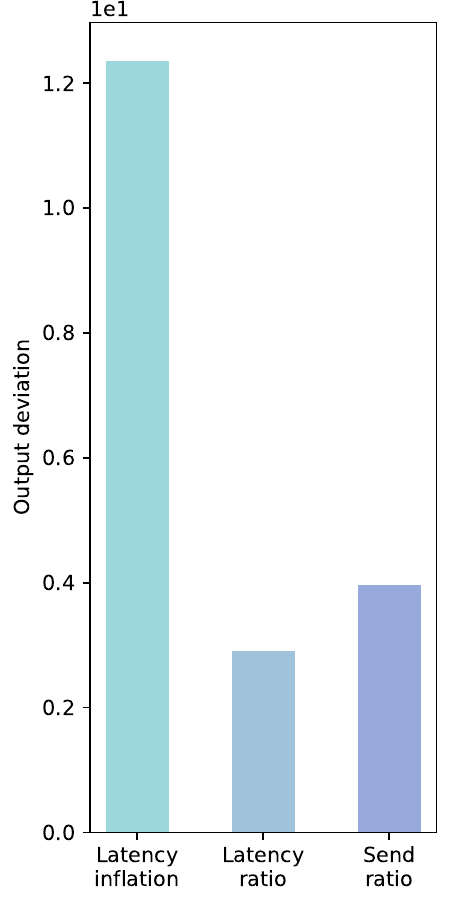}
    \makebox[0pt][r]{\makebox[30pt]{\raisebox{40pt}{\rotatebox[origin=c]{90}{Original Input 3}}}}%
    \includegraphics[width=\textwidth]
    {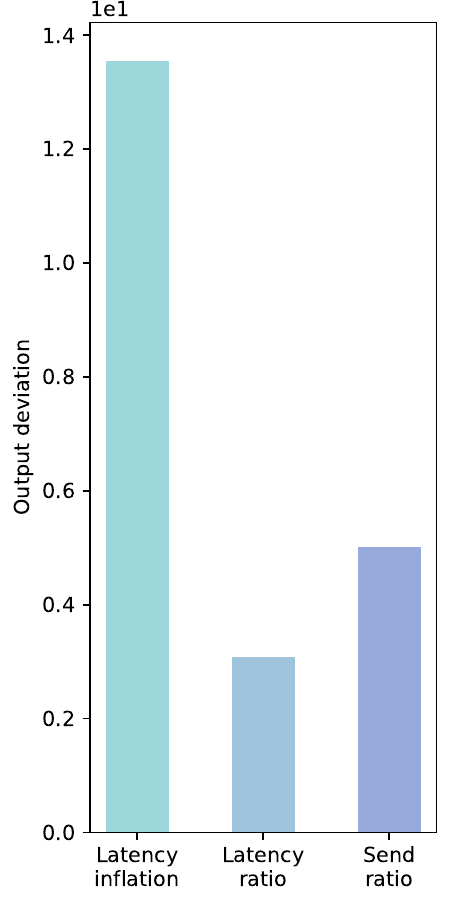}
    \makebox[0pt][r]{\makebox[30pt]{\raisebox{40pt}{\rotatebox[origin=c]{90}{Original Input 4}}}}%
    \includegraphics[width=\textwidth]
    {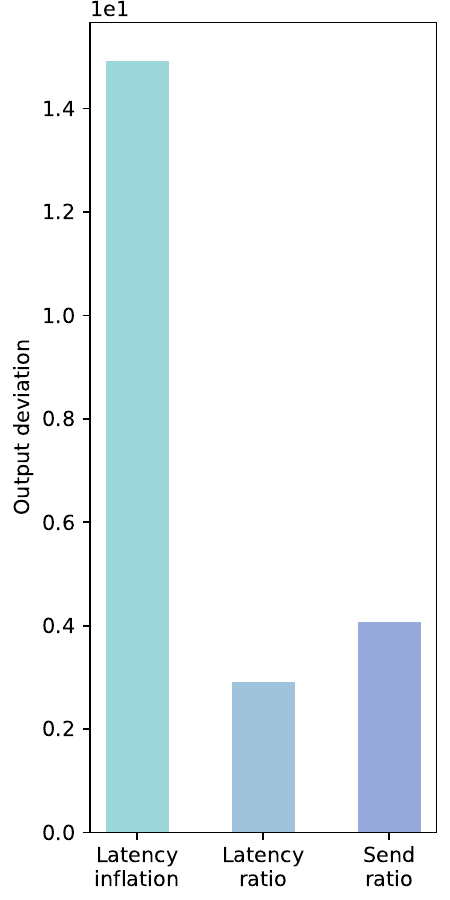}
    \caption{Level 1}
\end{subfigure}
\begin{subfigure}[t]{0.16\textwidth}
    \includegraphics[width=\textwidth]  
    {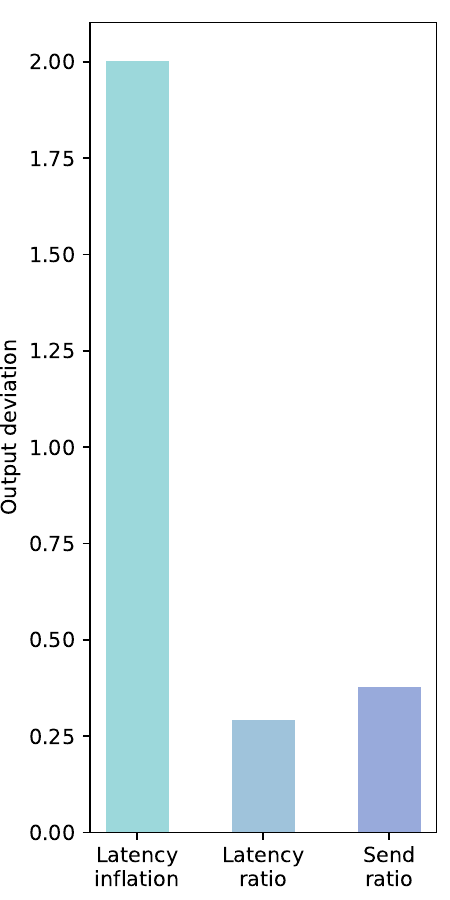}
    \includegraphics[width=\textwidth]
    {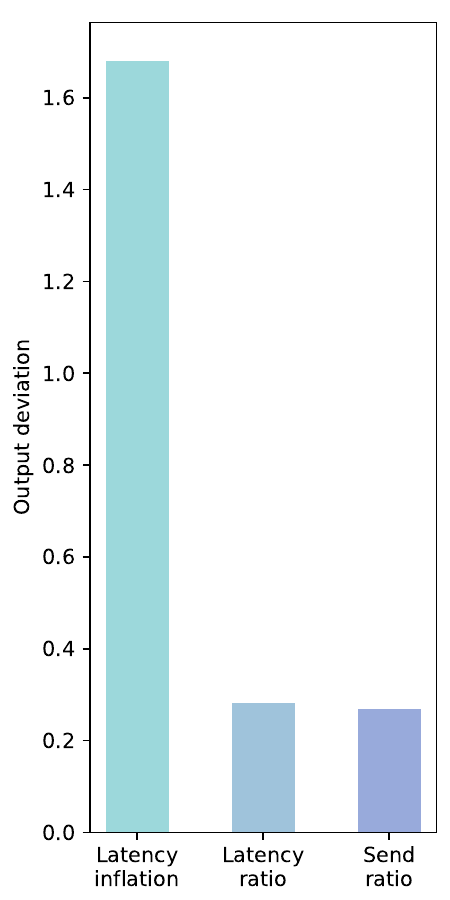}
    \includegraphics[width=\textwidth]
    {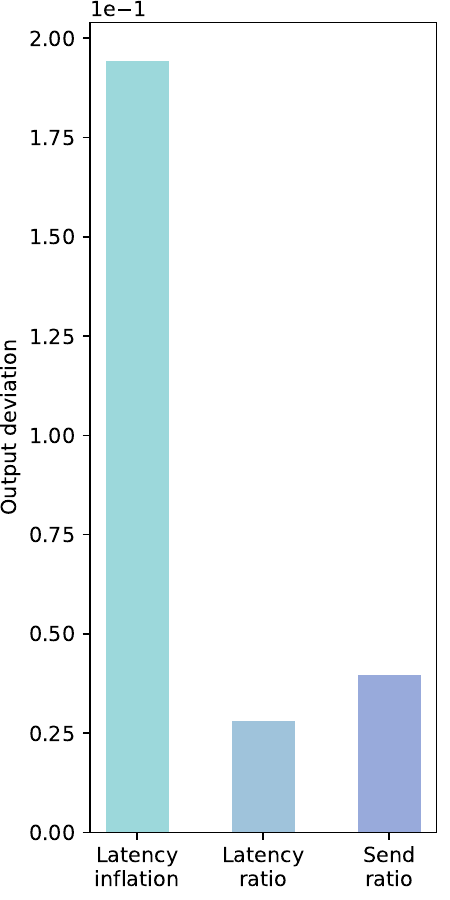}
    \includegraphics[width=\textwidth]
    {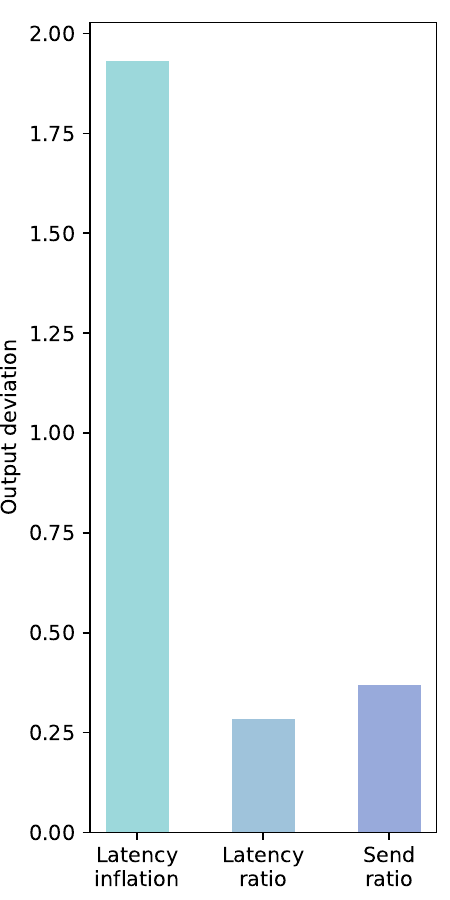}
    \caption{Level 2}
\end{subfigure}
\begin{subfigure}[t]{0.16\textwidth}
    \includegraphics[width=\textwidth]  
    {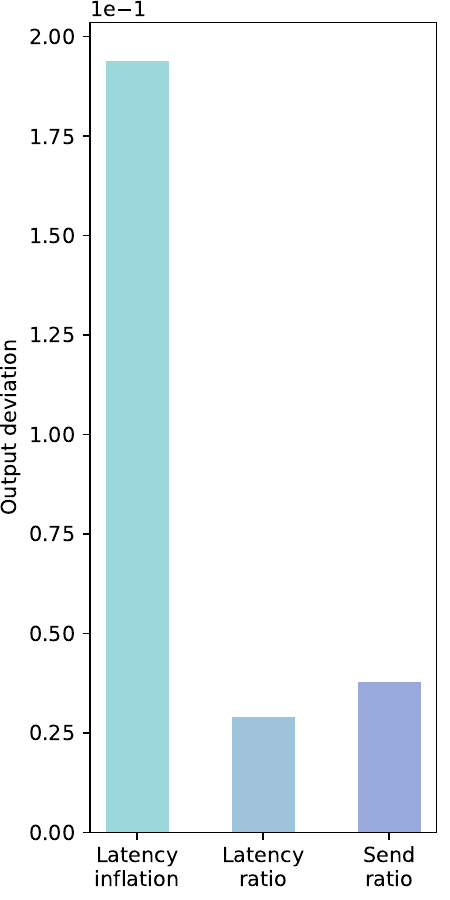}
    \includegraphics[width=\textwidth]
    {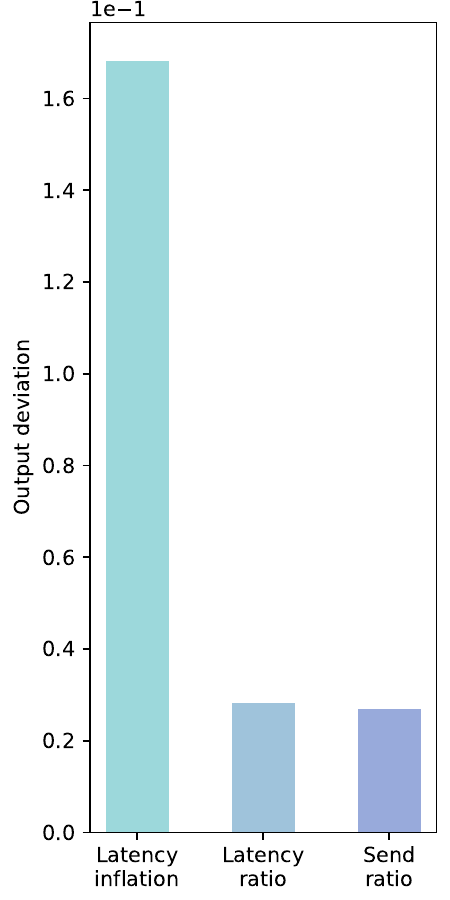}
    \includegraphics[width=\textwidth]
    {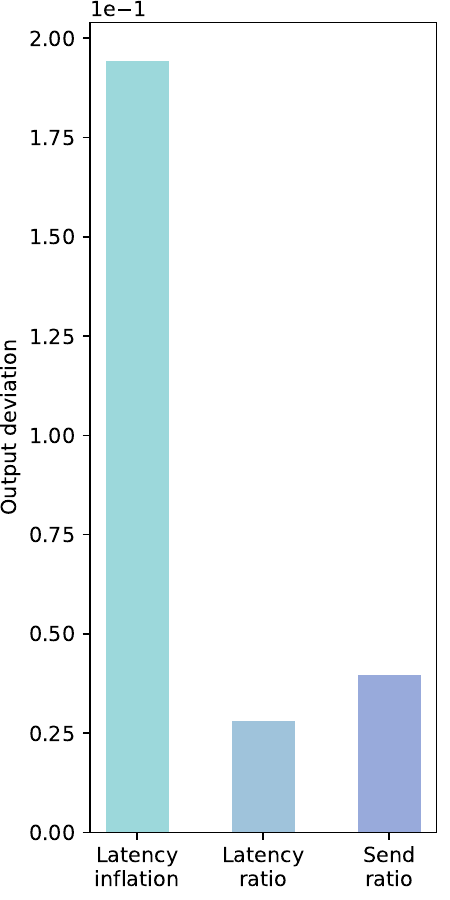}
    \includegraphics[width=\textwidth]
    {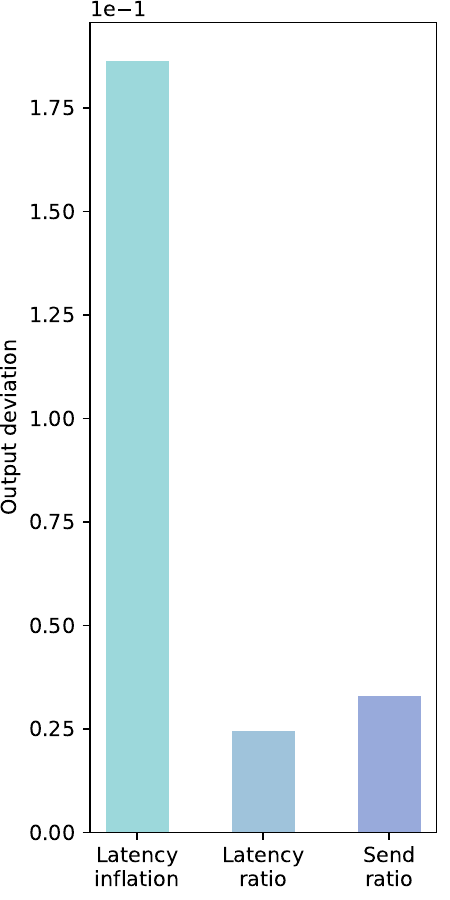}
    \caption{Level 3}
\end{subfigure}
\begin{subfigure}[t]{0.16\textwidth}
    \includegraphics[width=\textwidth]  
    {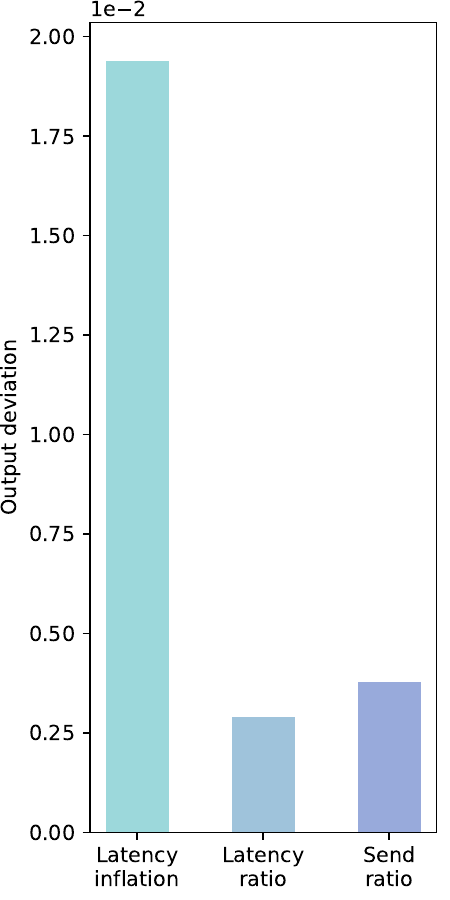}
    \includegraphics[width=\textwidth]
    {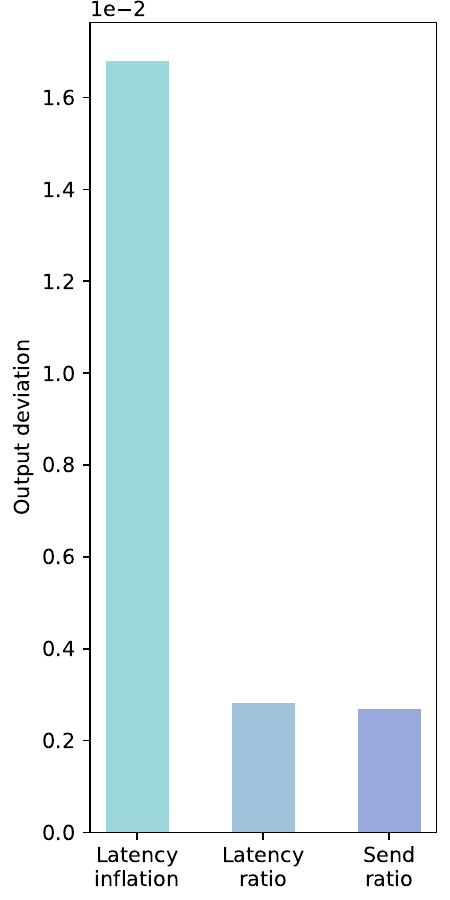}
    \includegraphics[width=\textwidth]
    {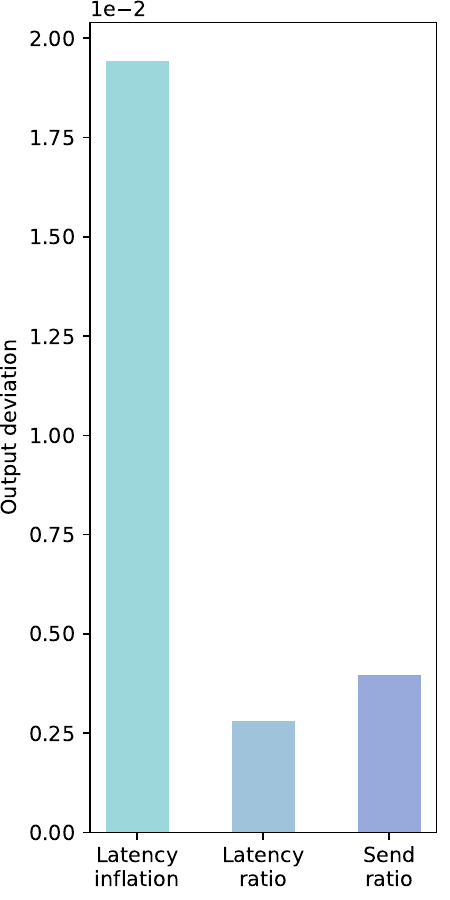}
    \includegraphics[width=\textwidth]
    {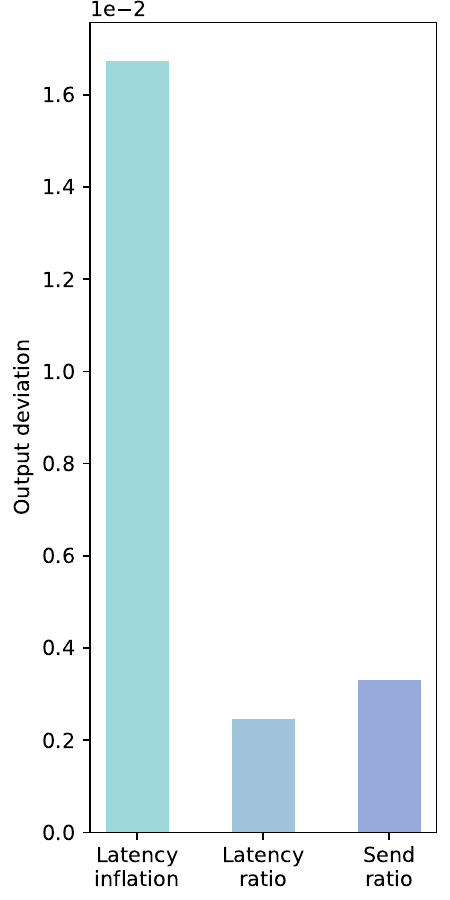}
    \caption{Level 4}
\end{subfigure}
\begin{subfigure}[t]{0.32\textwidth}
    \includegraphics[width=\textwidth]  
    {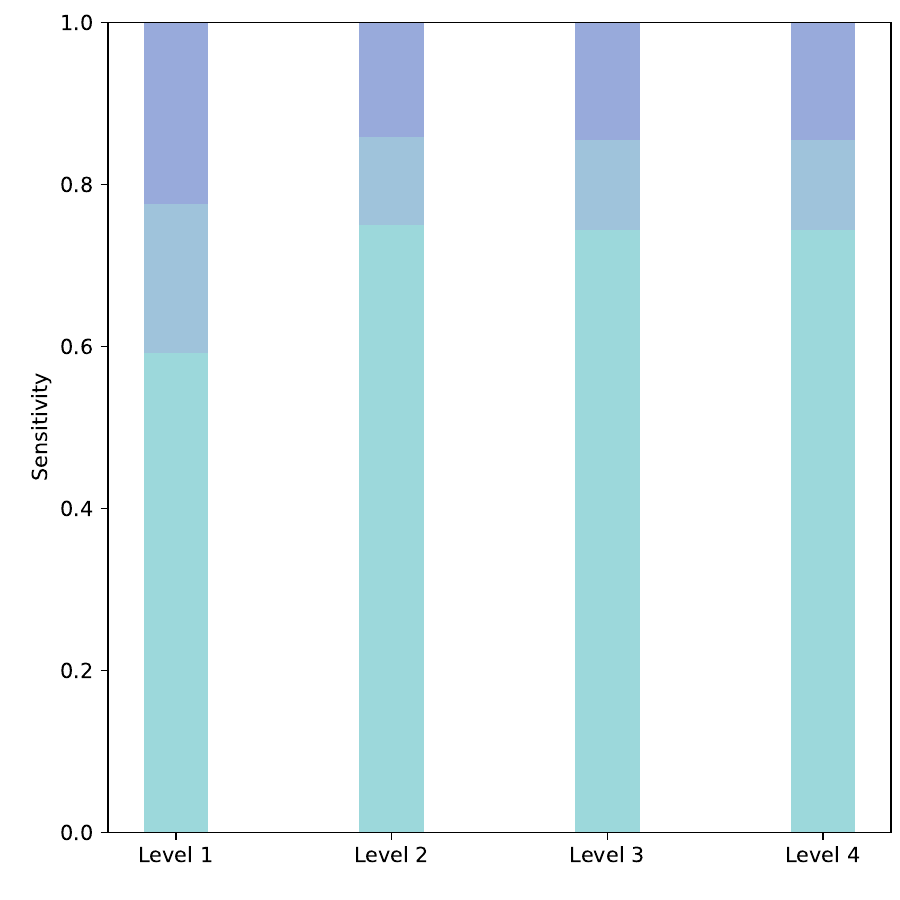}
    \includegraphics[width=\textwidth]
    {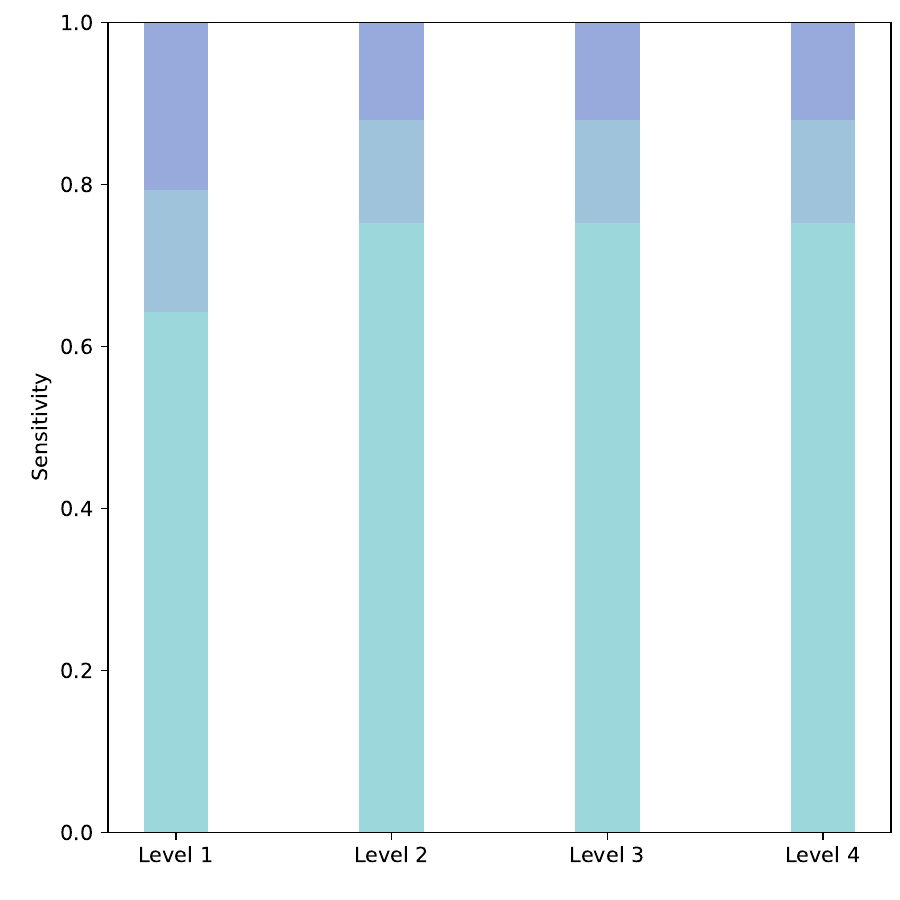}
    \includegraphics[width=\textwidth]
    {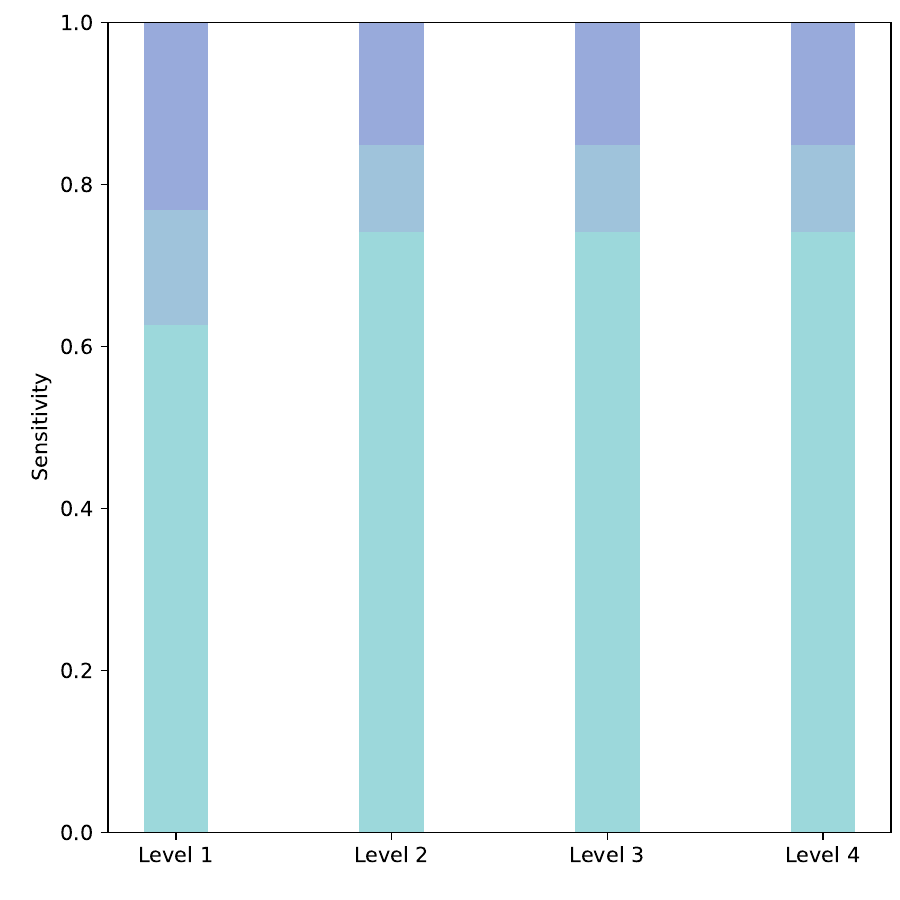}
    \includegraphics[width=\textwidth]
    {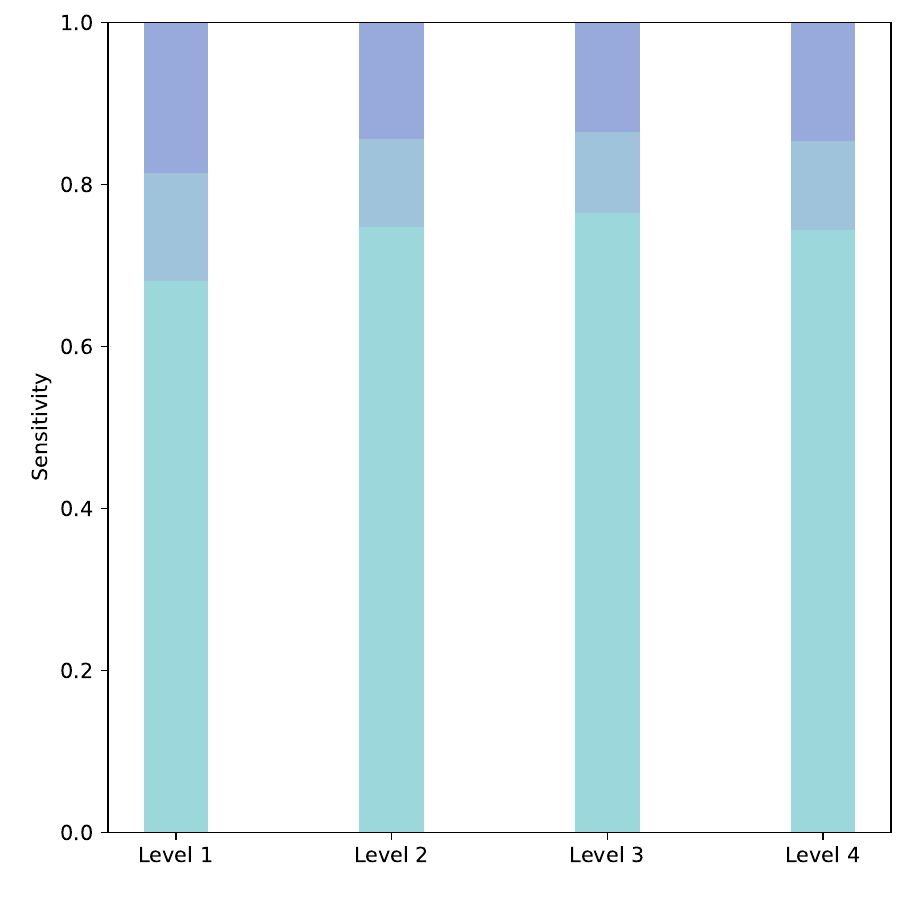}
    \caption{Sensitivity}
\end{subfigure}

\caption{Aurora sensitivity analysis results. Each row indicates different original input $\hat{x}_0$ as shown in Table~\ref{tab:seau}. The first to fourth columns indicate different Perturbation Levels as shown in Table~\ref{tab:seau}, and the fifth column indicates the proportion of Sensitivity ($Output\ deviation$) of each feature under the three features of the input.}
\label{fg:ause}
\end{figure*}


\paragraph{Sensitivity analysis.}

The sensitivity of features, as measured by the deviation of the output from its original value, varies based on the chosen perturbation levels for the features. To comprehensively analyze this, we employ four different perturbation levels denoted by $\varepsilon$ and utilize the same set of four original inputs $\hat{x}_0$ that were used in the Importance Analysis. The perturbation levels, as well as the corresponding original inputs, are listed in Table~\ref{tab:seau}. This approach helps us to explore how different perturbation levels influence the sensitivity of the features under evaluation.

The evaluation results are shown in Fig.~\ref{fg:ause}. They reaffirm the finding that perturbing latency inflation has the most substantial impact on the output, consistent with the observations from the Importance Analysis. Therefore, latency inflation remains the most sensitive feature under the given original inputs. This correlation between importance and sensitivity is reasonable, as important features tend to exhibit higher sensitivity due to their significant influence on the output. The high sensitivity of latency inflation suggests that slight variations in network latency can lead to noticeable changes in the agent's decision-making. This aligns with the notion that changes in latency may indicate shifts in network stability, prompting the agent to adjust its sending rate accordingly.

\paragraph{Comparison.}
\input{tab/Reintrainer/TrainProcess.tex}

\input{tab/Reintrainer/ablation.tex}
\input{tab/hp.tex}

\input{tab/Reintrainer/propdef.tex}
When compared to the interpreter UINT, our \rf~possesses several advantages shown in Table~\ref{tab:Advantage UINT}. Beyond its benefits in terms of user-friendliness and applicability to a wide range of scenarios, a crucial aspect is our conviction that our cohensive solutions based on \textit{breakpoints} can be rapidly adapted to new interpretability problems instead of formulating distinct solutions for each individual problem. While the process of searching for \textit{breakpoints} might be time-consuming, once these breakpoints are identified, an array of interpretable problems can be answered within seconds. This further underscores the scalability of our solutions.

\subsection{Reintrainer}

\subsubsection{Training Process Comparison.} \label{appx:sec_reintrainer_proc}

We track the total training time and the numbers of timesteps spent to develop high-performing models
as shown in Table~\ref{tab:train_proc}.
The statistics show that the vanilla algorithms have an advantage in total training time, while our approach takes more time due to the inclusion of formal verification and interpretation in training strategy generation. Nonetheless, we strongly believe that this additional time investment is justified, as it substantially increases the reliability of the models. 
Trainify's verification method, which treats the DNN as a black box, should have an efficiency edge. However, it is only slightly faster than \rt~in three benchmarks.
Speed improvements of \rt~can be achieved by parallelizing strategy generation processes or by reducing their precision.

Nevertheless, our approach holds a significant advantage over Trainify in terms of timesteps in the majority of tasks. This aligns with our approach's capability to provide precise feedback on reward shaping. It suggests that if the agent runs in an environment with high action costs, our method will still offer advantages in terms of efficiency.

We also track the number of violation occurrences. In conditions where violations occur frequently, such as in $\phi_{1}$, our approach demonstrates advantages in reducing violations. However, when there are few or no violations, such as in $\phi_{2}$, Reintrainer can actively generate a smaller number of counterexamples in the  buffer for the agents to learn. $\phi_{3}$ is aimed to guide the agent to try to make the pendulum more upright, so more violations occurring during the training process of Reintrainer to promote  property satisfaction of the final model are acceptable.

\subsubsection{Ablation Study}

Since there are a significant numerical difference in only the Pendulum task among the six benchmarks, we choose it for the ablation study. We obtain the following reward shaping strategies by selecting whether to enable the metrics in our algorithm.

\textbf{Const}: the fixed penalty when the violation is detected; here, we use $-2$ as the penalty. This strategy can be seen as a purely counterexample-guided training strategy.

\textbf{Raw Dist}: the penalty calculated as the sum of distances from the violating state to the property constraint boundary, using Eq.\ref{eq:raw}.

\textbf{Dist}: the penalty calculated by using Eq.\ref{eq:mid} and Eq.\ref{eq:dist} when a violation is detected, with the density set to 1. 

\textbf{Density}: the penalty calculated by measuring density using Eq.\ref{eq:rho} instead of setting it to 1, building on \textbf{Dist}. 

\textbf{Gap}: the penalty calculated by introducing \textsl{gap} metric and adjusting the learning rate, building on \textbf{Density}; here, $Lr(gap)=\frac{1}{1+e^{-gap}}$. 

Since $\phi_3$ is a single-step property, Traceback is not used, and $Gap$ represents the complete reward performance of \rt~under general conditions. 

We independently repeated each penalty condition 20 times and evaluated each with 100 episodes, calculating the mean and variance of the average episode reward over 20 independent repetitions, as shown in Table \ref{tab:abl}.

Experimental results indicate that our algorithm performs progressively better with more precise metrics and eventually surpasses the fine-tuned vanilla algorithm while ensuring property satisfaction.
\textbf{Const} is the simplest reward shaping method, using a constant value that makes it difficult for the learning algorithm to estimate the extent of violations, hence it performs the worst.

\textbf{Raw Dist} simply calculates the sum the distances from the violation to the nearest property-constrained, and compared to \textbf{Const}, it measures the violation more precisely, thereby showing improved performance. 

\textbf{Dist} in contrast to \textbf{Raw Dist}, calculates the overall distance instead of the sum of distances across individual dimensions and also normalizes the results to prevent issues arising from different dimensional scales, thus making significant advancements to enhance performance.
 
\textbf{Density} introduces density metric, which theoretically should provide more accurate measurements than \textbf{Raw Dist}, but the results here show a decline. We analyze that this might be related to the short training cycles of the Pendulum. In this experiment, the \texttt{learning starts} parameter is set to 10,000 timesteps, and a good model can be learned by 17,000 timesteps. This means the density calculated before learning starts is meaningless and may interfere with learning. Therefore, when the timesteps from learning start to model convergence are few, the density metric could be disabled.

\textbf{Gap} more accurately indicates the progress of property learning, thus reducing the impact of property learning on performance goals when the property learning is close to convergence, i.e., when the \textsl{gap} is small. In this experiment, it significantly improved performance even when density performed poorly, highlighting the importance of numerical results from verification for assessing the model's current state and learning.


\subsubsection{Training Hyper-Parameters}\label{appx:sec_reintrainer_hp}

We employ fine-tuned hyper-parameters from Stable Baselines~\cite{sb3} to train models as baselines of vanilla algorithms for tasks Mountain Car, Cartpole, and Pendulum. And we employ the same hyper-parameters from Trainify~\cite{trainify} to train models of both vanilla algorithms and Reintrainer for task B1, B2 and Tora.
All hyperparameters are shown in Table~\ref{tab:hyper_parameter_pd1}.

\subsubsection{Predefined Properties of Benchmarks.}\label{Appx:rf_prop}
The detailed definitions of properties in the benchmarks used in the main text are shown in Table~\ref{tab:train_propappx}.

\subsubsection{Transition Dynamics of Benchmarks.}
The following benchmarks are not built-in Gym\cite{gym} environments. Here are their transition dynamics.
\paragraph{B1.} $s_{t+1}^0=s_t^1, s_{t+1}^1=a(s_{t}^1)^2-s_t^0$
\paragraph{B2.} $s_{t+1}^0=s_t^1-(s_t^0)^3, s_{t+1}^1=a$
\paragraph{Tora.} $s_{t+1}^0=s_{t}^1, s_{t+1}^1=-s_{t}^0+0.1\sin s_{t}^2, s_{t+1}^2=s_{t}^3, s_{t+1}^3=a$

%% file: tab/Advantage_whiRL.tex
\begin{table*}[htb]
    \centering
    
    \begin{tabularx}{\textwidth}{p{2.5cm}XXp{3.8cm}}
    \hline
                     & \textbf{whiRL}                                                 & \textbf{Reinfier}                                           & Benefit                                               \\ \hline
    Input Format     & Hardcoded DNN verifier APIs   \& without APIs for DRL  & Separate DRLP text into user-friendly APIs           & Enhances reusability                                  \\ \hline
    Property Syntax  & Syntax of the DNN verifier                                     & Intuitive and concise DRLP  syntax            & Make it easy for users to read and write              \\ \hline
    DNN Structure    & A small amount of structures                                   & Almost all structures                                       & Support more DRL systems  and adapt to the rapid development of DNN                 \\ \hline
    Result Guarantee  & Results from a single DNN  verifier                    & Comprehensive results from   multiple DNN verifiers & Reduce the risk of verification error                 \\ \hline
    Type of Property & Safety, Liveness                                               & Safety, Liveness, Robustness                                & Facilitate thorough property  verification    \\ \hline
    \end{tabularx}
    \caption{Advantages of Reinfier compared to whiRL.}
    
    \label{tab:Advantage whiRL}
\end{table*}

%% file: tab/Aurora/Decision_Boundary.tex
\begin{table*}[htb]
    \centering
    
    \begin{tabularx}{\textwidth}    {p{2.8cm}Xp{2.8cm}XX}
    \hline
     & \textbf{$\phi_{9.d}$} & \textbf{$\phi_{7.a}$} & \textbf{$\phi_{11.a}$} & \textbf{$\phi_{12.a}$} 
    \\ \hline
     Type  & Safety & Liveness & Robustness & Robustness
     \\ \hline
     Variables 
     & 
           {
          \setlength{\abovedisplayskip}{-8pt}
          \setlength{\belowdisplayskip}{-8pt}
          \begin{alignat*}{8}
          &a &&\in [&1&.0&&,&10&.0&&] \\
          &b &&\in [& 1&&&,& 100&&&] 
          \end{alignat*}}
     & 
        {
          \setlength{\abovedisplayskip}{-8pt}
          \setlength{\belowdisplayskip}{-8pt}
          \begin{alignat*}{8}
          &a &&\in [&-0&.70&&,&-0&.50&&] \\
          &b &&\in [& 1&.0&&,& 20&.0&&] 
          \end{alignat*}}
     & 
         {
          \setlength{\abovedisplayskip}{-8pt}
          \setlength{\belowdisplayskip}{-8pt}
          \begin{alignat*}{8}
          &a &&\in [& 1&.5&&,&2&.5&&] \\
          &b &&\in [& 6&.0&&,& 12&.0&&] 
          \end{alignat*}}

     & 
         {
          \setlength{\abovedisplayskip}{-8pt}
          \setlength{\belowdisplayskip}{-8pt}
          \begin{alignat*}{8}
          &a &&\in [&-0&.70&&,&-0&.50&&] \\
          &b &&\in [& 0&.0&&,& 5&.0&&] 
          \end{alignat*}}
     \\ \hline
    State Boundary $S$ \newline or\newline Initial State $I$ 
      & 
        {
          \setlength{\abovedisplayskip}{-8pt}
          \setlength{\belowdisplayskip}{-8pt}
          \begin{alignat*}{8}
          &i &&\in [& a&&&,&100&&&] \\
          &l &&\in [& b&&&,&100&&&] \\
          &r &&\in [& 2&&&,&100&&&] 
          \end{alignat*}}
      & 
          {
          \setlength{\abovedisplayskip}{-8pt}
          \setlength{\belowdisplayskip}{-8pt}
          \begin{alignat*}{8}
          &i &&\in [& a&&&,&a&&&] \\
          &l &&\in [& b&&&,&b&&&] \\
          &r &&\in [& 1&&&,&1&&&] 
          \end{alignat*}}
      & 
          {
          \setlength{\abovedisplayskip}{-8pt}
          \setlength{\belowdisplayskip}{-8pt}
          \begin{alignat*}{8}
          &i &&\in [&&2.0&&,&&2.0&&] \\
          &l &&\in [&&a-\varepsilon&&,&&a+\varepsilon&&] \\
          &r &&\in         [&&b-\varepsilon&&,&&b+\varepsilon&&] \\
          &\varepsilon &&= 0.005
          \end{alignat*}
          }
      &  
        {
          \setlength{\abovedisplayskip}{-8pt}
          \setlength{\belowdisplayskip}{-8pt}
          \begin{alignat*}{8}
          &i &&\in [&&a-\varepsilon&&,&&a+\varepsilon&&] \\
          &l &&\in         [&&b-\varepsilon&&,&&b+\varepsilon&&] \\
          &r &&\in [&&1.00&&,&&2.00&&] \\
          &\varepsilon &&= 0.005
          \end{alignat*}
          }
    \\ \hline
    $\langle i,l,r\rangle$\newline
    \normalsize{Original Input $\hat{x}_0$} 
    & \multicolumn{2}{c}{None} 
    & {$\langle2.0,2.0,10.0\rangle$}
    & {$\langle-0.70,0.50,1.00\rangle$}
    \\ \hline
    Other Condition $C$ & None & No state cycle & None & None \\
    
    \hline
    State Transition $T$ & \multicolumn{2}{c}{{History features shift by one value}}  &  \multicolumn{2}{c}{None}
    
    \\ \hline
    Post-condition $Q$ 
    & Forall \newline $N(x_i) < 0$ 
    & Exist   \newline $N(x_i) > 0$
    & {$N(x_0)\approx N(\hat{x}_0)$} \newline{$\epsilon=0.5$} 
    & {$N(x_0)\approx N(\hat{x}_0)$} \newline{$\epsilon=0.5$} 
    \\ \hline
    
    \end{tabularx}
    \caption{Aurora decision boundary questions.}
    
    \label{tab:aubd}
\end{table*}

%% file: tab/Aurora/Intuitiveness_Examination.tex
\begin{table*}[htb]
    \centering
    
    \begin{tabularx}{\textwidth}{p{2.7cm}XXXX}
    \hline
     & \textbf{$\phi_{10.a}$} & \textbf{$\phi_{9.a}$} & \textbf{$\phi_{9.b}$} & \textbf{$\phi_{9.c}$} 
     \\ \hline
     Type &  Liveness & Safety & Safety & Safety
     \\ \hline
     Variable & $var\in[1,100]$ & $var\in[1,100]$ & $var\in[1,100]$ & $var\in[1,100]$
     \\ \hline
     State Boundary $S$ 
      & 
          {\setlength{\abovedisplayskip}{-8pt}
          \setlength{\belowdisplayskip}{-8pt}
          \begin{alignat*}{8}
          &i &&\in [&-0&.01&&,&+0&.01&&] \\
          &l &&\in [& 1&.00&&,& 1&.01&&] \\
          &y &&\in [& &var&&,& 100&.00&&] 
          \end{alignat*}}
      &
         {\setlength{\abovedisplayskip}{-8pt}
          \setlength{\belowdisplayskip}{-8pt}
          \begin{alignat*}{8}
          &i &&\in [&-0&.01&&,&+0&.01&&] \\
          &l &&\in [& 1&.00&&,& 1&.01&&] \\
          &y &&\in [& &var&&,& 100&.00&&] 
          \end{alignat*}}
      &
          {\setlength{\abovedisplayskip}{-8pt}
          \setlength{\belowdisplayskip}{-8pt}
          \begin{alignat*}{8}
          &i &&\in [&-0&.01&&,&+0&.01&&] \\
          &l &&\in [& 1&.00&&,& 1&.01&&] \\
          &y &&\in [& 2&.00&&,& 100&.00&&] 
          \end{alignat*}}
      &  
          {\setlength{\abovedisplayskip}{-8pt}
          \setlength{\belowdisplayskip}{-8pt}
          \begin{alignat*}{8}
          &i &&\in [&-0&.01&&,&+0&.01&&] \\
          &l &&\in [&  &var&&,& 1&.01&&] \\
          &y &&\in [& 2&.00&&,& 100&.00&&] 
          \end{alignat*}}
      \\
    \hline
    Other Condition $C$ & {No state cycle} &None & None & None \\
    \hline
    Initial State $I$ & \multicolumn{4}{c}{None} \\
    \hline
    State Transition $T$ & \multicolumn{4}{c}{History features shift by one value}  
    \\ \hline
    
    Post-condition $Q$ &  Exists \newline $N(x_i) < 0$ 
    & Forall \newline $N(x_i) < 0$
    & Forall \newline $N(x_i) < var$ &
    Forall \newline $N(x_i) < 0$
    \\    \hline
    Breakpoint(s)    & $var=50.11$ & $var=50.11$   & $var=86$      & $var=98.07$  \\
    Reinfier Result      & \texttt{Proven} & \texttt{Proven}  & \texttt{Proven}     & \texttt{Proven}    
    \\ \hline
    UINT Result & {Not applicable} & {Not applicable} & {Not applicable} & {Not applicable}  
    \\ \hline
    
    \end{tabularx}
    \caption{Aurora intuitiveness examination questions and comparison of their verification results by Reinfier and UINT.}
    
    \label{tab:auint}
\end{table*}

%% file: tab/Aurora/Counterfactual.tex
\begin{table*}[htb]
    \centering
    
    \begin{tabularx}{\textwidth}{p{2.8cm}XXXX}
    \hline
     & \textbf{$\phi_{9.d}$} & \textbf{$\phi_{7.a}$} & \textbf{$\phi_{11.b}$} & \textbf{$\phi_{12.b}$} 
     \\ \hline
     Type  & Safety & Liveness & Robustness & Robustness
     \\ \hline
    State Boundary $S$ \newline or\newline Initial State $I$ 
      & 
        {
          \setlength{\abovedisplayskip}{-8pt}
          \setlength{\belowdisplayskip}{0pt}
          \begin{alignat*}{8}
          &i &&\in [&&0.00+a&&,&&100&&] \\
          &l &&\in [&&1.00+b&&,&&100&&] \\
          &r &&\in [&&2.00+c&&,&&100&&] 
          \end{alignat*}
          }
      &
        {
          \setlength{\abovedisplayskip}{-8pt}
          \setlength{\belowdisplayskip}{0pt}
          \begin{alignat*}{8}
          &i &&\in [&-0&.01&&-a&&,\\ & & & &+0&.01&&+a&&] \\
          &l &&\in [& 1&.00&&-b&&,\\ & & & &1&.01&&+b&&] \\
          &r &&\in [& 1&.0&&-c&&,\\ & & & &1&.0&&+c&&] 
          \end{alignat*}
          }
      & 
          {
          \setlength{\abovedisplayskip}{-8pt}
          \setlength{\belowdisplayskip}{0pt}
          \begin{alignat*}{8}
          &i &&\in [& 0&.0&&-a&&, \\
          & & & &0&.0&&+a&&] \\
          &l &&\in [& 1&.0&&-b&&, \\
          & & & &1&.0&&+b&&] \\
          &r &&\in [& 1&.0&&-c&&, \\
          & & & &1&.0&&+c&&] 
          \end{alignat*}
          }
      &   
        {
          \setlength{\abovedisplayskip}{-8pt}
          \setlength{\belowdisplayskip}{0pt}
          \begin{alignat*}{8}
          &i &&\in [&-0&.7&&-a&&,\\
          & & & & -0&.7&&+a&&] \\
          &l &&\in [& 0&.5&&-b&&,\\
          & & & &  0&.5&&+b&&] \\
          &r &&\in [& 1&.0&&-c&&,\\
          & & & &  1&.0&&+c&&] 
          \end{alignat*}
          }
    \\ \hline
    Original Input $\hat{x}_0$ 
    & \multicolumn{2}{c}{None} 
    & {$\langle2.00,2.00,10.00\rangle$}
    & {$\langle-0.70,0.50,1.00\rangle$}
    \\ \hline
    $\langle i,l,r\rangle$\newline
    \normalsize{Original Input $\hat{x}_0$} 
    & \multicolumn{2}{c}{None} 
    & {-1.119}
    & {7.275}
    \\ \hline
    Other Condition $C$ & None &{No state cycle} & None & None \\
    
    \hline
    State Transition $T$ & \multicolumn{2}{c}{{History features shift by one value}}  &  \multicolumn{2}{c}{None}
    
    \\ \hline
    Post-condition $Q$ 
    & Forall \newline $N(x_i) < 0$ 
    & Exist \newline $N(x_i) \approx 0$
    & Exist \newline $N(x_0) > 0$
    & Exist \newline $N(x_0) < 0$ 
    \\ \hline
    $\langle a,b,c\rangle$  \newline
    Distance(Euclidean) \newline Reinfier Result    
    & {$\langle 0.60,0.40,0.87\rangle$ \newline 1.130} 
    & {$\langle 0.00,0.10,0.17\rangle$ \newline 0.199}
    & {$\langle 2.20,1.20,0.16\rangle$ \newline 2.511}
    & {$\langle 0.40,0.90,0.92\rangle$ \newline 1.348}
    \\ \hline
     $\langle a,b,c\rangle$ \newline UINT Result   & \footnotesize{Not applicable} & \footnotesize{Not applicable}   & {$\langle 2.20,1.20,0.16\rangle$ }
    & {$\langle 0.40,0.90,0.9\rangle$}
    \\ \hline
    
    \end{tabularx}
    \caption{Aurora counterfactual explanation questions and comparison of their results by Reinfier and UINT.}
    
    \label{tab:auce}
\end{table*}

%% file: tab/Aurora/Importance.tex
\begin{table*}[htb]
    \centering
    
    \begin{tabularx}{\textwidth}{p{4cm}XXXX}
    \hline
     & \textbf{Input 1} & \textbf{Input 2} & \textbf{Input 3} & \textbf{Input 4} 
     \\ \hline
     $\langle i,l,r\rangle$
    \newline 
    \normalsize Original Input $\hat{x}_0$ 
      & {$\langle 0.00,1.00,1.00\rangle$} 
      & {$\langle 0.00,1.00,2.00\rangle$} 
      & {$\langle 0.00,2.00,1.00\rangle$}
      & {$\langle 0.00,2.00,2.00\rangle$}   \\
    \hline
     $N(x_0) \not \approx N(\hat{x_0})$ Level 1 & $\epsilon=1\times10^{-0}$ & $\epsilon=1\times10^{-0}$   & $\epsilon=1\times10^{-0}$      & $\epsilon=1\times10^{-0}$  \\
     $N(x_0) \not \approx N(\hat{x_0})$ Level 2 & $\epsilon=1\times10^{-1}$ & $\epsilon=1\times10^{-1}$   & $\epsilon=1\times10^{-1}$      & $\epsilon=1\times10^{-1}$  \\
     $N(x_0) \not \approx N(\hat{x_0})$ Level 3 & $\epsilon=1\times10^{-2}$ & $\epsilon=1\times10^{-2}$   & $\epsilon=1\times10^{-2}$      & $\epsilon=1\times10^{-2}$  \\
     $N(x_0) \not \approx N(\hat{x_0})$ Level 4 & $\epsilon=1\times10^{-3}$ & $\epsilon=1\times10^{-3}$   & $\epsilon=1\times10^{-3}$      & $\epsilon=1\times10^{-3}$  \\
    \hline
    \end{tabularx}
    \caption{Aurora importance analysis questions.}
    
    \label{tab:imau}
\end{table*}

%% file: tab/Aurora/Sensitivity.tex
\begin{table*}[htb]
    \centering
    
    \begin{tabularx}{\textwidth}{p{4cm}XXXX}
    \hline
     & \textbf{Input 1} & \textbf{Input 2} & \textbf{Input 3} & \textbf{Input 4} 
     \\ \hline
    $\langle i,l,r\rangle$
    \newline 
    \normalsize Original Input $\hat{x}_0$ 
      & \scriptsize{$\langle 0.00,1.00,1.00\rangle$} 
      & \scriptsize{$\langle 0.00,1.00,2.00\rangle$} 
      & \scriptsize{$\langle 0.00,2.00,1.00\rangle$}
      & \scriptsize{$\langle 0.00,2.00,2.00\rangle$}   \\
    \hline
     Perturbation Level 1 & $\varepsilon=1\times10^{-0}$ & $\varepsilon=1\times10^{-0}$   & $\varepsilon=1\times10^{-0}$      & $\varepsilon=1\times10^{-0}$  \\
     Perturbation Level 2 & $\varepsilon=1\times10^{-1}$ & $\varepsilon=1\times10^{-1}$   & $\varepsilon=1\times10^{-1}$      & $\varepsilon=1\times10^{-1}$  \\
     Perturbation Level 3 & $\varepsilon=1\times10^{-2}$ & $\varepsilon=1\times10^{-2}$   & $\varepsilon=1\times10^{-2}$      & $\varepsilon=1\times10^{-2}$  \\
     Perturbation Level 4 & $\varepsilon=1\times10^{-3}$ & $\varepsilon=1\times10^{-3}$   & $\varepsilon=1\times10^{-3}$      & $\varepsilon=1\times10^{-3}$  \\
    \hline
    \end{tabularx}
    \caption{Aurora sensitivity analysis questions.}
    
    \label{tab:seau}
\end{table*}

%% file: tab/Advantage_UINT.tex
\begin{table*}[htb]
    \centering
    
    \begin{tabularx}{\textwidth}{{p{4cm}XXXX}}
    \hline
                     & \textbf{UINT}                                                  & \textbf{Reinfier}                                           & Benefit                                                                                        \\ \hline
    Property Syntax  & SMT formula                                                    & Intuitive and concise DRLP syntax                  & Make it easy for users to read and write                                                       \\ \hline
    DNN Structure    & Linear layer and \newline convolutional layer                  & Almost all structures                                       & Support more DRL systems and adapt to the rapid development of DNN                    \\ \hline
    Result Guarantee  & Results from an SMT solver                                     & Comprehensive results from  multiple DNN verifiers & Reduce the risk of wrong  answer                                                       \\ \hline
    Type of Property & Robustness                                                     & Safety, Liveness, Robustness                                & Facilitate comprehensive interpretation                                               \\ \hline
    Solution        & Completely different solutions for different problems & Unified based on \textit{breakpoints}                       & Accelerate total solution speed and facilitate problem extension                      \\ \hline
    \end{tabularx}
    \caption{Advantages of Reinfier compared to UINT.}
    \label{tab:Advantage UINT}
    
\end{table*}

%% file: tab/Reintrainer/TrainProcess.tex

\bgroup
\setlength{\textfloatsep}{12pt plus 1.0pt minus 2.0pt}
\setlength{\floatsep}{5pt plus 1.0pt minus 2.0pt}
\setlength{\intextsep}{5pt plus 1.0pt minus 2.0pt}
\setlength{\extrarowheight}{2pt} 
\begin{table*}[htb]
    \centering
    
    \begin{tabularx}{\textwidth}{p{1.6cm}p{2.8cm}XXX}
    \hline
     \multicolumn{2}{l}{\textbf{Task}} &  \multicolumn{1}{c}{\textbf{Mountain Car}} & \multicolumn{1}{c}{\textbf{Cartpole}} & \multicolumn{1}{c}{\textbf{Pendulum}}
     \\ \hline
     \multicolumn{2}{l}{Algorithm} & \multicolumn{1}{l}{DQN} & DQN & DDPG
     \\ \hline 
     \multicolumn{2}{l}{ID} & $\phi_{1}$ & $\phi_{2}$ & $\phi_{3}$
     \\ \hline 

     \multirow{2}{*}{\shortstack[c]{Network}} 
         & Activation Function
         & ReLU
         & ReLU
         & ReLU
     \\ \cline{2-5}
         & Size
         & $2\times 256$
         & $2\times 256$ 
         & $300,400$
     \\ \hline

     \multirow{4}{*}{\shortstack[c]{Training \\ Time \\ (\si{\second})}} 
         & Reintrainer
         & $878$
         & $192$
         & $809$
     \\ \cline{2-5}
         & Trainify\#1
         & $443$
         & $153$
         & $5923$
     \\ \cline{2-5}
         & Trainify\#2
         & $912$
         & $144$
         & $6209$
     \\ \cline{2-5}
         & Vanilla
         & $\bm{197}$
         & $\bm{60}$
         & $\bm{150}$
     \\ \hline

      \multirow{4}{*}{\shortstack[c]{Timesteps}} 
         & Reintrainer
         & $\bm{1.0\times10^5}$
         & $7.5\times10^4$
         & $1.7\times10^4$
     \\ \cline{2-5}
         & Trainify\#1
         & $6.3\times 10^5$
         & $1.8\times 10^5$
         & $8.0\times 10^5$
     \\ \cline{2-5}
         & Trainify\#2
         & $1.3\times 10^6$
         & $2.8\times 10^5$
         & $8.0\times 10^5$
     \\ \cline{2-5}
         & Vanilla
         & $1.2\times10^5$
         & $\bm{4.0\times10^4}$
         & $\bm{5.0\times10^3}$
     \\ \hline

      \multirow{4}{*}{\shortstack[c]{Violations}} 
         & Reintrainer
         & $\bm{96}$
         & $15$
         & $400$
     \\ \cline{2-5}
         & Trainify\#1
         & $668$
         & $\bm{0}$
         & $\bm{27}$
     \\ \cline{2-5}
         & Trainify\#2
         & $1735$
         & $\bm{0}$
         & $180$
     \\ \cline{2-5}
         & Vanilla
         & $117$
         & $\bm{0}$
         & $68$
     \\ \hline

    \end{tabularx}
    \begin{tabularx}{\textwidth}{p{1.6cm}p{2.8cm}XXp{1.65cm}p{1.65cm}}
    \hline
     \multicolumn{2}{l}{\textbf{Task}} &  B1 & B2 & \multicolumn{2}{c}{\textbf{Tora}}
     \\ \hline
     \multicolumn{2}{l}{Algorithm} & DDPG  & DDPG  & \multicolumn{2}{l}{DDPG}
     \\ \hline 
     \multicolumn{2}{l}{ID} & 
     \multicolumn{1}{l}{$\phi_{4}$} & 
     \multicolumn{1}{l}{$\phi_{5}$} & 
     \multicolumn{2}{l}{$\phi_{6}$}
     \\ \hline 

     \multirow{2}{*}{\shortstack[c]{Network}} 
         & Activation Function
         & \multicolumn{1}{l}{Tanh}
         & \multicolumn{1}{l}{Tanh}
         & Tanh
         & Tanh
     \\ \cline{2-6}
         & Size
         &  \multicolumn{1}{l}{$2\times 20$}
         &  \multicolumn{1}{l}{$2\times 20$}
         & $3\times 100$ 
         & $3\times 200$ 
     \\ \hline

     \multirow{3}{*}{\shortstack[c]{Training \\ Time \\ (\si{\second})}} 
         & Reintrainer
         & $84$
         & $140$
         & $491$
         & $751$
     \\ \cline{2-6}
         & Trainify\#1
         & $304$
         & $\bm{25}$
         & $716$
         & $797$
     \\ \cline{2-6}
         & Vanilla
         & $\bm{80}$
         & $54$
         & $\bm{33}$
         & $\bm{48}$
     \\ \hline

      \multirow{3}{*}{\shortstack[c]{Timesteps}} 
         & Reintrainer
         & $\bm{3.8\times10^4}$
         & $4.8\times10^4$
         & $\bm{9.6\times10^3}$
         & $\bm{9.6\times10^3}$
     \\ \cline{2-6}
         & Trainify\#1
         & $7.7\times 10^4$
         & $\bm{1.4\times 10^4}$
         & $1.4\times 10^5$
         & $1.5\times 10^5$
     \\ \cline{2-6}
         & Vanilla
         & $\bm{3.8\times10^4}$
         & $2.4\times10^4$
         & $\bm{9.6\times10^3}$
         & $\bm{9.6\times10^3}$
     \\ \hline

      \multirow{3}{*}{\shortstack[c]{Violations}} 
         & Reintrainer
         & $256$
         & $128$
         & $\textbf{0}$
         & $\textbf{0}$
     \\ \cline{2-6}
         & Trainify\#1
         & $361$
         & $67$
         & $\textbf{0}$
         & $\textbf{0}$
     \\ \cline{2-6}
         & Vanilla
         & $\textbf{160}$
         & $\textbf{64}$
         & $\textbf{0}$
         & $\textbf{0}$
     \\ \hline

    \end{tabularx}
    
    \caption{Comparison of the training process.}
    \label{tab:train_proc}
    
\end{table*}
\egroup

%% file: tab/Reintrainer/ablation.tex
\begin{table*}[htb]
\begin{tabularx}{\textwidth}{p{1.3cm}XXXXXX}
\hline
                                 & \multicolumn{1}{c}{Vanilla} & \multicolumn{1}{c}{Const}      & \multicolumn{1}{c}{Raw Dist} & \multicolumn{1}{c}{Dist} & \multicolumn{1}{c}{Density} & \multicolumn{1}{c}{Gap} \\ \hline
\multicolumn{1}{c}{$\Bar{R}\pm\sigma$} & -146.53$\pm$7.36             & -153.65$\pm$6.93 & -152.88$\pm$10.60             & -148.61$\pm$10.43        & -151.94$\pm$9.71            & -144.79$\pm$9.58        \\ \hline
\end{tabularx}

\caption{Ablation experiments of different reward shaping strategies.$\bm{\bar{R}\pm\sigma}$ stands for the average episode reward and its standard deviation of the final models' evaluation results.}
\label{tab:abl}
\end{table*}

%% file: tab/hp.tex
\setlength{\extrarowheight}{2pt}
\begin{table*}[htb]
    \centering

    \begin{minipage}[H]{0.49\textwidth}
        \begin{tabularx}{\linewidth}{p{3.5cm}XX}
            \hline
             \textbf{Task} &  \multicolumn{1}{c}{\textbf{Mountain Car}}&  \multicolumn{1}{c}{\textbf{Cartpole}}
             \\ \hline
             Algorithm & DQN & DQN
             \\ \hline
            Batch Size & 128 & 64\\ \hline
            Buffer Size & 10000 & 10000\\ \hline
            Exploration Final Eps & 0.07 & 0.04\\ \hline
            Exploration Fraction & 0.2 & 0.16\\ \hline
            Gamma & 0.98 & 0.99\\ \hline
            Gradient Steps & 8 & 128\\ \hline
            Learning Rate & 0.004 & 0.023\\ \hline
            Learning Starts & 1000 & 1000\\ \hline
            Policy & MlpPolicy & MlpPolicy\\ \hline
            Target Update Interval & 600 & 10 \\ \hline
            Train Freq & 16 Steps & 256 Steps \\ \hline
            Normalize & False & False\\ \hline
            
        \end{tabularx}
    \end{minipage}%
\hfill
    \begin{minipage}[H]{0.49\textwidth}

    \begin{tabularx}{\linewidth}{p{2.5cm}XXX}
    \hline
     \textbf{Task} &  \multicolumn{1}{c}{\textbf{Pendulum}} & \multicolumn{1}{c}{\textbf{B1, B2}} & \multicolumn{1}{c}{\textbf{Tora}} 
     \\ \hline
     Algorithm & DDPG & DDPG & DDPG
     \\ \hline
    Buffer Size & 200000 & 10000 & 10000\\ \hline
    Batch Size & 32 & 32 & 32\\ \hline
    Gamma & 0.98 & 0.99 & 0.99 \\ \hline
    Tau & 0.005 & 0.02 & 0.02\\ \hline
    Gradient Steps & 200 & 1& 1 \\ \hline
    Learning Rate & 0.001 & 0.001   & 0.001\\ \hline
    Learning Starts & 10000 & 0 & 0\\ \hline
    Noise Std & 0.1 & None & None\\ \hline
    Noise Type & Normal & None & None \\ \hline
    Policy & MlpPolicy & MlpPolicy  & MlpPolicy \\ \hline
    Train Freq & 1 Episode &   4 Steps &   10 Steps \\ \hline
    Normalize & False & False & False\\ \hline
    
    \end{tabularx}
    \end{minipage}
    
    \caption{Training hyper-parameters of Mountain Car, Cartpole, Pendulum, B1, B2, and Tora.}
    
    \label{tab:hyper_parameter_pd1}
\end{table*}

%% file: tab/Reintrainer/propdef.tex
\begin{table*}[htb]
    \centering
    
    \begin{tabularx}{\textwidth}{p{0.8cm}p{1.2cm}p{0.5cm}p{3.2cm}p{2.8cm}X}
    \hline
     \textbf{Task} &  \textbf{Type} & \textbf{ID} & \textbf{Pre-condition $P$} & \textbf{Post-condition $Q$} & \textbf{Meaning}
    \\ \hline
        MC & Safety & $\phi_{1}$
        & $p\in[-0.40,-0.60]  $\newline $v\in[\phantom{-}0.03,\phantom{-}0.07]$
        &  Forall  \newline \small $Action \neq 0$ 
        & When the car is moving at high speed to the right at the bottom of the valley, it should not accelerate to the left.
    \\ \hline
         CP & Safety & $\phi_{2}$
         & $p\in[-2.40,-2.00]  \newline v\in[\phantom{-}0.00,+\infty\phantom{0|}] \newline \theta\in[\phantom{-}0.15,\phantom{-}0.21] 
          \newline \omega\in[\phantom{-}1.00,+\infty\phantom{0}] $
         & Forall  \newline \small $Action = 1$ 
         &   When the car is moving to the right in the left edge area and the pole tilts to the right at a large angle, it should be pushed to the right.
    \\ \hline
         PD & Safety & $\phi_{3}$
         & $x\in[\phantom{-}0.00,\phantom{-}1.00]  \newline y\in[-0.10,\phantom{-}0.10]  \newline \omega\in[-0.50,\phantom{-}0.50]$
         & Forall  \newline \small $Action$ $\in [-1.00, 1.00]$
         & When the pendulum approaches the upright position of the target, large torque should not be applied.
     \\ \hline
          \multirow{1}{*}{\shortstack[c]{B1}} 
          & Liveness & $\phi_{4}$
         & $x_0^0\in[\phantom{-}0.80,\phantom{-}0.90]  \newline x_0^1\in[\phantom{-}0.50,\phantom{-}0.60]$
         & Exist \small \newline $x^0 \in [\phantom{-}0.00,\phantom{-}0.20]$ \newline $x^1 \in [\phantom{-}0.05,\phantom{-}0.30]$
         & The agent always reaches the target eventually.
      \\ \hline
          \multirow{1}{*}{\shortstack[c]{B2}} 
          & Liveness & $\phi_{5}$
         & $x_0^0\in[\phantom{-}0.70,\phantom{-}0.90]  \newline x_0^1\in[\phantom{-}0.70,\phantom{-}0.90]$
         & Exist \small \newline $x^0 \in [-0.30,\phantom{-}0.10]$  \newline $x^1 \in [-0.35,\phantom{-}0.50]$
         & The agent always reaches the target eventually.
      \\ \hline
      \multirow{1}{*}{\shortstack[c]{Tora}} 
         & Safety & $\phi_{6}$
         & $x_0^0\in[-0.77,-0.75]  \newline x_0^1\in[-0.45,-0.43]  \newline x_0^2\in[\phantom{-}0.51,\phantom{-}0.54]  \newline x_0^3\in[-0.30,-0.28]$
         & Forall  \small \newline $x^0 \in [\phantom{-}1.50,\phantom{-}1.50]$  \newline $x^2 \in [\phantom{-}1.50,\phantom{-}1.50]$
         & The agent always stays in the safe region.
      \\ \hline
    \end{tabularx}
    \caption{Predefined properties of benchmarks.}
    \label{tab:train_propappx}
    
\end{table*}

%% file: supsec/SupImplementation.tex
\section{Usage}\label{Appx:usage}
A DRLP object storing a property in DRLP format and an NN object storing an ONNX DNN are required for a DRL verification query in \rt. \rt~can be accessed at https://github.com/Kurayuri/Reinfier.

\vspace{10pt}
\begin{lstlisting}[frame=single]
import reintrainer as rf

network = rf.NN("/path/to/ONNX/file")
# or
network = rf.NN(ONNX_object)

property = rf.DRLP("/path/to/DRLP/file")
# or
property = rf.DRLP(DRLP_str)

# Verify API (default k-induction algorithm, Recommended)
result = rf.verify(network, property)
# or
# k-induction algorithm 
result = rf.k_induction(network, property) 
# or
# bounded model checking algorithm
result = rf.bmc(network, property) 
# or
# reachability analysis
algorithm
result = rf.reachability(network, property) 

\end{lstlisting}

\vspace{10pt}
And the result should be like: 
\vspace{10pt}

\begin{lstlisting}[frame=single]
( Falsified/Proven/None, 
  an int of verification depth,
  an violation instance of <numpy.ndarray>)
\end{lstlisting}

\vspace{10pt}
If you want to execute \textit{batch verification}, just run:
\vspace{10pt}

\begin{lstlisting}[frame=single]
result = rf.verify_linear(network, property) 
# or
result = rf.verify_hypercubic(network, property)
\end{lstlisting}

\vspace{10pt}
If you want to \textit{search breakpoints} and ask interpretability problems, just run:
\vspace{10pt}

\begin{lstlisting}[frame=single]
# set search keyword arguments
kwargs = {
    "a": {"lower_bound": -0.7,
          "upper_bound": -0.3,
          "precise": 0.02,
          "method": "linear", },
    "b": {"lower_bound": 1,
          "upper_bound": 20,
          "precise": 0.1,
          "method": "binary", },
}
# search breakpoints
breakpoints = rf.search_breakpoints(network,property,kwargs)

# analyze breakpoints
inline_breakpoints, inline_breaklines=analyze_breakpoints(breakpoints)

# answer interpretability problems
result = rf.interpreter.answer_importance_analysis(inline_breakpoints)
# or sensitivity_analysis
result = rf.interpreter.answer_sensitivity_analysis(inline_breakpoints)
# or
result = rf.interpreter.answer_intuitiveness_examination(inline_breakpoints)
#or
result = rf.interpreter.answer_counterfactual_explanation(inline_breakpoints)
#or
rf.interpreter.draw_decision_boundary(inline_breakpoints)

\end{lstlisting}

\vspace{10pt}
If you want to use Reintrainer to train a DRL model to satisfy desired DRLP properties, just run:
\vspace{10pt}

\begin{lstlisting}[frame=single]
reintrainer = rf.Reintrainer(
    [list of desired DRLP properties], 
    train_api)

reintrainer.train(max iterations)
\end{lstlisting}